%% file: neurips_2026.tex
\documentclass{article}

\PassOptionsToPackage{numbers, compress}{natbib}

\usepackage[main, final]{neurips_2026}

\usepackage[utf8]{inputenc} 
\usepackage[T1]{fontenc}    
\usepackage{hyperref}       
\usepackage{url}            
\usepackage{booktabs}       
\usepackage{amsfonts}       
\usepackage{nicefrac}       
\usepackage{microtype}      
\usepackage{xcolor}         
\newcommand{\notecolor}{black}
\newcommand{\note}[1]{{\color{\notecolor}#1}}

\usepackage{amsmath}
\usepackage{algorithm, algpseudocode}

\usepackage{amsthm}
\newtheorem{theorem}{Theorem}
\usepackage{aliascnt} 
\newaliascnt{lemma}{theorem}\newtheorem{lemma}[lemma]{Lemma}\aliascntresetthe{lemma}
\newaliascnt{proposition}{theorem}\newtheorem{proposition}[proposition]{Proposition}\aliascntresetthe{proposition}
\newaliascnt{corollary}{theorem}\newtheorem{corollary}[corollary]{Corollary}\aliascntresetthe{corollary}

\usepackage{graphicx}
\usepackage{caption}

\title{Asymptotically Exact Negative Guidance of \\ Diffusion Models via Positive--Unlabeled Learning}

\author{%
  Hirohane Takagi$^{1}$ \qquad\qquad\qquad Masashi Sugiyama$^{2,1}$ \\
  $^{1}$\,The University of Tokyo, Tokyo, Japan \\
  $^{2}$\,RIKEN AIP, Tokyo, Japan \\
  \texttt{htakagi@is.s.u-tokyo.ac.jp} \quad \texttt{sugi@k.u-tokyo.ac.jp} \\
}

\begin{document}

\maketitle

\begin{abstract}
Negative guidance of diffusion models generates samples that avoid an undesired condition,
specified either by reference samples or by a conditional score alongside an unconditional model.
Existing methods inject repulsive terms derived from local information at each sample,
so the resulting sampler does not target an explicit density,
and its bias cannot be systematically reduced by increasing computation.
We instead present a negative-guidance method that samples from an explicit target density,
which uniformly handles both reference-sample and score-based specifications and suppresses probability mass near the undesired condition with an avoidance strength that parametrizes the target density itself.
Realizing this target requires estimating, at each diffusion timestep,
how likely each sample is to belong to the undesired condition.
We cast this estimation as a Positive--Unlabeled learning problem that admits lightweight online training along the diffusion trajectory,
using condition samples as positives and diffusion particles as unlabeled data.
Because the guided particles drift from the unconditional marginal,
we further correct the residual mismatch via Sequential Monte Carlo,
yielding a sampler that is asymptotically exact in the particle limit and that turns any existing negative-guidance method into its proposal kernel.
Empirically, our method approaches the ground-truth target on Gaussian mixtures and Pareto-dominates existing methods on trade-offs between avoidance and generation quality, distributional bias, and diversity,
including in an image-generation setting with a pretrained score network.
\end{abstract}

\section{Introduction}
\label{sec:intro}

Diffusion models have become a dominant framework for sampling from a data density $p(x)$ over objects $x$.
Their conditional generation mechanisms, including classifier guidance (CG;~\citealp{NEURIPS2021_49ad23d1}) and classifier-free guidance (CFG;~\citealp{ho2022classifierfreediffusionguidance}), are now standard across modalities,
including images \citep{Rombach_2022_CVPR}, molecules \citep{pmlr-v162-hoogeboom22a}, audio \citep{pmlr-v202-liu23f}, decision-making and robotic control \citep{pmlr-v162-janner22a,ajay2023is}.
CFG amplifies a \emph{desired} condition by extrapolating between conditional and unconditional scores,
and recent work has characterized the induced density and the resulting fidelity--diversity trade-off \citep{pmlr-v235-wu24b,NEURIPS2024_9a3942c2,bradley2025classifierfree,pmlr-v267-li25as,zhan2026understanding}.

In contrast, \emph{negative} guidance aims to steer generation \emph{away from} an undesired condition.
This capability is important for AI safety, where harmful outputs should be suppressed \citep{Schramowski_2023_CVPR,Gandikota_2023_ICCV,Gandikota_2024_WACV},
and where memorized or copyrighted training data should not be reproduced \citep{291199,Somepalli_2023_CVPR}.
It is also important for controllable generation, where constraints are often more naturally expressed by negation than by positive prompts \citep{armandpour2023re,10.1007/978-3-031-73024-5_12}.
Throughout this paper, $c$ denotes the undesired condition and $\neg c$ its complement.
The conditional probability $p(\neg c\mid x)$ satisfies $p(\neg c\mid x)=1-p(c\mid x)$.
Thus $c$ plays the opposite role to the conditioning variable in CFG: the goal is to reduce $p(c\mid x)$ for generated samples while preserving the structure of $p(x)$.

Existing negative guidance methods add a CFG-like repulsive term to the sampling dynamics of a fixed pretrained diffusion model,
either without retraining the score network \citep{Gandikota_2023_ICCV,takahashi2025positiveunlabeled} or by editing internal representations \citep{Gandikota_2024_WACV,biswas2025cure}.
The repulsion is constructed from either a pretrained conditional score for $p(x\mid c)$ \citep{koulischer2025dynamic},
or a finite set of reference samples from $p(x\mid c)$ \citep{kim2025trainingfree}.
These methods share two structural limitations:
the repulsive update uses only information local to each sample,
so the resulting sampler does not target an explicit density,
and the induced bias cannot be reduced by allocating more computation.
We formalize these limitations in \autoref{sec:existing}.

We instead formulate negative guidance as sampling from the explicit avoidance target
\begin{equation}
\label{eq:target}
p_{\mathrm{target}}(x)
\;\propto\;
p(x)\,(1-p(c\mid x))^{\lambda},
\qquad \lambda\ge 0,
\end{equation}
where $\lambda$ is a user-controlled avoidance strength that parametrizes the target density itself:
$\lambda=0$ recovers the unconditional density,
and $\lambda\to\infty$ concentrates mass on the support of $\neg c$.
When the condition $c$ is a semantic concept rather than a deterministic rule, $p(c\mid x)$ varies smoothly with $x$,
and intermediate $\lambda$ continuously interpolates between fidelity to $p(x)$ and avoidance of $c$.
The remaining challenge is to sample from~\eqref{eq:target} using only the information available in negative guidance:
an unconditional score for $p(x)$ together with either a conditional score for $p(x\mid c)$ or reference samples from $p(x\mid c)$.

Our approach rests on the observation that estimating $p(c\mid x)$ in this setting can be cast as a Positive--Unlabeled (PU) learning problem \citep{DENIS200570,10.1145/1401890.1401920},
with samples associated with class $c$ as positives and unconditional diffusion particles as unlabeled data.
We pair an online PU estimator of $p(c\mid x)$ with Sequential Monte Carlo (SMC;~\citealp{del2006sequential,chopin2020introduction}) targeting~\eqref{eq:target},
which corrects the residual mismatch between guided particles and the target marginal.
The resulting sampler is asymptotically exact in the particle limit,
in the same spirit as test-time diffusion alignment via SMC \citep{NEURIPS2023_63e8bc7b,pmlr-v267-singhal25b,kim2025testtime},
while keeping per-step PU updates lightweight.

\paragraph{Our contributions.}
\begin{itemize}
\item We formulate negative guidance as sampling from the explicit avoidance density~\eqref{eq:target},
\note{realized through the finite-step reverse process of a trained diffusion model,}
for both reference-sample and conditional-score specifications.
\item We design an SMC sampler driven by an online PU estimator of $p(c\mid x)$,
prove its asymptotic exactness, and show that any existing negative-guidance method can serve as its proposal kernel.
\item Empirically, our sampler matches an oracle diffusion sampler with access to the true score for $p(x \mid \neg c)$ on Gaussian mixtures and Pareto-dominates representative baselines on avoidance--quality trade-offs under both specifications, including image generation with a pretrained score network.
\end{itemize}

\section{Existing Negative Guidance Methods}
\label{sec:existing}

This section reviews representative negative guidance methods,
namely sign-flipped CFG (NegCFG; also analyzed as standard negative prompting by \citealp{10.1007/978-3-031-73024-5_12}),
Dynamic Negative Guidance (DNG; \citealp{koulischer2025dynamic}), and the Safe Denoiser \citep{kim2025trainingfree}.
We then identify structural limitations that motivate our formulation.

\paragraph{Diffusion models.}
We use a variance-preserving diffusion process \citep{song2021scorebased,ho2020denoising} with forward kernel
\begin{equation}
\label{eq:forward-kernel}
q_t(x_t\mid x_0)
=
\mathcal N(x_t;\sqrt{\bar\alpha_t}x_0,(1-\bar\alpha_t)I),
\qquad t=1,\dots,T,
\end{equation}
where $\mathcal{N}(x;\mu,\Sigma)$ denotes the Gaussian density in $x$ with mean $\mu$ and covariance $\Sigma$ and $I$ is the identity matrix.
This kernel induces noisy marginals $p_t(x_t)$ with $p_T(x_T)$ close to $\mathcal N(0,I)$.
A pretrained noise predictor $\hat\epsilon_\theta(x_t,t)$ parameterizes the unconditional score via
\[
\hat\epsilon_\theta(x_t,t)
\approx
-\sqrt{1-\bar\alpha_t}\,\nabla_{x_t}\log p_t(x_t),
\]
and, when available, a conditional predictor $\hat\epsilon_\theta(x_t,t\mid c)$ analogously parameterizes $\nabla_{x_t}\log p_t(x_t\mid c)$.
Tweedie's formula \cite{efron2011tweedie} gives the standard denoised estimate
\[
\hat x_0(x_t)
=
\frac{x_t-\sqrt{1-\bar\alpha_t}\hat\epsilon_\theta(x_t,t)}{\sqrt{\bar\alpha_t}}
\approx
\mathbb E[x_0\mid x_t].
\]
DDPM (Denoising Diffusion Probabilistic Models) sampling \citep{ho2020denoising} starts from $x_T\sim\mathcal N(0,I)$
and applies the reverse kernels $p_\theta(x_{t-1}\mid x_t)=\mathcal N(x_{t-1};\mu_\theta(x_t,t),\sigma_t^2I)$ for $t=T,T-1,\ldots,1$.
All guidance methods reviewed here modify only the mean $\mu_\theta$, leaving the variance $\sigma_t^2 I$ unchanged; this is exploited in \autoref{sec:method}.
Full closed-form expressions are deferred to Appendix~\ref{app:prelim:diffusion}.

\paragraph{Sign-flipped CFG (NegCFG).}
NegCFG combines unconditional and conditional noise predictors:
\[
\hat\epsilon_{\mathrm{CFG}} (x_t, t)
\leftarrow
\hat\epsilon_\theta(x_t,t)
-\eta
\ \bigl(
\hat\epsilon_\theta(x_t,t\mid c)-\hat\epsilon_\theta(x_t,t)
\bigr),
\]
and throughout this paper $\eta>0$ consistently denotes repulsion from condition $c$, while $\eta<0$ corresponds to standard CFG that amplifies $c$.
Several variants build on this basic rule with stabilization mechanisms that suppress artifacts arising from score extrapolation, though these mechanisms are tailored to specific applications or architectures~\citep{Schramowski_2023_CVPR,chen2026normalized}.

\paragraph{Dynamic Negative Guidance (DNG).}
DNG~\citep{koulischer2025dynamic} starts from the exact score decomposition
\begin{equation}
\label{eq:neg-score}
\nabla_{x_t}\log p_t(x_t\mid \neg c)
\leftarrow
\nabla_{x_t}\log p_t(x_t)
+
\eta \ 
\nabla_{x_t}\log\bigl(1-p_t(c\mid x_t)\bigr),
\end{equation}
which requires the global posterior $p_t(c\mid x_t)$.
Applying Bayes' rule yields a CFG-style repulsion scaled by the posterior odds:
\[
\nabla_{x_t}\log(1-p_t(c\mid x_t)) = -\tfrac{p_t(c\mid x_t)}{1-p_t(c\mid x_t)}\bigl(\nabla_{x_t}\log p_t(x_t\mid c) - \nabla_{x_t}\log p_t(x_t)\bigr).
\]
Since $p_t(c\mid x_t)$ is unavailable during sampling,
DNG estimates it recursively via a local Gaussian approximation to the reverse transition \citep{Li_2023_ICCV}:
\begin{equation}
\label{eq:dngposterior}
\log p_{\note{t-1}}(c\mid x_{t-1})
\leftarrow
\log p_t(c\mid x_t)
-
\tfrac{1}{2\sigma_t^2}
\left(
\|\note{x_{t-1}}-\mu_\theta(x_t,t\mid c)\|^2
-
\|\note{x_{t-1}}-\mu_\theta(x_t,t)\|^2
\right),
\end{equation}
introducing an approximation bias whose magnitude has no known bound.
\note{Appendix~\ref{app:prelim:dng} supplements this point with the conditions under which DNG becomes exact and how they differ from our setting.}

\paragraph{Reference-particle negative guidance.}
NegCFG and DNG both require a pretrained conditional score $\hat\epsilon_\theta(x_t,t\mid c)$, which does not exist for user-specified conditions such as copyrighted or memorized content \citep{kim2025trainingfree}; in such cases, the condition is instead provided as a finite reference set $\{x_c^i\}_{i=1}^N$.
The Safe Denoiser \citep{kim2025trainingfree} addresses this regime by modifying the denoised estimate itself:
\[
\hat x_{\mathrm{safe}}(x_t)
=
\hat x_0(x_t)
+
\frac{\eta}{N}
\sum_{i=1}^N q_t(x_t\mid x_c^i)
\left(
\hat x_0(x_t)
-
\frac{\sum_{k=1}^N x_c^k q_t(x_t\mid x_c^k)}
{\sum_{j=1}^N q_t(x_t\mid x_c^j)}
\right),
\]
where $q_t(x_t\mid x_c^i)$ is the forward kernel~\eqref{eq:forward-kernel} evaluated as a likelihood (not a sampling density) and $\eta>0$ controls the repulsion strength.
Geometrically, the bracketed term pushes $\hat x_0(x_t)$ away from the kernel-weighted mean of the references, and the prefactor amplifies the repulsion when $x_t$ lies near the reference region.
In practice, $\eta$ and the time window over which guidance is applied must be tuned per condition \citep{kim2025trainingfree,kim2026safetyguided}, navigating an implicit avoidance--quality trade-off.
Shielded Diffusion~\citep{pmlr-v267-kirchhof25a} takes a similar reference-particle approach but sparsifies the repulsion, activating it only when $\hat{x}_0(x_t)$ falls within a fixed radius of the reference region.

\paragraph{Structural limitations.}
These methods are lightweight and effective in practice, but share a structural limitation:
each replaces the global posterior $p_t(c\mid x_t)$ in \eqref{eq:neg-score} with a single-particle surrogate, namely a conditional score, a Gaussian-transition ratio, or nearby reference particles.
The repulsion-strength hyperparameter exposed by each method thus carries no density-level meaning, and the induced bias cannot be bounded or reduced by additional computation.
\autoref{sec:method} instead estimates $p_t(c\mid x_t)$ from multiple particles via online PU learning and corrects residual bias with SMC, targeting \eqref{eq:target} with $\lambda$ as the sole density-level knob.

\section{Proposed Method: PU Steering}
\label{sec:method}

The limitations above have a common root: without a global estimate of $p_t(c\mid x_t)$, no per-particle update can target~\eqref{eq:target} as an explicit density.
\emph{PU Steering} closes this gap by combining (i)~a tempered SMC path toward~\eqref{eq:target},
(ii)~an \emph{online} PU classifier $f_t\approx p_t(c\mid x_t)$ fit from the running particles,
and (iii)~a particle-density correction that makes those particles unbiased unlabeled data for the PU step.
Any existing negative-guidance method plugs in as the proposal kernel, and an analysis-friendly form of the sampler is asymptotically exact for~\eqref{eq:target} (\autoref{sec:theory}).
\autoref{sec:method:framework} presents the meta-algorithm; \autoref{sec:method:ref} and \autoref{sec:method:score} specialize it to the reference-sample and score-based regimes.

\subsection{Common Framework}
\label{sec:method:framework}

PU Steering combines three components, described in turn below:
an SMC backbone, a PU posterior estimator that yields $f_t$ online,
and an importance-weight correction that bridges them.

\paragraph{SMC backbone.}
SMC targets a sequence $\gamma_T, \gamma_{T-1}, \dots, \gamma_0$ in turn via a weighted particle system,
alternating at each step a \emph{proposal} that moves each particle,
and a \emph{weight update} that retargets it to the next $\gamma_t$.
Following test-time alignment of diffusion models via SMC \citep{NEURIPS2023_63e8bc7b,kim2025testtime} (recalled in Appendix~\ref{app:prelim:das}),
we deploy SMC along the reverse diffusion with the tempered path
\begin{equation}
\label{eq:density-path}
\gamma_t(x_t) \;\propto\; p_{\theta,t}(x_t)\,(1 - f_t(x_t))^{\lambda_t},
\qquad
0 = \lambda_T \le \ldots \le \lambda_0 = \lambda,
\end{equation}
realizing our target~\eqref{eq:target} at $t = 0$ whenever $f_0 = \note{p_{\theta,0}(c\mid \cdot)}$.
Here $p_{\theta,t}$ denotes the marginal of the pretrained reverse process in \autoref{sec:existing}, so $p_{\theta,t}(x_t) \approx p_t(x_t)$ and $\gamma_T(x_T) = \mathcal N(x_T; 0, I)$.
\note{All operations of PU Steering below are defined on this path of the pretrained sampler: posteriors and class-conditional densities are those of $p_{\theta,t}$, and $p_{\mathrm{target}}$ denotes~\eqref{eq:target} with $p$ replaced by $p_{\theta,0}$.}
Unlike \cite{NEURIPS2023_63e8bc7b,kim2025testtime}, where the potential defining $\gamma_t$ is pre-specified, we estimate $f_t$ online from the running particles.
\autoref{sec:theory} shows that the marginal of $\gamma_0$ depends only on the terminal $f_0$,
so neither accurate $\nabla_{x_t} f_t$ for all $t$ nor accurate $f_t$ at $t > 0$ is required.

Concretely, $M$ particles are initialized at $t = T$ as $\{x_T^j\}_{j=1}^M \stackrel{\mathrm{i.i.d.}}{\sim} \mathcal N(0, I)$ with $w_{q,T}^j = 1$ for all $j$.
At each subsequent step, a Gaussian proposal kernel 
$m_{t-1}(\,\cdot \mid x_t) = \mathcal N(\mu^{\mathrm{prop}}_t,\,\sigma_t^2 I)$
whose mean $\mu^{\mathrm{prop}}_t$ is specified per setting in \autoref{sec:method:ref} and \autoref{sec:method:score}, draws $x_{t-1}^j$;
the weight update from $\gamma_t$ to $\gamma_{t-1}$ then admits the closed-form recursion
\begin{align}
\label{eq:wq-recursion}
w_{q,t-1} &\;=\; w_{q,t}\;r_t
\frac{(1-f_{t-1}(x_{t-1}))^{\lambda_{t-1}}}{(1-f_t(x_t))^{\lambda_t}},
\; \ \text {where} \nonumber \\
r_t
&\;:=\; \frac{\exp\bigl(-\|x_{t-1}-\mu_\theta(x_t,t)\|^2 / 2\sigma_t^2\bigr)}{\exp\bigl(-\|x_{t-1}-\mu^{\mathrm{prop}}_t\|^2 / 2\sigma_t^2\bigr)}
\;=\; \frac{p_\theta(x_{t-1}\mid x_t)}{m_{t-1}(x_{t-1}\mid x_t)},
\end{align}
i.e., $r_t$ is the density ratio between the pretrained reverse kernel and the proposal;
the trailing classifier ratio in~\eqref{eq:wq-recursion} accounts for the change of tilting from $\gamma_t$ to $\gamma_{t-1}$ (derivation in Appendix~\ref{app:prelim:das}).

\paragraph{PU posterior estimator.}
Estimating $\note{p_{\theta,t}(c\mid x_t)}$ is a binary classification problem with positive class $c$ and negative class $\neg c$,
but no samples from $\note{p_{\theta,t}(\cdot\mid \neg c)}$ are directly available in either of our settings.
What both settings provide is (i)~positive samples from $\note{p_{\theta,t}(\cdot\mid c)}$ and (ii)~the pretrained unconditional marginal $p_{\theta,t}$,
which mixes the two class-conditional distributions at the prior rate $\pi = p(c)$:
\[
p_{\theta,t}(x_t) 
= \pi\,\note{p_{\theta,t}(x_t\mid c)} + (1-\pi)\,\note{p_{\theta,t}(x_t\mid \neg c)}.
\]
This identity recasts the binary classification risk using only samples from $p(\cdot\mid c)$ and unconditional samples alone,
giving the unbiased PU (uPU) risk \citep{NIPS2014_f032bc3f,pmlr-v37-plessis15}
\begin{equation}
\label{eq:upu}
\mathcal L_{\mathrm{uPU}}(f)
= \pi\,\mathbb E_{p(x\mid c)}[\ell(f(x), +1)]
+ \mathbb E_{p(x)}[\ell(f(x), -1)]
- \pi\,\mathbb E_{p(x\mid c)}[\ell(f(x), -1)],
\end{equation}
where $\mathbb{E}$ denotes the expectation. 
This risk equals the standard binary classification risk (derivation in Appendix~\ref{app:prelim:pu}).
Minimizing it with a proper loss yields $f_t \approx \note{p_{\theta,t}(c\mid x_t)}$,
provided $\pi$ is known or estimable from domain knowledge.
The SMC particles, however, do not follow $p_{\theta,t}$ and thus cannot directly serve as unlabeled samples in \eqref{eq:upu}.

\paragraph{Particle density correction.}
We correct this mismatch by importance reweighting \citep{Sugiyama+Kawanabe:2012}. 
The density $\hat\gamma_t$ induced by the $M$ particles under the proposal kernel and resampling rule has no closed form.
\note{Our key observation concerns} the importance weight retargeting it to $p_{\theta,t}$, \note{conceptually defined as}
\begin{equation}
\label{eq:w-empty}
w_{\emptyset,t}(x_t) \;:=\; \frac{p_{\theta,t}(x_t)}{\hat\gamma_t(x_t)}
\qquad (\emptyset\text{ marks the unconditional target}).
\end{equation}
Setting $\lambda_t \equiv 0$ in~\eqref{eq:wq-recursion} collapses $\gamma_t$ to $p_{\theta,t}$ and the classifier ratio to one, \note{giving its closed form}
\begin{equation}
\label{eq:wempty-recursion}
w_{\emptyset,t-1} \;=\; w_{\emptyset,t}\,r_t,
\qquad w_{\emptyset,T} = 1;
\end{equation}
i.e., $w_{\emptyset,t}$ is itself the SMC weight that would target $p_{\theta,t}$ along the \emph{same} particle trajectory as $w_{q,t}$.
The two recursions share $r_t$ and differ only in the tilting factor $(1-f_t)^{\lambda_t}$, so $w_{\emptyset,t}$ is tracked alongside $w_{q,t}$ at no additional cost.
Empirically approximating $\mathbb E_{p(x)}$ in~\eqref{eq:upu} with $\{(x_t^j, w_{\emptyset,t}^j)\}_{j=1}^M$ restores unbiasedness of the uPU risk used to fit $f_t$.
This bridges the SMC sampler and the PU estimator: 
\eqref{eq:wq-recursion} drives particles toward $\gamma_t$ for sampling, 
while~\eqref{eq:wempty-recursion} pulls them back to $p_{\theta,t}$ inside the uPU loss for $f_t$.

\begin{algorithm}[t]
\caption{PU Steering: one reverse-diffusion step ($t \to t{-}1$).}
\label{alg:pu-steer-meta}
\begin{algorithmic}[1]
\Require particles $\{(x_t^j, w_{q,t}^j, w_{\emptyset,t}^j)\}_{j=1}^M$,\ 
         classifier $f_t$,\ schedule entry $\lambda_{t-1}$,\ class prior $\pi$.

\Statex \hspace*{-0.5em}\textit{\color[gray]{0.6}\small\# Propose particles via any negative-guidance kernel}
\State $x_{t-1}^j \sim m_{t-1}(\,\cdot \mid x_t^j),\ j = 1,\dots,M$
       \Comment{e.g.,~\eqref{eq:prop-fixed}}

\Statex \hspace*{-0.5em}\textit{\color[gray]{0.6}\small\small\# Advance $w_\emptyset$ via the closed-form recursion}
\State update $w_{\emptyset,t-1}^j$ toward $p_{\theta,t-1}$
       \Comment{\eqref{eq:wempty-recursion}}

\Statex \hspace*{-0.5em}\textit{\color[gray]{0.6}\small\small\# Refit the PU classifier on new particles}
\State Sample positives $\mathcal P_{t-1}$ from $\note{p_{\theta,t-1}(\,\cdot \mid c)}$
       \Comment{\textit{setting-specific}: \autoref{sec:method:ref}/\autoref{sec:method:score}}
\State $f_{t-1} \leftarrow \arg\min_f \mathcal L_{\mathrm{uPU}}(f)$ on 
       $\mathcal P_{t-1}$ and $\{(x_{t-1}^j, w_{\emptyset,t-1}^j)\}_{j=1}^M$
       \Comment{\eqref{eq:upu}}

\Statex \hspace*{-0.5em}\textit{\color[gray]{0.6}\small\small\# Reweight $w_q$ via the shared closed-form recursion}
\State update $w_{q,t-1}^j$ toward $\gamma_{t-1}$
       \Comment{\eqref{eq:wq-recursion}}

\Statex \hspace*{-0.5em}\textit{\color[gray]{0.6}\small\small\# Resample only when particles degenerate}
\State \textbf{if} $\mathrm{ESS}(\{w_{q,t-1}^j\}) < \mathrm{ESS}_{\min}$ \textbf{then}
       resample $\propto w_{q,t-1}^j$ and reweight $w_{\emptyset,t-1}^j$ accordingly
\end{algorithmic}
\end{algorithm}

\paragraph{Meta-algorithm.}
\autoref{alg:pu-steer-meta} summarizes one timestep of PU Steering, 
which propagates particles through a negative-guidance proposal, 
refits the PU classifier on the reweighted particles, 
and updates the target weights before resampling. 
The full procedure initializes 
$\{x_T^j\}_{j=1}^M \stackrel{\mathrm{i.i.d.}}{\sim} \mathcal N(0, I)$, 
$w_{q,T}^j = w_{\emptyset,T}^j = 1$, and any $f_T$, 
and then applies \autoref{alg:pu-steer-meta} for $t = T, T-1, \ldots, 1$.
The two settings in \autoref{sec:method:ref} and \autoref{sec:method:score} 
differ only in the proposal kernel $m_{t-1}$ and in how positives $\mathcal P_{t-1}$ are obtained;
all other steps are shared.

Negative guidance ultimately requires a particle set in which the undesired condition $c$ is suppressed, not a weighted empirical distribution dominated by a few high-weight particles.
\emph{Resampling} enforces this: ancestors $a(j) \propto w_{q,t-1}^j$ are drawn,
each particle is replaced by $x_{t-1}^{a(j)}$, $w_{q,t-1}^j$ is reset to $1$, 
and $w_{\emptyset,t-1}^j$ is rescaled to 
$w_{\emptyset,t-1}^{a(j)} / w_{q,t-1}^{a(j)}$ to preserve the ratio interpretation~\eqref{eq:w-empty}.
Following \cite{chopin2020introduction}, we trigger resampling only when the effective sample size 
$\mathrm{ESS}(\{w_{q,t-1}^j\}) := \bigl(\sum_j w_{q,t-1}^j\bigr)^2 
\big/ \sum_j (w_{q,t-1}^j)^2$ 
falls below $\mathrm{ESS}_{\min}$,
since unnecessary resampling adds Monte Carlo noise and computational cost without benefit.

\subsection{Reference-Based Setting}
\label{sec:method:ref}

We instantiate \autoref{alg:pu-steer-meta} when the undesired condition $c$ is specified by a finite reference set $\{x_c^i\}_{i=1}^N$ \note{of clean samples}.
The setting-specific slots are the positive set $\mathcal P_{t-1}$ in Step~3 and the proposal mean $\mu_t^{\mathrm{prop}}$ in Step~1.
Detailed pseudocode is given as \autoref{alg:neg_das_pu_fixed} in \autoref{app:algorithms}.

\paragraph{PU classifier.}
Only the terminal $f_0$ controls the marginal of $\gamma_0$ (\autoref{sec:method:framework}), so $f_t$ for $t > 0$ need not be highly accurate.
Classification of $c$ is also intrinsically easier on clean samples than on noised ones,
since the forward kernel of \eqref{eq:forward-kernel} destroys class-discriminative structure as $t$ grows. 
Both observations motivate the parameterization
\[
f_t(x_t) \;:=\; g_0\bigl(\hat x_0(x_t)\bigr), 
\qquad g_0(x_0) \approx \note{p_{\theta,0}(c \mid x_0)}.
\]
At $t = 0$, $\hat x_0(x_0) = x_0$ gives $f_0 \equiv g_0$ and 
$\gamma_0(x_0) \propto \note{p_{\theta,0}(x_0)}\bigl(1 - g_0(x_0)\bigr)^\lambda$, recovering $p_{\mathrm{target}}$.

Because $g_0$ acts on inputs in the denoised space,
the uPU loss must be estimated with both positives and unlabeled samples projected through $\hat x_0$. We accordingly set
\[
\mathcal P_{t-1} = \bigl\{\hat x_0(\tilde x_{t-1}^i)\bigr\}_{i=1}^N,
\qquad
\mathcal U_{t-1} = \bigl\{\bigl(\hat x_0(x_{t-1}^j),\, w_{\emptyset,t-1}^j\bigr)\bigr\}_{j=1}^M,
\]
where $\tilde x_{t-1}^i \sim q_{t-1}(\cdot \mid x_c^i)$ is a forward-noised reference,
so that positives and unlabeled samples are aligned on the same $\hat x_0$-pushforward at each level $t$ and reduce to the original references at $t=0$.
Since $g_0$ is the only trainable object, the per-step refit reduces to a few warm-started gradient steps on its parameters,
with $f_{t-1} = g_0 \circ \hat x_0$ updated implicitly.

\paragraph{SMC proposal.}
\textit{(i) CG-style proposal driven by $f_t$.}
\begin{equation}
\label{eq:prop-fixed}
\mu_t^{\mathrm{prop}}
\;:=\;
\mu_\theta(x_t, t)
\,+\,
\sigma_t^2\,\lambda_{t-1}\,
\nabla_{x_t} \log\bigl(1 - f_t(x_t)\bigr).
\end{equation}
Under $f_t = g_0 \circ \hat x_0$ the gradient routes through Tweedie's formula,
making \eqref{eq:prop-fixed} a CG variant that lifts a clean-domain predictor through $\hat x_0$\citep{chung2023diffusion,Bansal_2023_CVPR,Yu_2023_ICCV}.
The difference is that no accurate predictor is available at the start of sampling in our setting, so $g_0$ is fit online via PU learning.

\textit{(ii) Existing reference-based guidance as proposal.}
Methods such as the Safe Denoiser \citep{kim2025trainingfree} provide efficient guidance but target no explicit density.
Using their guidance as $\mu_t^{\mathrm{prop}}$ leaves the per-step update intact while the SMC layer ensures asymptotic exactness for $p_{\mathrm{target}}$.

\subsection{Score-Based Setting}
\label{sec:method:score}

We instantiate \autoref{alg:pu-steer-meta} when $c$ is specified by a conditional noise predictor $\hat\epsilon_\theta(\cdot, \cdot \mid c)$ but no reference set is given.
The parameterization $f_t = g_0 \circ \hat x_0$ and the two proposal options of \autoref{sec:method:ref} carry over unchanged.
The non-trivial step is constructing $\mathcal P_{t-1}$ in Step~3 of \autoref{alg:pu-steer-meta} without samples from $p_0(\cdot \mid c)$.
Our key observation is that the particle-density correction of \autoref{sec:method:framework} extends directly to the conditional marginal once its score is available,
supplying positives from the same particle trajectory.
Detailed pseudocode is given as \autoref{alg:neg_das_pu_singlechain} in \autoref{app:algorithms}.

\paragraph{Positives via importance reweighting.}
Applying the construction of $w_{\emptyset,t}$ in 
\autoref{sec:method:framework} to the conditional target gives an analogous weight that retargets the particles to $p_{\theta,t}(\cdot \mid c)$,
\[
w_{c,t}(x_t) \;:=\; \frac{p_{\theta,t}(x_t \mid c)}{\hat\gamma_t(x_t)}.
\]
The same particles serve as positives and unlabeled data under different weights:
\[
\mathcal P_{t-1} = \bigl\{\bigl(\hat x_0(x_{t-1}^j),\, w_{c,t-1}^j\bigr)\bigr\}_{j=1}^M, 
\quad 
\mathcal U_{t-1} = \bigl\{\bigl(\hat x_0(x_{t-1}^j),\, w_{\emptyset,t-1}^j\bigr)\bigr\}_{j=1}^M.
\]
Sharing the particle set across the two terms of $\mathcal L_{\mathrm{uPU}}$ preserves unbiasedness in the particle limit (\autoref{sec:theory}).
Step~2 of \autoref{alg:pu-steer-meta} then carries a recursion for $w_{c,t}$ alongside the existing $w_{\emptyset,t}$ update;
both share the same closed-form recursion as in \autoref{sec:method:framework}, with $\mu_\theta(x_t,t)$ replaced by $\mu_\theta(x_t,t \mid c)$.

\paragraph{SMC proposal.}
Both options of \autoref{sec:method:ref} extend directly: 
\textit{(i)} the CG-style mean \eqref{eq:prop-fixed} driven by the $f_t$,
or \textit{(ii)} an existing score-based negative-guidance method, such as DNG \citep{koulischer2025dynamic}, plugged in as $\mu_t^{\mathrm{prop}}$,
with our SMC layer correcting its density-level bias against $p_{\mathrm{target}}$.
When the proposal already evaluates the conditional score, the $w_{c,t}$ recursion reuses that evaluation at no extra cost.

\section{Theoretical Analysis}
\label{sec:theory}

We summarize the convergence result for the reference-based instantiation of \autoref{alg:pu-steer-meta} (\autoref{sec:method:ref}) and extend the result to the score-based setting (\autoref{sec:method:score}).
Formal statements and proofs are in \autoref{app:full-proof}.

For implementation efficiency, \autoref{alg:pu-steer-meta} uses a single particle system to both train $f_t$ and form the final SMC estimator.
For analytical clarity, we instead train $f_t$ on an independent auxiliary system $\{x_t^k\}_{k=1}^{M'}$ driven by the same kernels,
which removes the dependence between the classifier and the main SMC particles 
$\{(x_t^j, w_{q,t}^j)\}_{j=1}^M$, whose weighted empirical measure we denote by $\widehat\nu_t^{\,M\mid M',N}$.
We establish asymptotic exactness in the resulting infinite-particle limit $\lim_{M',N\to\infty}\lim_{M\to\infty}$,
in which the inner $M\to\infty$ controls the main-system SMC error at frozen classifiers and the outer $M',N\to\infty$ drives the classifier estimation error to zero.

\paragraph{Assumptions.}
Let $\mathcal G_t$ denote the function class over which $f_t$ is fit.
We require four conditions on $\mathcal G_t$ and on the \note{posterior $p_{\theta,t}(c\mid x)$}.
\textit{(i) Uniform convergence on empirical positive samples}:
$\sup_{g\in\mathcal G_t}\bigl|\tfrac{1}{N}\sum_{i=1}^N g(X_i^+)
-\mathbb E_{\note{p_{\theta,t}(\cdot\mid c)}}[g]\bigr|\to 0$ almost surely as $N\to\infty$, with
$X_i^+\stackrel{\mathrm{i.i.d.}}{\sim}\note{p_{\theta,t}(\cdot\mid c)}$.
\textit{(ii) Realizability of the Bayes logit}:
$\log\{\note{p_{\theta,t}(c\mid x)}/(1-\note{p_{\theta,t}(c\mid x)})\}\in\mathcal G_t$.
\textit{(iii) Two-sided non-degeneracy}:
$\delta\le \note{p_{\theta,t}(c\mid x)}\le 1-\delta$ for some $\delta\in(0,1)$.
\textit{(iv) Total boundedness} of $\mathcal G_t$ in the supremum norm.

\paragraph{Reference-based setting.}
The $N$ i.i.d.\ references $\{x_c^i\}_{i=1}^N$ supply the positive term of the uPU risk~\eqref{eq:upu},
and the $M'$ auxiliary particles supply the unlabeled term.
We argue in three steps.
\textbf{Step~1}, fix timestep $t$ and assume that the unlabeled empirical average converges uniformly to $\mathbb E_{p_{\theta,t}}$ over $\mathcal G_t$.
Then uPU risk minimization \citep{pmlr-v37-plessis15}, combined with (i) and (ii), 
yields $L^1$-convergence $f_t\to \note{p_{\theta,t}(c\mid\cdot)}$ as $M',N\to\infty$. 
We will invoke this at $t=0$.
\textbf{Step 2}, this convergence does hold: under (iii) and (iv), and with a mild proposal regularization 
(Appendix~\ref{app:regularized-proposal}) that keeps the importance weights bounded,
the weighted empirical measure of the auxiliary particles $\{(x_t^k, w_{\emptyset,t}^{\note{k}})\}_{k=1}^{M'}$ \note{converges} uniformly over $\mathcal G_t$ to $p_{\theta,t}$ by the SMC law of large numbers (LLN) in \citep{chopin2020introduction}.
\textbf{Step 3}, the tempering factors $(1-f_t)^{\lambda_t}$ enter the recursion~\eqref{eq:wq-recursion} as ratios between consecutive times,
so the cumulative product telescopes and only the terminal $f_0$ enters $\gamma_0$, similar to \citep{kim2025testtime}.
Combining these three steps yields \autoref{thm:main-informal}.

\begin{theorem}[Asymptotic exactness, informal version of \autoref{thm:plugin_smc_exactness} in Appendix~\ref{app:plug-in-smc}]
\label{thm:main-informal}
Consider \autoref{alg:pu-steer-meta} with the auxiliary $M'$ system introduced above. For every bounded measurable
$\varphi:\mathcal X_0\to\mathbb R$,
\[
  \lim_{M',N\to\infty}\;\lim_{M\to\infty}\;
  \widehat\nu_0^{\,M\mid M',N}(\varphi)
  \;=\;p_{\mathrm{target}}(\varphi)\quad\text{almost surely,}
\]
under (i)--(iv) for the reference-based setting.
\end{theorem}

\paragraph{Extension to score-based setting.}
Without external references, the positive term of the uPU risk~\eqref{eq:upu} is supplied by the same $M'$ auxiliary particles,
re-weighted by $w_{c,t}$ and $w_{\emptyset,t}$.
The argument mirrors the reference-based one, with only \textbf{Step 2} changed.
By an analogous SMC LLN \citep{chopin2020introduction},
the weighted empirical measure of $\{(x_t^k, w_{c,t}^{\note{k}})\}_{k=1}^{M'}$ converges uniformly over $\mathcal G_t$ to $p_{\theta,t}(\cdot\mid c)$
as $M'\to\infty$.
Combined with the unlabeled-term convergence from \textbf{Step 2} of the reference-based setting,
this gives uniform convergence of both terms of the uPU risk and hence $f_t\to \note{p_{\theta,t}(c\mid\cdot)}$ as $M'\to\infty$ at each timestep $t$.
\textbf{Step 1} and \textbf{Step 3} then carry over unchanged, 
yielding $\widehat\nu_0^{M\mid M'}(\varphi)\to p_{\mathrm{target}}(\varphi)$ under (ii)--(iv),
for every bounded measurable $\varphi$ (\autoref{thm:alg2-plugin-exactness} in Appendix~\ref{app:alg2-dual-weight-exactness}).

\section{Experiments}
\label{sec:experiments}

\begin{figure}[t]
\centering
\begin{minipage}[c]{0.5\linewidth}
  \centering
  \includegraphics[width=\linewidth]{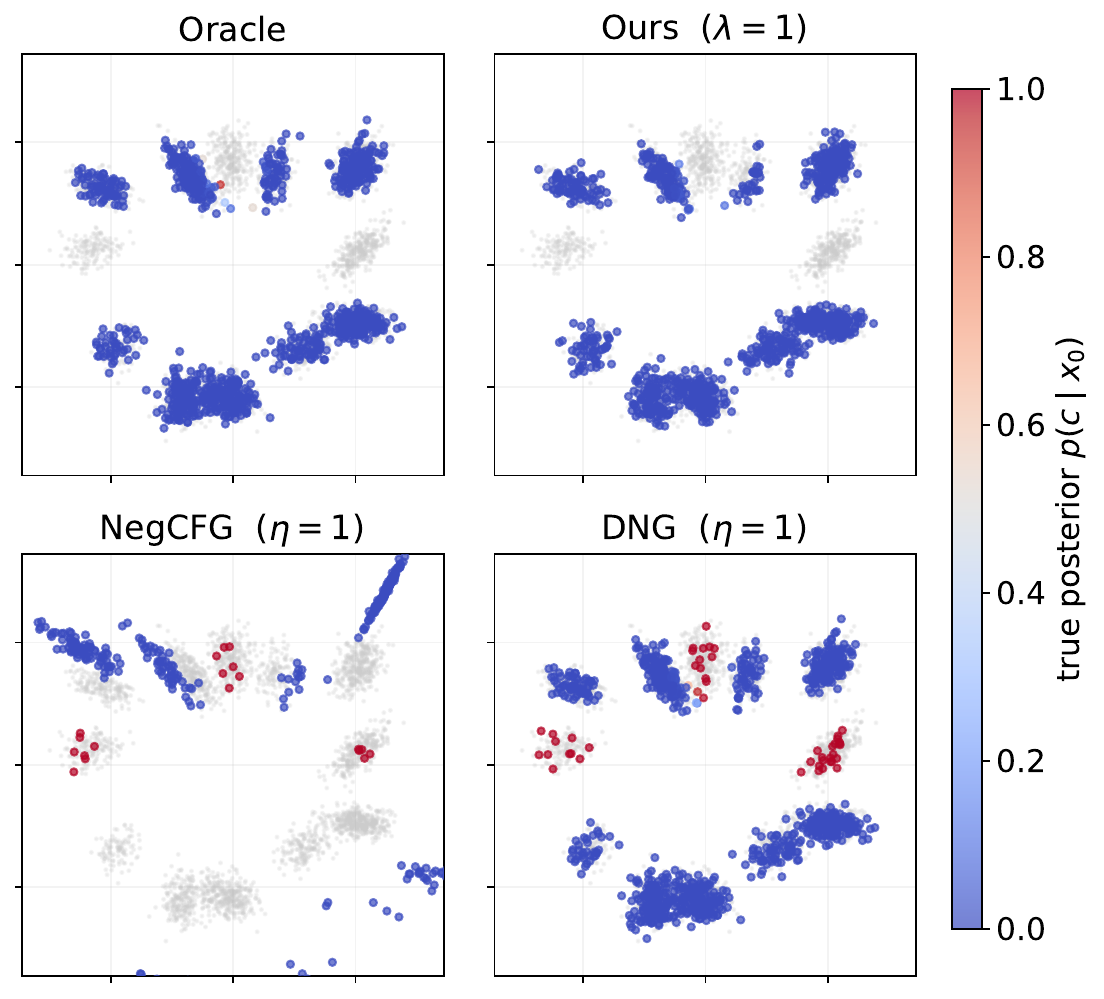}
\end{minipage}\hfill
\begin{minipage}[c]{0.5\linewidth}
  \centering
  \small
  \setlength{\tabcolsep}{4pt}
  \begin{tabular}{lcc}
    \toprule
    Method & Posterior leak\,$\downarrow$ & $W_2$\,$\downarrow$ \\
    \midrule
    Oracle & $0.012{\pm}0.010$ & $1.5{\pm}0.3$  \\
    Uncond ($\eta{=}0$) & $0.253{\pm}0.016$ & $3.6{\pm}0.6$ \\
    NegCFG ($\eta{=}0.5$) & $0.068{\pm}0.008$ & $9.7{\pm}2.4$ \\
    NegCFG ($\eta{=}1$) & $0.018{\pm}0.005$ & $20.5{\pm}6.3$ \\
    DNG ($\eta{=}1$) & $0.052{\pm}0.010$ & $1.6{\pm}0.3$ \\
    DNG ($\eta{=}2$) & $0.015{\pm}0.005$ & $3.5{\pm}0.7$ \\
    \textbf{Ours} & $\mathbf{0.012{\pm}0.013}$ & $\mathbf{1.8{\pm}0.6}$ \\
    \bottomrule
  \end{tabular}
  \phantomsection
  \captionof{table}{Comparison with baselines across repulsion strengths $\eta$, mean$\pm$std over $50$ trials.
  Uncond corresponds to NegCFG/DNG at $\eta{=}0$, which reduces to unconditional DDPM from $p(x)$. 
  Posterior leak $\mathbb{E}[p(c | x_0)]$ and $W_2$ measure $2$-Wasserstein distance to ground-truth. 
  Baselines show a trade-off between leak and $W_2$, while ours matches the Oracle in leak without inflating $W_2$.}
  \label{tab:mog}
\end{minipage}
\caption{Generated samples on one trial. Each point is a sample in its first two coordinates, colored by the true posterior $p(c\mid x)\in[0,1]$ (red: $p(c\mid x){=}1$, blue: $0$). Oracle and Ours avoid the condition $c$ and cover all remaining modes. NegCFG suppresses leak by collapsing onto a few remaining modes, and DNG keeps coverage but retains residual samples on the avoided modes.}
\label{fig:mog}
\end{figure}

This section summarizes our experimental results in \note{three settings of increasing scale}.
\autoref{sec:exp:gmm} uses the score-based setting (\autoref{sec:method:score}) on a 3-dimensional Gaussian mixture model (GMM),
where closed-form densities enable rigorous metric evaluation and direct visualization of the full distribution.
\autoref{sec:exp:cifar10} uses the reference-based setting (\autoref{sec:method:ref}) on CIFAR-10 image generation with a pretrained score network,
under different specifications of the avoidance condition $c$ (a class or set of classes to avoid).
\note{\autoref{sec:exp:sd} applies the score-based setting to text-to-image generation with Stable Diffusion.}
\note{All three} compare against the baselines from \autoref{sec:existing}, with full experimental details deferred to \autoref{app:exp}.
To align the evaluation protocol with these baselines,
we treat the final SMC output after intermediate resampling as particles with unit weight rather than as weighted particles.\footnote{\note{Code and experimental results are available at \url{https://github.com/htkg/pu-steering}.}}

\subsection{GMM with Score-Based Setting}
\label{sec:exp:gmm}

\paragraph{Setup.}
The unconditional density is a 3-dimensional GMM of $12$ modes,
with the avoided condition $c$ assigned to $3$ modes and $\neg c$ to the other $9$. 
The layout of the GMM is randomly drawn.
We define the ground-truth target by setting $\lambda{=}1$ in \eqref{eq:target},
giving $p_{\mathrm{target}}{=}p(\cdot\mid\neg c)$. 
All methods use a common $T{=}100$ DDPM discretization with $M{=}1024$ samples.
The Oracle alone has access to the true score for the target $\nabla \log p_t(x_t\mid\neg c)$,
serving as the achievable upper bound.
We compare against NegCFG and DNG from \autoref{sec:existing}, and our method uses a kernel-based PU classifier \citep{pmlr-v37-plessis15}.
We report posterior leak $\mathbb{E}[p(c\mid x_0)]$ (average posterior of $c$ at generated samples, lower is stronger avoidance) in closed form, and the $2$-Wasserstein distance ($W_2$, calculated by \citep{flamary2021pot}) to $3000$ i.i.d.\ draws from $p(\cdot\mid\neg c)$. 
Results are averaged over $50$ trials with the GMM layout and DDPM sampling seed both randomized.

\paragraph{Results.}
\autoref{fig:mog} shows the generated samples from each method on one trial.
NegCFG collapses onto a few $\neg c$-modes, and DNG retains samples on the avoided $c$-modes,
both consistent with the limitations of local repulsion away from $c$ discussed in \autoref{sec:existing}.
In contrast, our method preserves all nine $\neg c$-modes and matches the Oracle, 
leveraging global information across particles, consistent with \autoref{thm:main-informal}.
Table~\ref{tab:mog} reports the same trend quantitatively over $50$ trials.
Ours matches the Oracle in posterior leak (both $0.012$) and stays within one Oracle standard deviation in $W_2$ ($1.8{\pm}0.6$ vs.\ $1.5{\pm}0.3$), 
while NegCFG and DNG exhibit a trade-off between leak and $W_2$.
Full hyperparameter sweeps are in Appendix~\ref{app:exp:gmm}.

\subsection{CIFAR-10 with Reference-Based Setting}
\label{sec:exp:cifar10}

\begin{figure}[t]
\centering
\includegraphics[width=0.97\linewidth]{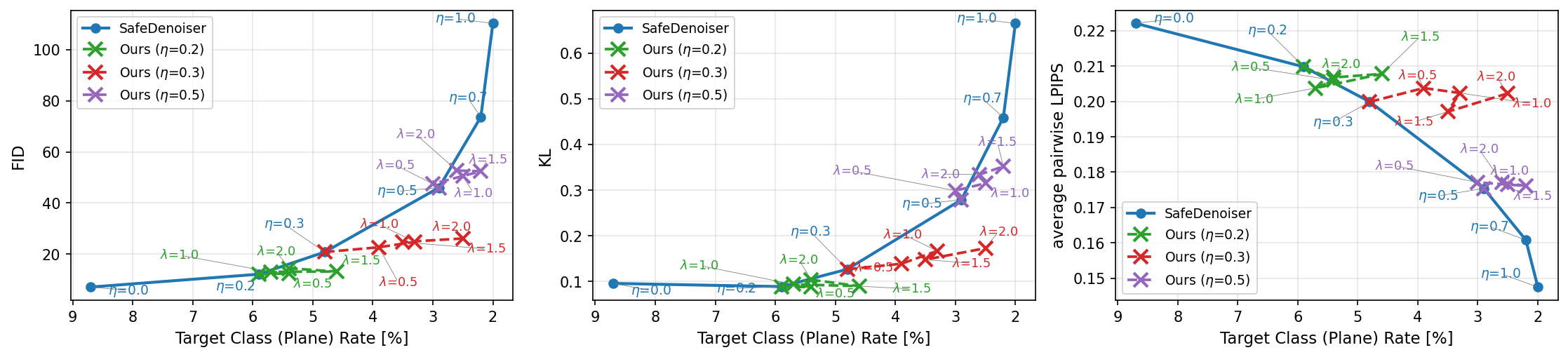}\\
\includegraphics[width=0.97\linewidth]{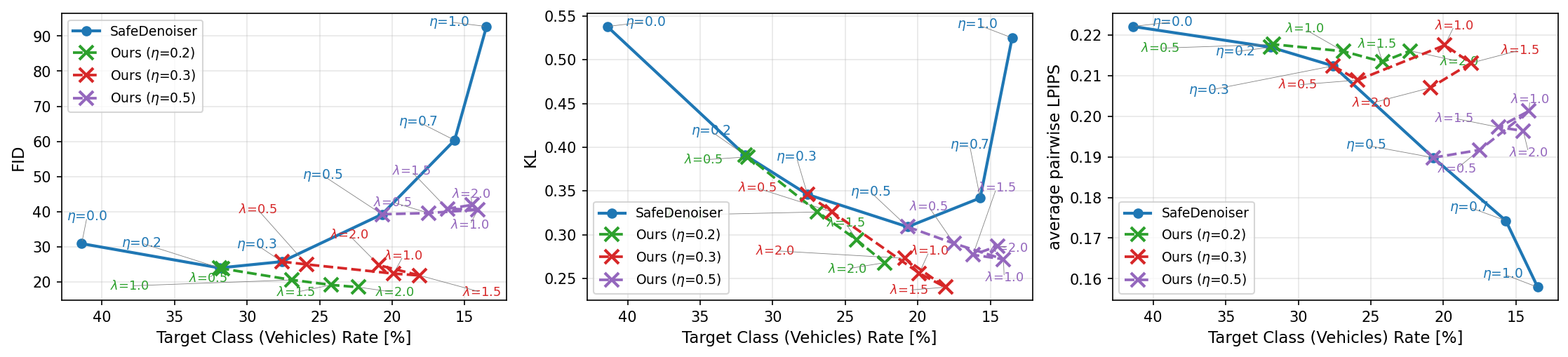}
\caption{Avoidance rate vs.\ generation quality on CIFAR-10. Top: plane. Bottom: four-class vehicles. Columns plot target-class rate against FID, KL, and average pairwise LPIPS (left to right); lower is better except for LPIPS. The Safe Denoiser sweeps $\eta$. Ours uses the Safe Denoiser as its proposal kernel, fixes $\eta\in\{0.2,0.3,0.5\}$, and sweeps $\lambda$, tracing the Pareto frontier in all panels.}
\label{fig:cifar10}
\end{figure}

\paragraph{Setup.}
We use a pretrained UNet for the unconditional CIFAR-10 diffusion model from \cite{pmlr-v139-nichol21a},
and run DDPM sampling with a batch size of $M{=}128$ over $T{=}250$ timesteps.
The avoided condition $c$ is one or more CIFAR-10 classes.
We evaluate plane (single class) and the four vehicle classes (plane, car, ship, truck).
For each class in $c$, we draw $N{=}200$ references from the training set.
The PU discriminator $f_t$ is a linear probe on pooled UNet bottleneck features.
Because $N$ and $M$ are both small here, we replace the uPU loss of \autoref{sec:method} with binary cross-entropy (BCE),
treating the references as positives and the SMC particles as soft negatives.
After resampling, with so few SMC particles in pixel space, those satisfying $c$ are quickly depleted, so the particle set increasingly behaves as a true negative set as $t$ decreases, making BCE an effective proxy in this regime.

\paragraph{Evaluation metrics.}
We aggregate $80$ seeds $\times\,128$ samples per configuration.
Fr\'echet Inception Distance (\emph{FID}, \citealp{heusel2017gans}) is computed against the CIFAR-10 training set from non-avoided classes, using the implementation of \cite{Seitzer2020FID}.
Each generated image is then classified by an independently trained ResNet-$18$~\citep{He_2016_CVPR} with $\sim$$95\%$ test accuracy into the class histogram $\hat p$,
from which we report the \emph{target-class rate} on avoided classes (lower is stronger avoidance) and the \emph{KL} divergence of $\hat p$ from a uniform over non-avoided classes (lower is more uniform).
As an auxiliary check, we also report average pairwise Learned Perceptual Image Patch Similarity (LPIPS; \citealp{Zhang_2018_CVPR}) via AlexNet \citep{NIPS2012_c399862d} on $256$ random samples (higher is more diverse).

\paragraph{Results.}
We first run the Safe Denoiser \citep{kim2025trainingfree} sweeping its repulsion strength $\eta$.
We then run our method (\autoref{sec:method:ref}), which uses the Safe Denoiser as its proposal $\mu_t^{\mathrm{prop}}$, at three fixed values $\eta\in\{0.2, 0.3, 0.5\}$ while sweeping the terminal tempering strength $\lambda$, so that both methods are compared on a shared $\eta$ axis.
\autoref{fig:cifar10} plots target-class rate against each quality metric for plane and four-class vehicles avoidance.
At each fixed $\eta$, our $\lambda$-sweep Pareto-dominates the Safe Denoiser at the same $\eta$ on FID and KL across the full operating range, and is no worse on LPIPS over the bulk of it.
A weak negative guidance leaves residual mass on the avoided classes that the classifier and resampling then correct, raising the avoidance rate without the quality degradation incurred by stronger negative guidance.
The added machinery beyond the Safe Denoiser is one linear probe on UNet features and ESS-based adaptive resampling on the classifier-driven weights.
Per-class results for cat, frog, and ship and regularization ablations are in Appendix~\ref{app:exp:cifar10}.

\note{
\subsection{Stable Diffusion with Score-Based Setting}
\label{sec:exp:sd}

\begin{figure}[t]
\centering
\begin{minipage}[c]{0.49\linewidth}
  \centering
  \includegraphics[width=\linewidth]{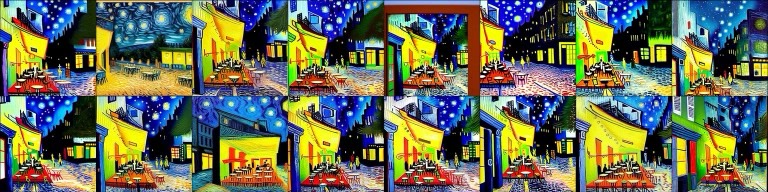}\\
  {\small (a) Avoided style}\\[10pt]
  \includegraphics[width=\linewidth]{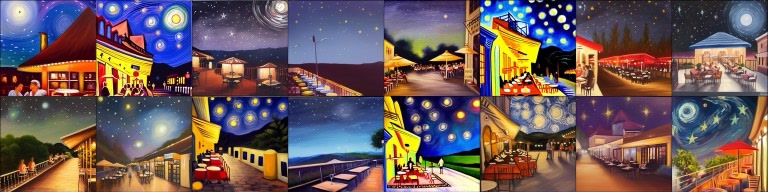}\\
  {\small (b) PU Steering (ours)}
\end{minipage}\hfill
\begin{minipage}[c]{0.49\linewidth}
  \centering
  \includegraphics[width=0.9\linewidth,trim=0 76 485 24,clip]{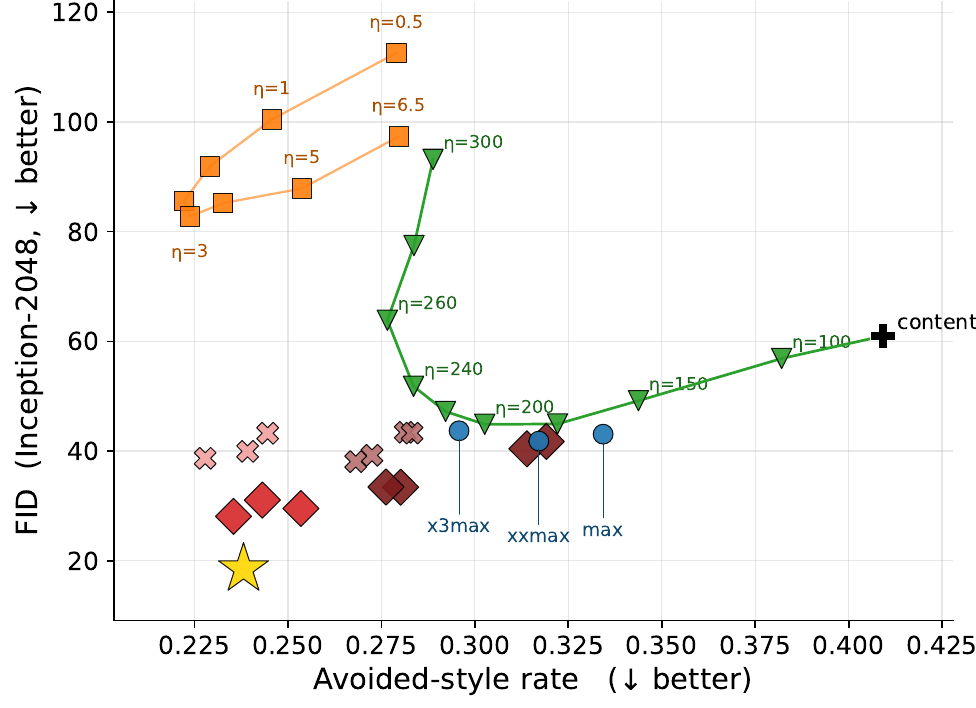}\\
  \includegraphics[width=0.9\linewidth,trim=259 5 253 386,clip]{sd_vangogh_frontier.pdf}
\end{minipage}
\caption{\note{Results of Van Gogh style avoidance on Stable Diffusion v1.4.
Left: generated examples of (a) the avoided style and (b) ours at $M{=}128$, $\lambda{=}16$,
each showing $16$ randomly selected images.
These are the top two rows of the random $8\times8$ grids in \autoref{fig:sd-grids} of Appendix~\ref{app:exp:sd},
which also shows the other methods.
Right: avoided-style rate against FID; lower left is better.
The baselines sweep their strengths, and ours uses SLD as its proposal kernel,
reaching the oracle-floor rate with the lowest FID while keeping the prompted scene.
Full results are in Appendix~\ref{app:exp:sd}.}}
\label{fig:sd-main}
\end{figure}

\paragraph{Setup.}
We use Stable Diffusion v1.4~\citep{Rombach_2022_CVPR}, a pretrained text-to-image latent diffusion model,
with classifier-free guidance at scale $7.5$ as the sampler of the base density $p_\theta$,
running a batch size of $M{=}128$ over $T{=}50$ timesteps.
The prompt \texttt{"a painting of a cafe terrace at night under a starry sky"} often yields Van Gogh's style,
which we specify as the avoided condition $c$ to suppress.
Appending \texttt{"in the style of Van Gogh"} to the prompt supplies the conditional score,
which generates images such as \autoref{fig:sd-main}(a),
and we use it to run the score-based setting of \autoref{sec:method:score}.
The PU classifier $f_t$ is a kernel logistic model on $\hat x_0(x_t)$.
Its positives and unlabeled data are $\hat x_0$ of the same particles under the style-appended and original prompts, respectively,
so they share the prompted content and differ mainly in the style.
This contrast keeps the fit stable, so we use the uPU loss~\eqref{eq:upu} as written, unlike the BCE in \autoref{sec:exp:cifar10}.

\paragraph{Evaluation metrics.}
We aggregate $2048$ images per configuration, generated with distinct seeds.
An independently trained detector, a logistic head on a style embedding~\citep{somepalli2024measuring},
assigns each image a probability of the avoided style, and we report its mean as the \emph{avoided-style rate} (lower is stronger avoidance).
As no samples of the target are available, we approximate them by filtering:
from $8192$ images of the prompt without avoidance, the $2048$ least like the style under the detector form the reference set,
against which FID and KID~\citep{binkowski2018demystifying} are computed (lower is better).
Applying the same filter to an independent pool generated with disjoint seeds gives the \emph{oracle floor},
a reference point for an ideal avoidance rather than a lower bound.

\paragraph{Results.}
We first run SLD~\citep{Schramowski_2023_CVPR}, a negative-guidance method specialized to Stable Diffusion, and DNG, each sweeping its strength.
We then run our method (\autoref{sec:method:score}), which uses SLD as its proposal $\mu_t^{\mathrm{prop}}$, while sweeping $\lambda$.
\autoref{fig:sd-main} (right) plots the avoided-style rate against FID.
DNG traces an avoidance--quality trade-off: at its best strengths it is comparable to SLD, and stronger guidance only degrades FID.
NegCFG lowers the rate only at FIDs above $80$.
Our method at $M{=}128$, $\lambda{=}16$ reaches a rate of $0.235$, at the oracle floor of $0.238$, with the lowest FID of $28.10$,
improving both avoidance and quality beyond SLD and DNG.
Halving the particle count to $M{=}64$ gives a rate of $0.276$ and an FID of $33.46$, so additional particles improve both.
Qualitatively, \autoref{fig:sd-main} (left) shows that our samples keep the prompted scene while suppressing the avoided style and spreading over a variety of painting styles.
The Stable Diffusion results are thus consistent with the CIFAR-10 results in \autoref{sec:exp:cifar10}.
The full setup and results are in Appendix~\ref{app:exp:sd}.

}

\section{Limitations and Future Work}
\label{sec:limitations}

\paragraph{\note{Scope of the experiments.}}
\note{Our experiments cover several settings up to image-generation scale (\autoref{sec:experiments}).
Extending the framework to the broader range of applications that motivate negative guidance (\autoref{sec:intro}) is left open.}
\note{Our} setting also differs from SMC-based diffusion alignment~\citep{NEURIPS2023_63e8bc7b,pmlr-v267-singhal25b,kim2025testtime},
which sharpens particles toward a desired condition.
PU Steering instead targets everything other than the avoided condition $c$,
which is harder to evaluate because no single reference distribution is available beyond a labeled validation set.
A parallel line of recent work develops SMC and other samplers that route interacting particles to a target along the diffusion path~\citep{wu2026reverse,takagi2026alternating}, 
training-free and with \note{convergence} guarantees on multi-modal targets.
These samplers are demonstrated on low- to moderate-dimensional settings ($d<100$),
and scaling them to image-scale dimensions is left open.
PU Steering reaches CIFAR-10 \note{and Stable Diffusion} by combining a pretrained UNet with an existing negative-guidance kernel~\citep{kim2025trainingfree,Schramowski_2023_CVPR} as the SMC proposal, and the same recipe applies to any pretrained diffusion model with such a kernel.
This connects principled probabilistic samplers to alignment of modern generative models.
We discuss future directions for improving and \note{specializing} our framework below\note{, which would also let us examine the effect of scaling the particle count}.

\paragraph{Variance of the importance weights.}
A known challenge in SMC is weight degeneracy,
in which the variance of the importance weights grows and the ESS decreases \citep{chopin2020introduction}.
This can lead to frequent resampling, and may eventually result in a few common ancestors dominating the particle set.
For a fixed target path, \cite{kim2025testtime} derives a proposal kernel that minimizes the one-step weight variance.
In our case, the target path itself depends on the online classifier $f_t$, and inaccuracies in $f_t$ further affect the weight variance.
Extending such an analysis is left for future work.
Task-specific regularization and other practical stabilization techniques are natural directions in this line.

\paragraph{Stabilizing the online PU classifier.}
Our method trains the classifier $f_t$ in an online manner.
Recent PU learning literature offers stabilization techniques for limited positive samples,
small-batch variability, and high-capacity models such as neural networks that are prone to overfitting~\citep{NIPS2017_7cce53cf,NEURIPS2021_47b4f1bf}.
Online and importance-weighted PU learning further address streaming data and weighted unlabeled samples~\citep{ijcai2020p311,pmlr-v258-kumagai25a}.
Incorporating these techniques may further stabilize our method.

\section{Conclusion}
\label{sec:conclusion}

In this paper, we addressed negative guidance of diffusion models,
the task of steering generation away from an undesired condition.
We recast the problem as sampling from an explicit avoidance density and realized it via tempered SMC driven by an online PU classifier,
where particles exploit global information across samples rather than local repulsion.
Our asymptotic exactness analysis implies that additional computation reduces distributional bias and improves generation quality.
On the practical side, we showed that augmenting existing guidance methods with a lightweight classifier and resampling improves the avoidance--quality trade-off.

\begin{ack}
MS was supported by JST ASPIRE Grant Number JPMJAP25B1.
\end{ack}

\bibliographystyle{plainnat}
\bibliography{main_validated_titlecase}

\clearpage

\input{checklist.tex}

\clearpage

\appendix

\input{appendix1}
\input{appendix2}
\input{appendix_proof}
\input{appendix_exp}

\section{Broader Impacts}
\label{app:broader-impacts}
This paper presents work whose goal is to advance the statistical foundations of negative guidance for diffusion models.
The same mechanism can support safety-related applications of pretrained generative models.
It can also in principle be repurposed to amplify rather than suppress an undesired condition.
Such repurposing does not exceed what is already afforded by the underlying pretrained model and existing guidance methods,
and our work does not increase the generative capability of the base model.
We release no new pretrained model or dataset, and use only publicly released assets.





\end{document}

%% file: checklist.tex
\section*{NeurIPS Paper Checklist}

\begin{enumerate}

\item {\bf Claims}
    \item[] Question: Do the main claims made in the abstract and introduction accurately reflect the paper's contributions and scope?
    \item[] Answer: \answerYes{}
    \item[] Justification: Our claims match the theoretical and experimental results.

\item {\bf Limitations}
    \item[] Question: Does the paper discuss the limitations of the work performed by the authors?
    \item[] Answer: \answerYes{}
    \item[] Justification: Limitations are discussed in \autoref{sec:limitations}

\item {\bf Theory assumptions and proofs}
    \item[] Question: For each theoretical result, does the paper provide the full set of assumptions and a complete (and correct) proof?
    \item[] Answer: \answerYes{}
    \item[] Justification: We provide full assumptions and complete proofs in the \autoref{app:full-proof}.

    \item {\bf Experimental result reproducibility}
    \item[] Question: Does the paper fully disclose all the information needed to reproduce the main experimental results of the paper to the extent that it affects the main claims and/or conclusions of the paper (regardless of whether the code and data are provided or not)?
    \item[] Answer: \answerYes{}
    \item[] Justification: All settings of the experiments are documented in \autoref{sec:experiments} and \autoref{app:exp}.

\item {\bf Open access to data and code}
    \item[] Question: Does the paper provide open access to the data and code, with sufficient instructions to faithfully reproduce the main experimental results, as described in supplemental material?
    \item[] Answer: \answerYes{}
    \item[] Justification: Code and per-experiment documentation are provided in the supplementary material. \note{A cleaned-up version is available at \url{https://github.com/htkg/pu-steering}.}

\item {\bf Experimental setting/details}
    \item[] Question: Does the paper specify all the training and test details (e.g., data splits, hyperparameters, how they were chosen, type of optimizer) necessary to understand the results?
    \item[] Answer: \answerYes{}
    \item[] Justification: The details of the implementation are presented in \autoref{sec:experiments} and \autoref{app:exp}.

\item {\bf Experiment statistical significance}
    \item[] Question: Does the paper report error bars suitably and correctly defined or other appropriate information about the statistical significance of the experiments?
    \item[] Answer: \answerYes{}
    \item[] Justification: For experiments on CPUs in \autoref{sec:exp:gmm}, results are reported as mean $\pm$ standard deviation (1-sigma) over a sufficient number of trials. For experiments on GPUs in \autoref{sec:exp:cifar10}, only the mean over a large sample size is reported, and error bars are not included mainly due to the lack of computational resources.

\item {\bf Experiments compute resources}
    \item[] Question: For each experiment, does the paper provide sufficient information on the computer resources (type of compute workers, memory, time of execution) needed to reproduce the experiments?
    \item[] Answer: \answerYes{}
    \item[] Justification: We explain the compute resources we used for experiments in \autoref{app:exp}.
    
\item {\bf Code of ethics}
    \item[] Question: Does the research conducted in the paper conform, in every respect, with the NeurIPS Code of Ethics \url{https://neurips.cc/public/EthicsGuidelines}?
    \item[] Answer: \answerYes{}
    \item[] Justification: The research conforms with the NeurIPS Code of Ethics.

\item {\bf Broader impacts}
    \item[] Question: Does the paper discuss both potential positive societal impacts and negative societal impacts of the work performed?
    \item[] Answer: \answerYes{}
    \item[] Justification: Broader impacts are discussed in \autoref{app:broader-impacts}.
    
\item {\bf Safeguards}
    \item[] Question: Does the paper describe safeguards that have been put in place for responsible release of data or models that have a high risk for misuse (e.g., pre-trained language models, image generators, or scraped datasets)?
    \item[] Answer: \answerNA{}
    \item[] Justification: We release no new pretrained model or dataset.

\item {\bf Licenses for existing assets}
    \item[] Question: Are the creators or original owners of assets (e.g., code, data, models), used in the paper, properly credited and are the license and terms of use explicitly mentioned and properly respected?
    \item[] Answer: \answerYes{}
    \item[] Justification: We cite the original references of the assets used in the paper, and document the assets actually used in \autoref{app:exp} and the files of the supplementary material.

\item {\bf New assets}
    \item[] Question: Are new assets introduced in the paper well documented and is the documentation provided alongside the assets?
    \item[] Answer: \answerNA{}
    \item[] Justification: We do not release new pretrained models or datasets.

\item {\bf Crowdsourcing and research with human subjects}
    \item[] Question: For crowdsourcing experiments and research with human subjects, does the paper include the full text of instructions given to participants and screenshots, if applicable, as well as details about compensation (if any)? 
    \item[] Answer: \answerNA{}
    \item[] Justification: The paper does not involve crowdsourcing or research with human subjects.

\item {\bf Institutional review board (IRB) approvals or equivalent for research with human subjects}
    \item[] Question: Does the paper describe potential risks incurred by study participants, whether such risks were disclosed to the subjects, and whether Institutional Review Board (IRB) approvals (or an equivalent approval/review based on the requirements of your country or institution) were obtained?
    \item[] Answer: \answerNA{}
    \item[] Justification: The paper does not involve crowdsourcing or research with human subjects.

\item {\bf Declaration of LLM usage}
    \item[] Question: Does the paper describe the usage of LLMs if it is an important, original, or non-standard component of the core methods in this research? Note that if the LLM is used only for writing, editing, or formatting purposes and does \emph{not} impact the core methodology, scientific rigor, or originality of the research, declaration is not required.
    \item[] Answer: \answerNA{}
    \item[] Justification: LLMs are not a component of the core methodology. We use LLMs only for writing polish and coding assistance.

\end{enumerate}

%% file: appendix1.tex
\section{Extended Preliminaries}
\label{app:prelim}

This appendix provides standard background and derivations on three prerequisite topics---diffusion models \citep{song2021scorebased,ho2020denoising},
Positive--Unlabeled (PU) learning \citep{NIPS2014_f032bc3f,pmlr-v37-plessis15},
and Diffusion Alignment \note{as} Sampling (DAS; \citealp{kim2025testtime})---that the main text (\autoref{sec:existing} and \autoref{sec:method}) treats briefly.

\subsection{Background on Diffusion Models}
\label{app:prelim:diffusion}

Diffusion models \citep{ho2020denoising,song2021scorebased} are a class of generative models that learn to generate data by reversing a gradual noising process applied to training samples.

\paragraph{(i) Forward process.}
Let $\{\alpha_t\}_{t=1}^T$ be a noise schedule with $\alpha_t = 1-\beta_t \in (0,1)$ and cumulative product $\bar\alpha_t = \prod_{s=1}^t \alpha_s$.
The \emph{forward Markov chain} is
\[
q(x_t\mid x_{t-1}) \;=\; \mathcal N\bigl(x_t;\,\sqrt{\alpha_t}\,x_{t-1},\;\beta_t I\bigr),
\qquad t=1,\dots,T,
\]
which by Gaussian convolution admits the closed-form conditional marginal
\[
q_t(x_t\mid x_0) \;=\; \mathcal N\bigl(x_t;\,\sqrt{\bar\alpha_t}\,x_0,\;(1-\bar\alpha_t) I\bigr).
\]
The schedule is chosen so that $\bar\alpha_T\approx 0$, ensuring $p_T(x_T)\approx\mathcal N(x_T;\,0,\,I)$.

\paragraph{(ii) Denoising score matching and the noise-predictor identity.}
The \emph{score} of the time-$t$ marginal is $\nabla_{x_t}\log p_t(x_t)$.
Diffusion models are trained with the \emph{denoising score matching} objective \citep{song2021scorebased,ho2020denoising}
\[
\min_\theta \;\;
\mathbb E_{x_0\sim p,\;\epsilon\sim\mathcal N(0,I),\;t\sim\mathrm{Unif}\{1,\dots,T\}}
\Bigl[
\bigl\|\hat\epsilon_\theta(x_t,t) - \epsilon\bigr\|^2
\Bigr],
\qquad
x_t \;=\; \sqrt{\bar\alpha_t}\,x_0 + \sqrt{1-\bar\alpha_t}\,\epsilon.
\]
The optimum $\hat\epsilon_\theta^\star$ predicts the noise $\epsilon$ that was added to $x_0$ to produce $x_t$,
and differentiating the log of the Gaussian density $q_t(x_t\mid x_0)$ gives
$\nabla_{x_t}\log q_t(x_t\mid x_0) = -\epsilon/\sqrt{1-\bar\alpha_t}$.
Taking the conditional expectation over $x_0\mid x_t$ and using
$\nabla_{x_t}\log p_t(x_t) = \mathbb E_{x_0\mid x_t}[\nabla_{x_t}\log q_t(x_t\mid x_0)]$ yields the identity
\[
\hat\epsilon_\theta^\star(x_t,t) \;=\; -\sqrt{1-\bar\alpha_t}\,\nabla_{x_t}\log p_t(x_t)\note{.}
\]

\paragraph{(iii) Tweedie's formula.}
For $Y = \sqrt{\bar\alpha_t}\,X + \sqrt{1-\bar\alpha_t}\,Z$ with $Z\sim\mathcal N(0,I)$,
Tweedie's formula \citep{efron2011tweedie} gives the posterior mean
\[
\mathbb E[X\mid Y] \;=\; \frac{Y + (1-\bar\alpha_t)\,\nabla_Y\log p_Y(Y)}{\sqrt{\bar\alpha_t}}.
\]
Substituting $Y=x_t$, $X=x_0$ and the score identity from~(ii) yields
\[
\hat x_0(x_t) \;=\; \frac{x_t - \sqrt{1-\bar\alpha_t}\,\hat\epsilon_\theta(x_t,t)}{\sqrt{\bar\alpha_t}}
\;\approx\; \mathbb E[x_0\mid x_t].
\]
Many downstream tasks require evaluating a function $g_0$ that is defined on clean samples $x_0$,
but one only has access to noisy samples $x_t$ during the reverse process.
Using $\hat x_0(x_t)$ as a proxy for $x_0$, one can \emph{lift} $g_0$ to noisy-sample space by defining
$f_t(x_t) := g_0(\hat x_0(x_t))$.
This lifting technique is used in, e.g.,
Diffusion Posterior Sampling \citep{chung2023diffusion} and the Twisted Diffusion Sampler \citep{NEURIPS2023_63e8bc7b}.

\paragraph{(iv) Reverse process and DDPM parameterization.}
By Bayes' rule the ground-truth reverse transition conditioned on the clean sample is Gaussian,
\[
q(x_{t-1}\mid x_t, x_0)
\;=\; \mathcal N\bigl(x_{t-1};\,\tilde\mu_t(x_t,x_0),\;\tilde\sigma_t^2 I\bigr),
\quad
\tilde\sigma_t^2 \;=\; \frac{1-\bar\alpha_{t-1}}{1-\bar\alpha_t}\,\beta_t.
\]
The DDPM parameterization \citep{ho2020denoising} approximates this by plugging in the Tweedie estimate for $x_0$,
yielding the learned reverse kernel
\[
p_\theta(x_{t-1}\mid x_t) \;=\; \mathcal N\bigl(x_{t-1};\,\mu_\theta(x_t,t),\;\sigma_t^2 I\bigr),
\]
where
\[
\mu_\theta(x_t,t) \;=\; \frac{1}{\sqrt{\alpha_t}}\Bigl(x_t - \frac{\beta_t}{\sqrt{1-\bar\alpha_t}}\,\hat\epsilon_\theta(x_t,t)\Bigr),
\qquad
\sigma_t^2 \;=\; \tilde\sigma_t^2.
\]
Unconditional generation (``ancestral sampling''~\citep{song2021scorebased}) then initializes $x_T\sim\mathcal N(0,I)$ and iteratively draws $x_{t-1}\sim p_\theta(\cdot\mid x_t)$ for $t=T,T-1,\dots,1$.

\subsection{PU Learning}
\label{app:prelim:pu}

PU binary classification considers a setting where \emph{positive} samples are drawn from $p(x\mid c)$ and \emph{unlabeled} samples from the marginal $p(x)$,
which mixes $c$ and $\neg c$ in proportion $\pi=p(c)$,
but no labeled $\neg c$ samples are available.
The goal is to train a binary classifier with label $y=+1$ for class $c$ and $y=-1$ for $\neg c$.

We now derive the \emph{unbiased PU (uPU) risk}.
Given the data decomposition $p(x) = \pi\,p(x\mid c) + (1-\pi)\,p(x\mid\neg c)$
with class prior $\pi\in(0,1)$ and a per-sample loss $\ell(\cdot,y)$,
the standard fully supervised (Positive--Negative, PN) risk is
\[
\mathcal L_{\mathrm{PN}}(f)
= \pi\,\mathbb E_{p(x\mid c)}[\ell(f(x),+1)]
+ (1-\pi)\,\mathbb E_{p(x\mid\neg c)}[\ell(f(x),-1)].
\]
Rearranging the data decomposition $p(x) = \pi\,p(x\mid c) + (1-\pi)\,p(x\mid\neg c)$
expresses the unobservable negative-class term as
$(1-\pi)\,p(x\mid\neg c) = p(x) - \pi\,p(x\mid c)$,
replacing it with the marginal $p(x)$ and the positive class $p(x\mid c)$.
\cite{NIPS2014_f032bc3f,pmlr-v37-plessis15} apply this substitution at the level of expectations to obtain the exact identity
\[
(1-\pi)\,\mathbb E_{p(x\mid\neg c)}[\ell(f(x),-1)]
= \mathbb E_{p(x)}[\ell(f(x),-1)] - \pi\,\mathbb E_{p(x\mid c)}[\ell(f(x),-1)].
\]
Substituting yields the \emph{unbiased PU (uPU) risk}~\eqref{eq:upu},
\[
\mathcal L_{\mathrm{uPU}}(f)
= \pi\,\mathbb E_{p(x\mid c)}[\ell(f(x),+1)]
+ \mathbb E_{p(x)}[\ell(f(x),-1)]
- \pi\,\mathbb E_{p(x\mid c)}[\ell(f(x),-1)],
\]
which can be estimated from $P$- and $U$-samples alone.
When $\ell$ is a proper loss (e.g.\ logistic) and the hypothesis class is sufficiently expressive,
$\mathcal L_{\mathrm{uPU}}$ shares its population minimizer with $\mathcal L_{\mathrm{PN}}$;
the minimizer $f^\star$ is calibrated to $p(c\mid x)$ and therefore serves as a drop-in replacement for a supervised classifier.

\subsection{Diffusion Alignment \texorpdfstring{\note{as}}{as} Sampling (DAS)}
\label{app:prelim:das}

A \emph{reward} $R:\mathcal X_0\to\mathbb R$ assigns a scalar score to each clean sample $x_0$ according to how well it satisfies a desired property.
DAS \citep{kim2025testtime} draws samples from a reward-tilted target \citep[Eq.~(3)]{kim2025testtime}
$p_{\mathrm{DAS}}(x_0)\propto p_{\theta,0}(x_0)\exp(\lambda R(x_0))$
for guidance strength $\lambda>0$.
Because exact reward-conditioned marginals along the reverse process are intractable,
DAS defines tractable intermediate targets \citep[Eq.~(11)]{kim2025testtime}
\[
\gamma_t(x_t)
\;\propto\;
p_{\theta,t}(x_t)\exp\bigl(\lambda_t R(\hat x_0(x_t))\bigr),
\qquad
0=\lambda_T\le\cdots\le\lambda_0=\lambda,
\]
which lift the data-space reward to noisy-sample space via the Tweedie estimate $\hat x_0(x_t)$.
Sampling from this path uses SMC with a weighted particle system $\{(x_t^i,w_t^i)\}_{i=1}^M$,
initialized with $w_T^i=1$ (since $\lambda_T=0$ makes $\gamma_T\propto p_{\theta,T}$,
so particles drawn from $\mathcal N(0,I)=p_{\theta,T}$ already follow $\gamma_T$ exactly).
\emph{Remark on notation.}
\cite{kim2025testtime} parameterize the guidance strength as $1/\alpha$ with a schedule
$\lambda_t\in[0,1]$; we use the equivalent form $\lambda_t\in[0,\lambda]$ with $\lambda:=1/\alpha$,
matching the main text.
(Here $\alpha$ is the DAS hyperparameter, distinct from the diffusion noise schedule $\alpha_t$ of Appendix~\ref{app:prelim:diffusion}.)

\paragraph{Proposal kernel.}
The \emph{locally optimal proposal} $m_{t-1}^*$ is the kernel that minimizes the conditional variance $\mathrm{Var}(w_{t-1}\mid x_t)$ of the importance weight \citep[Proposition~1]{kim2025testtime}.
Because sampling from $m_{t-1}^*$ exactly is non-trivial, DAS approximates it \citep[Eq.~(13)]{kim2025testtime} with a Gaussian,
\[
m_{t-1}(x_{t-1}\mid x_t)
\;:=\;
\mathcal N\Bigl(
x_{t-1};\,
\underbrace{\mu_\theta(x_t,t)
+
\sigma_t^2\,\lambda_{t-1}\,\nabla_{x_t} R(\hat{x}_0(x_t))}_{\displaystyle\mu_t^{\mathrm{prop}}},\;
\sigma_t^2 I
\Bigr),
\]
which shifts the DDPM mean by the gradient of the lifted reward.

\paragraph{Weight update derivation.}
The general SMC importance weight is
\[
w_{t-1}(x_{t-1})
\;\propto\;
w_t(x_t)\,
\frac{\gamma_{t-1}(x_{t-1})\,L_t(x_t\mid x_{t-1})}{\gamma_t(x_t)\,m_{t-1}(x_{t-1}\mid x_t)},
\]
where $L_t(x_t\mid x_{t-1})$ is a \emph{backward kernel} that reverses the forward step.
DAS chooses \citep[Eq.~(9)]{kim2025testtime}
\[
L_t(x_t\mid x_{t-1})
\;=\; \frac{p_\theta(x_{t-1}\mid x_t)\,p_{\theta,t}(x_t)}{p_{\theta,t-1}(x_{t-1})}.
\]
Substituting $\gamma_t(x_t)=p_{\theta,t}(x_t)\exp(\lambda_t R(\hat x_0(x_t)))$ (dropping the common normalizing constants, which cancel in ratios) and the backward kernel:
\begin{align*}
\frac{\gamma_{t-1}(x_{t-1})\,L_t(x_t\mid x_{t-1})}{\gamma_t(x_t)}
&=
\frac{
  p_{\theta,t-1}(x_{t-1})\exp(\lambda_{t-1} R(\hat x_0(x_{t-1})))
  \;\cdot\;
  \dfrac{p_\theta(x_{t-1}\mid x_t)\,p_{\theta,t}(x_t)}{p_{\theta,t-1}(x_{t-1})}
}{
  p_{\theta,t}(x_t)\exp(\lambda_t R(\hat x_0(x_t)))
}\\[4pt]
&=
\frac{\exp(\lambda_{t-1} R(\hat x_0(x_{t-1})))\,p_\theta(x_{t-1}\mid x_t)}{\exp(\lambda_t R(\hat x_0(x_t)))}.
\end{align*}
Since $p_\theta(\cdot\mid x_t)$ and $m_{t-1}(\cdot\mid x_t)$ share covariance $\sigma_t^2 I$,
their ratio reduces to
\[
r_t
\;:=\;
\frac{p_\theta(x_{t-1}\mid x_t)}{m_{t-1}(x_{t-1}\mid x_t)}
\;=\;
\exp\Bigl(\frac{1}{2\sigma_t^2}\bigl(
\|x_{t-1}-\mu_t^{\mathrm{prop}}\|^2 - \|x_{t-1}-\mu_\theta\|^2
\bigr)\Bigr),
\]
giving the weight recursion \citep[Eq.~(14)]{kim2025testtime}
\[
w_{t-1}(x_{t-1})
\;=\;
w_t(x_t)
\;\frac{\exp(\lambda_{t-1} R(\hat{x}_0(x_{t-1})))}{\exp(\lambda_t R(\hat{x}_0(x_t)))}
\;r_t.
\]

\paragraph{Asymptotic exactness.}
Under standard regularity assumptions \citep[Appendix~C.2.2]{kim2025testtime},
the weighted particle system is asymptotically exact
\citep[Proposition~8, Appendix~C.2]{kim2025testtime}: for any $\varphi\in\mathcal C_b(\mathcal X_0)$,
\[
\sum_{i=1}^M \note{\frac{w_0^i}{\sum_{k=1}^M w_0^k}}\,\varphi(x_0^i)
\;\xrightarrow{\mathrm{a.s.}}\;
p_{\mathrm{DAS}}(\varphi)
\qquad\text{as }M\to\infty.
\]

\paragraph{Extension to our density path.}
Our method instead \note{uses} the density path \eqref{eq:density-path} in \autoref{sec:method:framework},
$\gamma_t(x_t)\propto p_{\theta,t}(x_t)(1-f_t(x_t))^{\lambda_t}$,
with the optional parameterization $f_t(x_t):=\note{g_0}(\hat x_0(x_t))$ of \autoref{sec:method:ref},
where $\note{g_0}$ is a PU classifier estimating $\note{p_{\theta,0}(c\mid x_0)}$.
This corresponds to instantiating the DAS framework with
$R(x_0)\leftarrow\log(1-\note{g_0}(x_0))$, since
\[
\exp\bigl(\lambda_t R(\hat x_0(x_t))\bigr)
\;=\;\exp\bigl(\lambda_t\log(1-\note{g_0}(\hat x_0(x_t)))\bigr)
\;=\;(1-f_t(x_t))^{\lambda_t}.
\]
Substituting into the weight recursion above yields
\begin{equation*}
w_{t-1}(x_{t-1})
\;=\;
w_{t}(x_t)\;r_t\;
\frac{(1-f_{t-1}(x_{t-1}))^{\lambda_{t-1}}}{(1-f_t(x_t))^{\lambda_t}},
\end{equation*}
which is the recursion~\eqref{eq:wq-recursion}.

\note{
\subsection{On the Exactness of Dynamic Negative Guidance}
\label{app:prelim:dng}

DNG~\citep{koulischer2025dynamic} is exact under three conditions:
exact scores, the limit $T\to\infty$, and repulsion strength $\eta=1$.
This subsection examines the last two.

\paragraph{The posterior recursion is exact only as $T\to\infty$.}
The guided score~\eqref{eq:neg-score} needs the marginal posterior
$p_t(c\mid x_t)$, whereas the recursion~\eqref{eq:dngposterior} is built from the trajectory $x_T,\dots,x_t$ generated so far.
For two reverse chains started at the same $x_T\sim\mathcal N(0,I)$, one unconditional and one conditioned on $c$, Bayes' rule gives
\begin{equation}
\label{eq:dng-traj}
  p(c\mid x_T,\dots,x_t)
  \;=\; p(c)\prod_{s=t+1}^{T}\frac{p(x_{s-1}\mid x_s,c)}{p(x_{s-1}\mid x_s)}
  \;=\; p_t(c\mid x_t)\,
     \underbrace{\frac{p(x_T,\dots,x_{t+1}\mid x_t,c)}
                      {p(x_T,\dots,x_{t+1}\mid x_t)}}_{(\star)} .
\end{equation}
The logarithm of the middle expression is exactly~\eqref{eq:dngposterior}, so the recursion tracks the trajectory posterior, which equals the marginal one only when $(\star)=1$.
\citet{koulischer2025dynamic} prove $(\star)=1$ in their Appendix~C as a consequence of the Markov assumption, the forward chain carrying no condition.
The argument presumes exact time reversal, which holds only as $T\to\infty$,
the reverse kernel being Gaussian only for small steps.
Table~\ref{tab:gmm-dng-T} measures the residual bias on the Gaussian mixture,
where the scores are exact.

\paragraph{The tempered density is not guaranteed as a data-level target.}
At $\eta=1$ the guided score is that of $p_t(\cdot\mid\neg c)$, the forward marginal of the data density $p_0(\cdot\mid\neg c)$, so the sampler reverses a genuine forward diffusion.
At $\eta>1$ the guidance instead follows
$\nabla_{x_t}\log\bigl[p_t(x_t)\bigl(1-p_t(c\mid x_t)\bigr)^{\eta}\bigr]$,
aiming at the tempered densities in score form, and Footnote~2 of~\citet{koulischer2025dynamic} states that this implies sampling from them.
The implication does not hold in general: raising the second factor to a power does not commute with the integral defining the forward marginal,
so the tempered densities need not be the forward marginals of any single data density.
Theorem~3.2 of \citet{bradley2025classifierfree} turns the same gap into a counterexample for CFG, where the sampler provably does not produce the powered density at a given guidance strength.

The analysis of \autoref{sec:theory} is stated under the complementary conditions: it is relative to the pretrained model's own reverse process, so it assumes no exact score,
it is carried out on the discrete chain at finite $T$,
and it holds for every avoidance strength $\lambda\ge 0$, with the remaining error controlled by the particle counts.
}

%% file: appendix2.tex
\section{Algorithms}
\label{app:algorithms}

\autoref{alg:neg_das_pu_fixed} gives detailed pseudocode for the reference-based setting (\autoref{sec:method:ref});
\autoref{alg:neg_das_pu_singlechain} for the score-based setting (\autoref{sec:method:score}).
The following notation supplements \autoref{sec:existing}.
\begin{itemize}
\item \textbf{Conditional Tweedie denoiser.}
$\hat x_0(x_t\mid c)
:= \bigl(x_t - \sqrt{1-\bar\alpha_t}\,\hat\epsilon_\theta(x_t,t\mid c)\bigr)/\sqrt{\bar\alpha_t}$
is the conditional analog of $\hat x_0(x_t)$.
\autoref{alg:neg_das_pu_singlechain} uses it for the positive set because conditioning on $c$ yields a sharper denoised view at intermediate~$t$,
easing discrimination.
Both $\hat x_0(x_t)$ and $\hat x_0(x_t\mid c)$ reduce to the identity at $t{=}0$.
\item \textbf{uPU oracle.}
$\mathcal{O}_{\mathrm{uPU}}(\pi,\mathcal P,\mathcal U,f)$ denotes a few gradient steps minimizing the uPU risk~\eqref{eq:upu} with positives~$\mathcal P$,
weighted unlabeled set~$\mathcal U$, class prior~$\pi$, warm-started from classifier~$f$.
At $t{=}T$ the oracle is called without~$f$ (cold start).
\item \textbf{Time-varying class prior.}
Both algorithms accept a schedule $\{\pi_t\}_{t=0}^T$.
Setting $\pi_t\equiv\pi$ recovers the main-text formulation.
A varying $\pi_t$ is straightforward in practice since the theory requires only $f_0$ to be accurate at $t{=}0$ (\autoref{sec:theory}).
\end{itemize}

\clearpage

\begin{algorithm}[H]
\caption{Negative Alignment via SMC with Tweedie--uPU: Reference-Based Setting}
\label{alg:neg_das_pu_fixed}
\begin{algorithmic}[1]
\Require Pretrained DDPM reverse mean $\mu_\theta$ and Tweedie denoiser $\hat x_0$ (see \autoref{sec:existing}),
noise schedule $\{\beta_t\}_{t=1}^T$,
particle count $M$,
resampling threshold $\mathrm{ESS}_{\min}$,
class priors $\{\pi_t\}_{t=0}^T$,
uPU oracle $\mathcal{O}_{\mathrm{uPU}}$,
reference sampler producing $\{x_{c,t}^{(i)}\}_{i=1}^N$ at all $t$,
proposal mean $\mu^{\mathrm{prop}}(x_t,t,\lambda,f)$ (default: \eqref{eq:prop-fixed}),
tempering: precomputed $\{\lambda_t\}_{t=0}^T$ ($0=\lambda_T\le\cdots\le\lambda_0$)
\textbf{or} ESS target $\mathrm{ESS}_{\mathrm{tar}}$ for adaptive tempering.

\vspace{4pt}
\State \textbf{Initialize} $\{x_T^{(j)}\}_{j=1}^M \sim \mathcal N(0,I)$,\quad
$w_{q,T}^{(j)} \leftarrow 1$,\quad $w_{\emptyset,T}^{(j)} \leftarrow 1$,\quad $\lambda_T \leftarrow 0$
\vspace{4pt}

\State $\mathcal P_T=\{\hat x_0(x_{c,T}^{(i)})\}_{i=1}^N, \quad
       \mathcal U_T=\{(\hat x_0(x_T^{(j)}),w_{\emptyset,T}^{(j)})\}_{j=1}^M$
\Comment{Initialize classifier}
\State $f_T \leftarrow \mathcal{O}_{\mathrm{uPU}}(\pi_T,\mathcal P_T,\mathcal U_T)$
\vspace{4pt}

\For{$t=T,T-1,\dots,1$}

\If{adaptive tempering}
\State Find $\delta_t \ge 0$ s.t.\
$\mathrm{ESS}(\{w_{q,t}^{(j)} (1-f_t(\hat x_0(x_t^{(j)})))^{\delta_t}\}) \approx \mathrm{ESS}_{\mathrm{tar}}$;\quad
$\lambda_{t-1}\leftarrow\min(\lambda_t+\delta_t,\,\note{\lambda_0})$
\Else
\State $\lambda_{t-1} \leftarrow$ precomputed value
\EndIf
\vspace{4pt}

\For{$j=1,\dots,M$} \Comment{Proposal}
\State $\mu_t^{(j)} = \mu^{\mathrm{prop}}(x_t^{(j)}, t, \lambda_{t-1}, f_t)$
{\color[gray]{0.6}\footnotesize
\Statex \hspace{40pt} // default:
$\mu_\theta(x_t^{(j)},t)+\sigma_t^2\lambda_{t-1}\nabla_{x_t}\log(1-f_t(\hat x_0(x_t^{(j)})))$
}
\State $x_{t-1}^{(j)}\sim\mathcal N(\mu_t^{(j)},\sigma_t^2 I)$
\State $r_t^{(j)} \leftarrow \exp\!\left(\tfrac{1}{2\sigma_t^2}\bigl(\|x_{t-1}^{(j)}-\mu_t^{(j)}\|^2-\|x_{t-1}^{(j)}-\mu_\theta(x_t^{(j)},t)\|^2\bigr)\right)$
\State $w_{\emptyset,t-1}^{(j)} \leftarrow w_{\emptyset,t}^{(j)}\, r_t^{(j)}$
\Comment{Update $w_\emptyset$}
\EndFor
\vspace{2pt}

\State $\mathcal P_{t-1}=\{\hat x_0(x_{c,t-1}^{(i)})\}_{i=1}^N,\quad
       \mathcal U_{t-1}=\{(\hat x_0(x_{t-1}^{(j)}),w_{\emptyset,t-1}^{(j)})\}_{j=1}^M$
\Comment{Refit classifier}
\State $f_{t-1} \leftarrow \mathcal{O}_{\mathrm{uPU}}(\pi_{t-1},\mathcal P_{t-1},\mathcal U_{t-1},f_t)$
\vspace{4pt}

\For{$j=1,\dots,M$} \Comment{Update $w_q$~\eqref{eq:wq-recursion}}
\State $w_{q,t-1}^{(j)} \leftarrow w_{q,t}^{(j)}\, r_t^{(j)}
\dfrac{\bigl(1-f_{t-1}(\hat x_0(x_{t-1}^{(j)}))\bigr)^{\lambda_{t-1}}}{\bigl(1-f_t(\hat x_0(x_t^{(j)}))\bigr)^{\lambda_t}}$
\EndFor

\If{$\mathrm{ESS}(\{w_{q,t-1}^{(j)}\})<\mathrm{ESS}_{\min}$} \Comment{Resampling}
\State Sample ancestors $a(1{:}M)$ using $\{w_{q,t-1}^{(j)}\}$
\For{$j=1,\dots,M$}
\State $x_{t-1}^{(j)}\leftarrow x_{t-1}^{(a(j))}$,\quad
      $w_{\emptyset,t-1}^{(j)}\leftarrow
      \dfrac{w_{\emptyset,t-1}^{(a(j))}}{w_{q,t-1}^{(a(j))}}$,\quad
      $w_{q,t-1}^{(j)}\leftarrow1$
\EndFor
\EndIf

\EndFor
\State \Return $\{x_0^{(j)},w_{q,0}^{(j)}\}_{j=1}^M$
\end{algorithmic}
\end{algorithm}

\begin{algorithm}[H]
\caption{Negative Alignment via SMC with Tweedie--uPU: Score-Based Setting}
\label{alg:neg_das_pu_singlechain}
\begin{algorithmic}[1]
\Require Pretrained DDPM reverse means $\mu_\theta(x_t,t)$ and $\mu_\theta(x_t,t\mid c)$,
Tweedie denoisers $\hat x_0(x_t)$ and $\hat x_0(x_t\mid c)$ (see \autoref{sec:existing} and above),
noise schedule $\{\beta_t\}_{t=1}^T$,
particle count $M$,
resampling threshold $\mathrm{ESS}_{\min}$,
class priors $\{\pi_t\}_{t=0}^T$,
uPU oracle $\mathcal{O}_{\mathrm{uPU}}$,
proposal mean $\mu^{\mathrm{prop}}(x_t,t,\lambda,f)$ (default: \eqref{eq:prop-fixed}),
tempering: precomputed $\{\lambda_t\}_{t=0}^T$ ($0=\lambda_T\le\cdots\le\lambda_0$)
\textbf{or} ESS target $\mathrm{ESS}_{\mathrm{tar}}$ for adaptive tempering;\quad
initialize $\lambda_T\leftarrow 0$.

\vspace{4pt}
\State \textbf{Initialize} $\{x_T^{(j)}\}_{j=1}^M \sim \mathcal N(0,I)$,\quad
$w_{q,T}^{(j)}\leftarrow 1$,\quad $w_{\emptyset,T}^{(j)}\leftarrow 1$,\quad $w_{c,T}^{(j)}\leftarrow 1 \quad\forall j$

\vspace{4pt}
\State $\mathcal P_T=\{(\hat x_0(x_T^{(j)}\mid c),w_{c,T}^{(j)})\}_{j=1}^M$,\quad
$\mathcal U_T=\{(\hat x_0(x_T^{(j)}),w_{\emptyset,T}^{(j)})\}_{j=1}^M$
\State $f_T \leftarrow \mathcal{O}_{\mathrm{uPU}}(\pi_T,\mathcal P_T,\mathcal U_T)$
\Comment{Initialize classifier}

\vspace{6pt}
\For{$t=T,T-1,\dots,1$}

\If{adaptive tempering}
\State Find $\delta_t\ge 0$ s.t.\
$\mathrm{ESS}(\{w_{q,t}^{(j)} (1-f_t(\hat x_0(x_t^{(j)})))^{\delta_t}\}) \approx \mathrm{ESS}_{\mathrm{tar}}$;\quad
$\lambda_{t-1}\leftarrow\min(\lambda_t+\delta_t,\,\note{\lambda_0})$
\Else
\State $\lambda_{t-1} \leftarrow$ precomputed value
\EndIf

\vspace{4pt}
\For{$j=1,\dots,M$} \Comment{Proposal}
\State $\mu_t^{(j)} = \mu^{\mathrm{prop}}(x_t^{(j)}, t, \lambda_{t-1}, f_t)$
{\color[gray]{0.6}\footnotesize
\Statex \hspace{40pt} // default:
$\mu_\theta(x_t^{(j)},t)+\sigma_t^2\lambda_{t-1}\nabla_{x_t}\log(1-f_t(\hat x_0(x_t^{(j)})))$
}
\State $x_{t-1}^{(j)}\sim\mathcal N(\mu_t^{(j)},\sigma_t^2 I)$
\State $r_t^{(j)} \leftarrow \exp\!\left(\tfrac{1}{{2\sigma_t^2}}\bigl(\|x_{t-1}^{(j)}-\mu_t^{(j)}\|^2-\|x_{t-1}^{(j)}-\mu_\theta(x_t^{(j)},t)\|^2\bigr)\right)$
\State $w_{\emptyset,t-1}^{(j)} \leftarrow w_{\emptyset,t}^{(j)}\, r_t^{(j)}$
\Comment{Update $w_\emptyset$}
\State $w_{c,t-1}^{(j)} \leftarrow w_{c,t}^{(j)}
\exp\!\left(\tfrac{1}{{2\sigma_t^2}}\bigl(\|x_{t-1}^{(j)}-\mu_t^{(j)}\|^2-\|x_{t-1}^{(j)}-\mu_\theta(x_t^{(j)},t\mid c)\|^2\bigr)\right)$
\Comment{Update $w_c$}
\EndFor

\vspace{4pt}
\State $\mathcal P_{t-1}=\{(\hat x_0(x_{t-1}^{(j)}\mid c),w_{c,t-1}^{(j)})\}_{j=1}^M$,\quad
$\mathcal U_{t-1}=\{(\hat x_0(x_{t-1}^{(j)}),w_{\emptyset,t-1}^{(j)})\}_{j=1}^M$
\State $f_{t-1} \leftarrow \mathcal{O}_{\mathrm{uPU}}(\pi_{t-1},\mathcal P_{t-1},\mathcal U_{t-1},f_t)$
\Comment{Refit classifier}

\vspace{6pt}
\For{$j=1,\dots,M$} \Comment{Update $w_q$~\eqref{eq:wq-recursion}}
\State $w_{q,t-1}^{(j)} \leftarrow w_{q,t}^{(j)}\, r_t^{(j)}
\dfrac{(1-f_{t-1}(\hat x_0(x_{t-1}^{(j)})))^{\lambda_{t-1}}}{(1-f_t(\hat x_0(x_t^{(j)})))^{\lambda_t}}$
\EndFor

\vspace{4pt}
\If{$\mathrm{ESS}(\{w_{q,t-1}^{(j)}\})<\mathrm{ESS}_{\min}$} \Comment{Resampling}
\State Sample ancestors $a(1{:}M)$ using $\{w_{q,t-1}^{(j)}\}$
\For{$j=1,\dots,M$}
\State $x_{t-1}^{(j)}\leftarrow x_{t-1}^{(a(j))}$,\quad
$w_{\emptyset,t-1}^{(j)}\leftarrow \dfrac{w_{\emptyset,t-1}^{(a(j))}}{w_{q,t-1}^{(a(j))}}$,\quad
$w_{c,t-1}^{(j)}\leftarrow \dfrac{w_{c,t-1}^{(a(j))}}{w_{q,t-1}^{(a(j))}}$,\quad
$w_{q,t-1}^{(j)}\leftarrow 1$
\EndFor
\EndIf

\EndFor
\State \Return $\{x_0^{(j)},w_{q,0}^{(j)}\}_{j=1}^M$
\end{algorithmic}
\end{algorithm}

\clearpage

%% file: appendix_proof.tex
\section{Proof and Technical Detail}
\label{app:full-proof}

This appendix provides formal statements and proofs supporting \autoref{sec:theory}.
The object of analysis is the meta-algorithm \autoref{alg:pu-steer-meta}, in its reference-based instantiation (\autoref{sec:method:ref}) and score-based instantiation (\autoref{sec:method:score}).
We identify the avoidance target~\eqref{eq:target} $p_{\mathrm{target}}$ with the terminal Feynman--Kac measure
\[
\nu_0(x)\;\propto\;(1-\eta_0(x))^{\lambda_0}\,p_{\theta,0}(x),
\qquad
\eta_0(x):=\note{p_{\theta,0}(c\mid x)},
\]
and show that the weighted empirical measure $\widehat\nu_0^{M\mid M',N}$ produced by \autoref{alg:pu-steer-meta} converges to $\nu_0$ almost surely in the iterated limit $\lim_{M',N\to\infty}\lim_{M\to\infty}$.
 
For analytical clarity we adopt a proof-oriented variant in which classifier estimation is decoupled from the final SMC step. The classifier $f_t$ is fit on an auxiliary particle system of size $M'$ that shares the kernels of, but is independent of, the main system of size $M$. Positive samples for the PU risk are the $N$ external references in the reference-based setting, and are supplied by the auxiliary system itself in the score-based setting. The proposal kernel is left arbitrary in the algorithm and regularized within the proof so that the importance weights remain uniformly bounded.
 
The argument has three steps. (i) Appendix~\ref{app:pu-fixed-t} establishes consistency of the weighted PU classifier at each fixed diffusion time, taking the positive and unlabeled samples as exogenous inputs satisfying suitable uniform laws of large numbers. (ii) The SMC construction underlying \autoref{alg:pu-steer-meta} is shown to satisfy these uniform laws for both $p_{\theta,t}$ and $p_{\theta,t}(\cdot\mid c)$. (iii) A plug-in argument, combined with the telescoping of the tilting factors $(1-f_t)^{\lambda_t}$ in the recursion~\eqref{eq:wq-recursion}, yields convergence of the main particle system to $\nu_0$. The reference-based instantiation serves as the master argument. The score-based instantiation follows by the same scheme, with the external positive samples replaced by the auxiliary particles reweighted by $w_{c,t}$, which retargets them to $p_{\theta,t}(\cdot\mid c)$ as in \autoref{sec:method:score}.

\subsection{Consistency of Weighted PU Logistic Regression at a Fixed Diffusion Time}
\label{app:pu-fixed-t}

Fix a diffusion time $t$ and write
\[
\eta_t(x) := \note{p_{\theta,t}(c \mid x)}, \qquad x \in \mathcal X_t,
\]
for the conditional class posterior at time $t$. Let $P_t$ denote the marginal law of $x_t$, let
\[
\pi_t := P_t(c) \in (0,1),
\]
and let $P_t^+$ denote the conditional law of $x_t$ given $c$.

We observe:
\begin{itemize}
    \item positive samples $X_1^+,\dots,X_N^+ \overset{\mathrm{i.i.d.}}{\sim} P_t^+$;
    \item unlabeled samples $U_1,\dots,U_M$ together with nonnegative importance weights $w_t(U_j)$.
\end{itemize}
Define the normalized weights
\[
\bar w_{t,j} := \frac{w_t(U_j)}{\sum_{k=1}^M w_t(U_k)}.
\]

Let $\mathcal G_t$ be a class of measurable functions $g:\mathcal X_t\to\mathbb R$, and define
\[
\ell_1(z) := \log(1+e^{-z}),
\qquad
\ell_0(z) := \log(1+e^{z}),
\qquad
\sigma(z) := \frac{1}{1+e^{-z}}.
\]

For $g\in\mathcal G_t$, define the weighted empirical PU risk
\begin{align}
\widehat L_{t,M,N}(g)
:=
&\;\pi_t \frac1N \sum_{i=1}^N \ell_1\!\bigl(g(X_i^+)\bigr)
+
\sum_{j=1}^M \bar w_{t,j}\,\ell_0\!\bigl(g(U_j)\bigr)
-
\pi_t \frac1N \sum_{i=1}^N \ell_0\!\bigl(g(X_i^+)\bigr)
\nonumber\\
=\;
&\;-\pi_t \frac1N \sum_{i=1}^N g(X_i^+)
+
\sum_{j=1}^M \bar w_{t,j}\,\ell_0\!\bigl(g(U_j)\bigr).
\label{eq:empirical_pu_risk_fixed_t}
\end{align}
This is the logistic objective of the standard PU learning \citep{pmlr-v37-plessis15},
with the unlabeled empirical average replaced by an importance-weighted one.
Correspondingly, define the population PU risk
\begin{align}
L_t^{\mathrm{PU}}(g)
:=
&\;\pi_t \,\mathbb E_{P_t^+}[\ell_1(g(X))]
+
\mathbb E_{P_t}[\ell_0(g(X))]
-
\pi_t \,\mathbb E_{P_t^+}[\ell_0(g(X))]
\nonumber\\
=\;
&\;-\pi_t \,\mathbb E_{P_t^+}[g(X)]
+
\mathbb E_{P_t}[\ell_0(g(X))].
\label{eq:population_pu_risk_fixed_t}
\end{align}

We impose the following assumptions.

\begin{enumerate}
    \item[(A1)] \textbf{Positive-sample uniform law of large numbers.}
    There exists a function class $\mathcal H_t$ containing
    $\{g : g\in\mathcal G_t\}
    \cup
    \{\ell_1\circ g : g\in\mathcal G_t\}
    \cup
    \{\ell_0\circ g : g\in\mathcal G_t\}$
    such that, as $N\to\infty$,
    \[
    \sup_{h\in\mathcal H_t}
    \left|
    \frac1N\sum_{i=1}^N h(X_i^+) - \mathbb E_{P_t^+}[h(X)]
    \right|
    \to 0
    \qquad \text{a.s.}
    \]

    \item[(A2)] \textbf{Weighted unlabeled uniform convergence.}
    For the same class $\mathcal H_t$, as $M\to\infty$,
    \[
    \sup_{h\in\mathcal H_t}
    \left|
    \sum_{j=1}^M \bar w_{t,j} h(U_j) - \mathbb E_{P_t}[h(X)]
    \right|
    \to 0
    \qquad \text{a.s.}
    \]

    \item[(A3)] \textbf{Realizability of the Bayes logit.}
    The posterior satisfies $\eta_t(x)\in(0,1)$ for $P_t$-a.e. $x$, and the Bayes logit
    $g_t^\star(x) := \log \frac{\eta_t(x)}{1-\eta_t(x)}$
    belongs to $\mathcal G_t$.
\end{enumerate}

Finally, assume that for every $(M,N)$ under consideration, a measurable empirical minimizer exists:
\[
\widehat g_{t,M,N}\in\arg\min_{g\in\mathcal G_t}\widehat L_{t,M,N}(g).
\]
Define the induced probabilistic discriminator by
$\widehat f_{t,M,N}(x):=\sigma(\widehat g_{t,M,N}(x))$.

\begin{theorem}[Consistency of weighted PU logistic discrimination at a fixed diffusion time]
\label{thm:fixed_t_weighted_pu_consistency}
Suppose \textnormal{(A1)--(A3)} hold. Then, along any sequence $(M,N)$ such that $M\to\infty$ and $N\to\infty$,
\[
\|\widehat f_{t,M,N}-\eta_t\|_{L^1(P_t)} \to 0
\qquad \text{a.s.}
\]
\end{theorem}

\begin{proof}
Let
\begin{align*}
\Delta_{M,N}
:=
&\;\sup_{g\in\mathcal G_t}
\bigl|
\widehat L_{t,M,N}(g)-L_t^{\mathrm{PU}}(g)
\bigr|
\\
\le
&\;\pi_t\sup_{g\in\mathcal G_t}
\left|
\frac1N\sum_{i=1}^N g(X_i^+) - \mathbb E_{P_t^+}[g(X)]
\right|
\;+
\sup_{g\in\mathcal G_t}
\left|
\sum_{j=1}^M \bar w_{t,j}\ell_0(g(U_j)) - \mathbb E_{P_t}[\ell_0(g(X))]
\right|.
\end{align*}
By Assumptions~(A1) and~(A2), $\Delta_{M,N}\to 0$ a.s.

We first rewrite the population PU risk. 
Let $P_t^-$ denote the conditional law of $x_t$ given $\neg c$.
Since
$P_t=\pi_t P_t^+ + (1-\pi_t)P_t^-$ and
$\pi_t\,dP_t^+=\eta_t\,dP_t$, we have
\[
L_t^{\mathrm{PU}}(g)
=
\mathbb E_{P_t}\!\left[
\eta_t\,\ell_1(g) + (1-\eta_t)\,\ell_0(g)
\right].
\]
Writing $q_g=\sigma(g)$ and using
$\ell_1(g)=-\log q_g$, $\ell_0(g)=-\log(1-q_g)$,
this becomes
\[
L_t^{\mathrm{PU}}(g)
=
\mathbb E_{P_t}\!\left[
-\eta_t\log q_g - (1-\eta_t)\log(1-q_g)
\right].
\]
Hence
\[
L_t^{\mathrm{PU}}(g)-L_t^{\mathrm{PU}}(g_t^\star)
=
\mathbb E_{P_t}\!\left[
\mathrm{KL}(\mathrm{Bern}(\eta_t)\,\|\,\mathrm{Bern}(q_g))
\right],
\]
and by Pinsker's inequality and Jensen's inequality,
\[
L_t^{\mathrm{PU}}(g)-L_t^{\mathrm{PU}}(g_t^\star)
\ge
2\,\mathbb E_{P_t}\!\left[|\eta_t-q_g|^2\right]
\ge
2\,\|q_g-\eta_t\|_{L^1(P_t)}^2.
\]

Next, since $\widehat g_{t,M,N}$ minimizes $\widehat L_{t,M,N}$ over $\mathcal G_t$,
\[
\widehat L_{t,M,N}(\widehat g_{t,M,N})
\le
\widehat L_{t,M,N}(g_t^\star).
\]
Hence
\begin{align*}
L_t^{\mathrm{PU}}(\widehat g_{t,M,N})-L_t^{\mathrm{PU}}(g_t^\star)
&=
\Bigl(L_t^{\mathrm{PU}}(\widehat g_{t,M,N})-\widehat L_{t,M,N}(\widehat g_{t,M,N})\Bigr)
\\
&\quad+
\Bigl(\widehat L_{t,M,N}(\widehat g_{t,M,N})-\widehat L_{t,M,N}(g_t^\star)\Bigr)
\\
&\quad+
\Bigl(\widehat L_{t,M,N}(g_t^\star)-L_t^{\mathrm{PU}}(g_t^\star)\Bigr)
\\
&\le
\Bigl|L_t^{\mathrm{PU}}(\widehat g_{t,M,N})-\widehat L_{t,M,N}(\widehat g_{t,M,N})\Bigr|
+
\Bigl|\widehat L_{t,M,N}(g_t^\star)-L_t^{\mathrm{PU}}(g_t^\star)\Bigr|
\\
&\le
2\sup_{g\in\mathcal G_t}
\bigl|
\widehat L_{t,M,N}(g)-L_t^{\mathrm{PU}}(g)
\bigr|
=
2\Delta_{M,N}.
\end{align*}
Since $\Delta_{M,N}\to 0$ a.s., it follows that
\[
L_t^{\mathrm{PU}}(\widehat g_{t,M,N})-L_t^{\mathrm{PU}}(g_t^\star)\to 0
\qquad\text{a.s.}
\]

On the other hand, from the lower bound established above, applied with
$g=\widehat g_{t,M,N}$, we have
\[
L_t^{\mathrm{PU}}(\widehat g_{t,M,N})-L_t^{\mathrm{PU}}(g_t^\star)
\ge
2\,\|q_{\widehat g_{t,M,N}}-\eta_t\|_{L^1(P_t)}^2
=
2\,\|\widehat f_{t,M,N}-\eta_t\|_{L^1(P_t)}^2.
\]
Therefore,
\[
0
\le
2\,\|\widehat f_{t,M,N}-\eta_t\|_{L^1(P_t)}^2
\le
L_t^{\mathrm{PU}}(\widehat g_{t,M,N})-L_t^{\mathrm{PU}}(g_t^\star)
\to 0
\qquad\text{a.s.}
\]
and hence
\[
\|\widehat f_{t,M,N}-\eta_t\|_{L^1(P_t)}\to 0
\qquad\text{a.s.}
\]
\end{proof}

\subsection{Verification of Dual-Weight SMC Construction}
\label{app:verify-a2}

This subsection establishes the theoretical validity of the dual-weight SMC construction.
We show that, under a regularized common proposal, the resulting particle system simultaneously provides:
(i) a consistent approximation of the unconditional marginal $p_{\theta,t}$, and
(ii) a particle approximation of the tilted targets $\nu_t$ associated with the negative-alignment tempering.

The key ingredients are:
a common proposal yielding uniformly bounded incremental weights,
a telescoping structure of the associated Feynman--Kac models,
and a ratio transform linking the tilted and unconditional marginals.

These results together imply that the particle system supplies the weighted empirical measure required in Assumption~(A2).

We introduce the following assumption, used throughout the remainder of the appendix.

\begin{enumerate}
    \item[(A4)] \textbf{Non-degeneracy of the conditional class posterior.}
    There exists $\delta\in(0,1)$ such that for all $t\in\{0,\dots,T\}$,
    \[
    \delta \le \eta_t(x) \le 1-\delta
    \qquad\text{for } p_{\theta,t}\text{-a.e.\ }x.
    \]
\end{enumerate}

\subsubsection{Regularized Common Proposal}
\label{app:regularized-proposal}

\paragraph{Reverse kernels and tilted targets.}
Fix a diffusion horizon $T\in\mathbb N$. For each $t\in\{0,1,\dots,T\}$, let
$p_{\theta,t}$ denote the unconditional marginal law at time $t$,
and define the unconditional reverse diffusion kernel by
\[
K_{t-1}(x_t,dx_{t-1}) := p_\theta(dx_{t-1}\mid x_t),
\qquad t=1,\dots,T.
\]
Let $\{\lambda_t\}_{t=0}^T$ be a deterministic tempering schedule satisfying
$0=\lambda_T \le \lambda_{T-1} \le \cdots \le \lambda_0 < \infty$.
Fix a measurable version $\eta_t^\sharp$ of $\eta_t$, and define the clipped posterior
\[
\bar\eta_t(x)
:=
\min\{1-\delta,\max\{\eta_t^\sharp(x),\delta\}\} 
\quad\in [\delta, 1-\delta]
\quad\text{for all } x.
\]
Define
\[
\omega_t(x) := \bigl(1-\bar\eta_t(x)\bigr)^{\lambda_t},
\qquad
\nu_t(dx)
:=
\frac{\omega_t(x)\,p_{\theta,t}(dx)}
{\int \omega_t(x)\,p_{\theta,t}(dx)}.
\]
By Assumption~(A4), we have $\bar\eta_t(x)=\eta_t(x)$ for $p_{\theta,t}$-a.e.\ $x$,
so the definition of $\nu_t$ is unchanged.
Since $\lambda_T=0$, we have $\omega_T\equiv 1$ and $\nu_T = p_{\theta,T}$.
Under Assumption~(A4),
\begin{equation}
\label{eq:omega-bounds}
\delta^{\lambda_0}\le \omega_t(x)\le 1
\qquad\text{for all } x.
\end{equation}

\paragraph{Regularized proposal.}
Let $\widetilde M_{t-1}(x_t,dx_{t-1})$ be any measurable Markov kernel (the nominal guided proposal).
Fix $\rho\in(0,1)$ and define the \emph{regularized common proposal}
\begin{equation}
\label{eq:regularized-common-proposal}
M_{t-1}^{\mathrm{reg}}(x_t,dx_{t-1})
:=
(1-\rho)\,\widetilde M_{t-1}(x_t,dx_{t-1})
+
\rho\,K_{t-1}(x_t,dx_{t-1}).
\end{equation}

\begin{lemma}[Regularized common proposal]
\label{prop:existence-common-proposal}
For any nominal proposal $\widetilde M_{t-1}$ and any $\rho\in(0,1)$,
the mixture kernel $M_{t-1}^{\mathrm{reg}}$ satisfies
\[
K_{t-1}(x_t,\cdot) \ll M_{t-1}^{\mathrm{reg}}(x_t,\cdot),
\qquad
\frac{dK_{t-1}(x_t,\cdot)}{dM_{t-1}^{\mathrm{reg}}(x_t,\cdot)}(x_{t-1})
\le \rho^{-1}
\]
for $M_{t-1}^{\mathrm{reg}}(x_t,\cdot)$-a.e.\ $x_{t-1}$ and every $x_t$.
\end{lemma}

\begin{proof}
By construction, $M_{t-1}^{\mathrm{reg}}(x_t,A)\ge \rho\,K_{t-1}(x_t,A)$ for every measurable set $A$.
Absolute continuity is immediate, and the Radon--Nikodym bound follows.
\end{proof}

\subsubsection{Feynman--Kac Models and Path Measures}
\label{app:fk-models}

Define the unconditional and tilted incremental weights on the common proposal. For $t=1,\dots,T$, let
\begin{equation}
\label{eq:G-empty}
G_{t-1}^{\emptyset}(x_t,x_{t-1})
:=
\frac{dK_{t-1}(x_t,\cdot)}{dM_{t-1}^{\mathrm{reg}}(x_t,\cdot)}(x_{t-1}),
\end{equation}
and
\begin{equation}
\label{eq:G-q}
G_{t-1}^{q}(x_t,x_{t-1})
:=
G_{t-1}^{\emptyset}(x_t,x_{t-1})
\frac{\omega_{t-1}(x_{t-1})}{\omega_t(x_t)}.
\end{equation}

\begin{lemma}[Uniform boundedness of both potentials]
\label{lem:bounded-potentials}
Under Assumption~\textnormal{(A4)},
\[
0 \le G_{t-1}^{\emptyset}(x_t,x_{t-1}) \le \rho^{-1},
\qquad
0 \le G_{t-1}^{q}(x_t,x_{t-1}) \le \rho^{-1}\delta^{-\lambda_0},
\]
for all $t=1,\dots,T$, for all $x_t$, and for
$M_{t-1}^{\mathrm{reg}}(x_t,\cdot)$-a.e.\ $x_{t-1}$.
\end{lemma}

\begin{proof}
The bound on $G^{\emptyset}$ is \autoref{prop:existence-common-proposal}.
By Eq.~\eqref{eq:omega-bounds},
$\omega_{t-1}(x_{t-1})\le 1$ and $\omega_t(x_t)^{-1}\le \delta^{-\lambda_0}$,
giving
$G_{t-1}^{q}\le \rho^{-1}\delta^{-\lambda_0}$.
\end{proof}

\paragraph{Path measures.}
Define the common proposal path measure
\[
\mathbf{M}(dx_{0:T})
:=
p_{\theta,T}(dx_T)\prod_{s=1}^T M_{s-1}^{\mathrm{reg}}(x_s,dx_{s-1}),
\]
and the unnormalized path weights
\[
W_t^{a}(x_{t:T})
:=
\prod_{s=t+1}^T G_{s-1}^{a}(x_s,x_{s-1}),
\qquad a\in\{\emptyset,q\},
\]
for $t=0,\dots,T$ (empty products equal $1$).
Define the associated unnormalized path measures
$\mathbf{\Gamma}_t^{a}(dx_{t:T})
:=
W_t^{a}(x_{t:T})\,
p_{\theta,T}(dx_T)\prod_{s=t+1}^T M_{s-1}^{\mathrm{reg}}(x_s,dx_{s-1})$.

\begin{proposition}[Identification of the path models]
\label{prop:path-identification}
For every $t \in \{0,\dots,T\}$:
\begin{enumerate}
    \item[\textnormal{(i)}] the $x_t$-marginal of\/ $\mathbf{\Gamma}_t^{\emptyset}$ is $p_{\theta,t}$;
    \item[\textnormal{(ii)}] the $x_t$-marginal of\/ $\mathbf{\Gamma}_t^{q}$ is $\omega_t(x)\,p_{\theta,t}(dx)$,
    so the normalized marginal is $\nu_t$.
\end{enumerate}
\end{proposition}

\begin{proof}
\emph{Part~(i).}
Cancellation of proposal densities gives
\[
\mathbf{\Gamma}_t^{\emptyset}(dx_{t:T})
=
p_{\theta,T}(dx_T)\prod_{s=t+1}^T K_{s-1}(x_s,dx_{s-1}),
\]
the unconditional reverse diffusion path law. Its $x_t$-marginal is $p_{\theta,t}$.

\emph{Part~(ii).}
Using the definition~\eqref{eq:G-q}, the product of tilted-to-unconditional ratios telescopes:
\[
\prod_{s=t+1}^T \frac{\omega_{s-1}(x_{s-1})}{\omega_s(x_s)}
=
\frac{\omega_t(x_t)}{\omega_T(x_T)}
=
\omega_t(x_t),
\]
since $\omega_T\equiv 1$.
Therefore
$W_t^{q}(x_{t:T}) = W_t^{\emptyset}(x_{t:T})\,\omega_t(x_t)$,
and taking the $x_t$-marginal yields $\omega_t(x_t)\,p_{\theta,t}(dx_t)$.
\end{proof}

\paragraph{Ratio recovery.}
Define $R_t(x) := \omega_t(x)^{-1}$.
By Eq.~\eqref{eq:omega-bounds}, $1 \le R_t(x) \le \delta^{-\lambda_0}$.

\begin{lemma}[Ratio identity]
\label{lem:ratio-identity}
For every bounded measurable $\varphi$,
\begin{equation}
\label{eq:ratio-identity}
p_{\theta,t}(\varphi)
=
\frac{\nu_t(R_t \varphi)}{\nu_t(R_t)}.
\end{equation}
\end{lemma}

\begin{proof}
By definition of $\nu_t$,
$\nu_t(R_t\varphi)
=
\frac{\int R_t(x)\varphi(x)\,\omega_t(x)\,p_{\theta,t}(dx)}
{\int \omega_t(x)\,p_{\theta,t}(dx)}
=
\frac{\int \varphi(x)\,p_{\theta,t}(dx)}
{\int \omega_t(x)\,p_{\theta,t}(dx)}$
since $R_t\,\omega_t\equiv 1$.
Similarly, $\nu_t(R_t) = {1}/{\int \omega_t\,dp_{\theta,t}}$.
Taking the ratio gives the identity~\eqref{eq:ratio-identity}.
\end{proof}

\subsubsection{Particle Approximation and Uniform Convergence}
\label{app:particle-approx}

We introduce one further regularity assumption:

\begin{enumerate}
    \item[(A5)] \textbf{\note{Total} boundedness of the class in (A1)--(A2).}
    For the same class $\mathcal H_t$ as in Assumptions~\textnormal{(A1)}--\textnormal{(A2)},
    $\mathcal H_t$ is uniformly bounded and totally bounded under the supremum norm
    $\|\cdot\|_\infty$.
\end{enumerate}

\emph{Remark}.
Assumption~\textnormal{(A4)} implies that the true posterior satisfies
$\eta_t(x)\in[\delta,1-\delta]$, so the Bayes logit
$g_t^\star(x)=\log\frac{\eta_t(x)}{1-\eta_t(x)}$
is uniformly bounded.
Hence, without excluding the target minimizer, one may work with a uniformly bounded hypothesis class $\mathcal G_t$ containing $g_t^\star$.
For such a class, the logistic loss transforms $\ell_1\circ g$ and $\ell_0\circ g$ are also uniformly bounded,
and it is therefore natural to impose total boundedness on the enlarged class $\mathcal H_t$ used in Assumptions~\textnormal{(A1)}--\textnormal{(A2)}.

\paragraph{SMC convergence for the tilted model.}
Consider an SMC sampler with multinomial resampling targeting the Feynman--Kac model with proposal $\{M_{t-1}^{\mathrm{reg}}\}$ and potentials $\{G_{t-1}^q\}$.
Let
$\nu_t^M := \sum_{j=1}^M \bar w_{q,t}^j \,\delta_{X_t^j}$
denote the normalized particle approximation of $\nu_t$.
Since the potentials are uniformly bounded (\autoref{lem:bounded-potentials}),
the standard strong law of large numbers (LLN) for particle filters with multinomial resampling
\citep[Prop.~11.4]{chopin2020introduction} gives,
for each fixed $t$ and each bounded measurable $\varphi$,
\[
\nu_t^M(\varphi)
\;\xrightarrow{\mathrm{a.s.}}\;
\nu_t(\varphi).
\]

To verify Assumption~\textnormal{(A2)}, we require convergence that is uniform over a class of test functions.
We therefore strengthen the above pointwise convergence to a uniform one over $\mathcal H_t$.

\begin{lemma}[Uniform SMC convergence over a totally bounded class]
\label{lem:uniform_smc}
Under Assumption~\textnormal{(A5)},
\[
\sup_{h\in\mathcal H_t}
|\nu_t^M(h)-\nu_t(h)|
\xrightarrow{\mathrm{a.s.}} 0.
\]
\end{lemma}

\begin{proof}
Let $\varepsilon>0$ and take a finite $\varepsilon$-net
$\{h_1,\dots,h_K\}\subset\mathcal H_t$ in $\|\cdot\|_\infty$.
For each $h\in\mathcal H_t$, choose $h_i$ with $\|h-h_i\|_\infty\le\varepsilon$. Then
\[
\sup_{h\in\mathcal H_t}
|\nu_t^M(h)-\nu_t(h)|
\le
\max_{1\le i\le K}
|\nu_t^M(h_i)-\nu_t(h_i)|
+2\varepsilon.
\]
Each term converges to zero almost surely by the SMC LLN.
Since the maximum is over finitely many $i$, it also converges to zero almost surely.
Letting $M\to\infty$ and then $\varepsilon\downarrow0$ yields the claim.
\end{proof}

\paragraph{Recovering $p_{\theta,t}$ via ratio transform.}
Define the particle-based estimator
\[
\widehat P_t^M(\varphi)
:=
\frac{\nu_t^M(R_t\varphi)}{\nu_t^M(R_t)}.
\]

\begin{theorem}[SMC verification of weighted uniform convergence]
\label{thm:smc_a2}
Assume \textnormal{(A4)} and \textnormal{(A5)}.
Then for each fixed $t$,
\[
\sup_{h\in\mathcal H_t}
\bigl|\widehat P_t^M(h)-p_{\theta,t}(h)\bigr|
\;\xrightarrow{\mathrm{a.s.}}\;0.
\]
In particular, defining
\[
\widetilde w_{t,j}^M
:=
\frac{\bar w_{q,t}^j\,R_t(X_t^j)}
{\sum_{k=1}^M \bar w_{q,t}^k\,R_t(X_t^k)},
\]
we have
\[
\widehat P_t^M(h)
=
\sum_{j=1}^M \widetilde w_{t,j}^M\,h(X_t^j),
\]
and hence Assumption~\textnormal{(A2)} holds for the weighted sample
$\{(X_t^j,\widetilde w_{t,j}^M)\}_{j=1}^M$ and the class $\mathcal H_t$.
\end{theorem}

\begin{proof}
Fix $t$.
By the SMC LLN, for every bounded measurable $\varphi$,
\[
\nu_t^M(\varphi)\xrightarrow{\mathrm{a.s.}}\nu_t(\varphi).
\]
In particular,
\[
\nu_t^M(R_t)\xrightarrow{\mathrm{a.s.}}\nu_t(R_t).
\]
Since $R_t\ge 1$ by Assumption~\textnormal{(A4)} and Eq.~\eqref{eq:omega-bounds}, we have $\nu_t(R_t)\ge 1$.
Hence, on an event of probability one, there exists $M_0$ such that for all $M\ge M_0$,
\[
\nu_t^M(R_t)\ge \frac12.
\]

Next, since $R_t$ is bounded and $\mathcal H_t$ is totally bounded,
the class $R_t\mathcal H_t:=\{R_t h:\;h\in\mathcal H_t\}$ is uniformly bounded and totally bounded.
An $\varepsilon$-net argument identical to that in \autoref{lem:uniform_smc} yields
\[
\sup_{h\in\mathcal H_t}
|\nu_t^M(R_t h)-\nu_t(R_t h)|
\xrightarrow{\mathrm{a.s.}}0.
\]

Using the ratio identity
\[
p_{\theta,t}(h)=\frac{\nu_t(R_t h)}{\nu_t(R_t)},
\]
we write
\[
\widehat P_t^M(h)-p_{\theta,t}(h)
=
\frac{\nu_t^M(R_t h)-\nu_t(R_t h)}{\nu_t^M(R_t)}
+
\nu_t(R_t h)
\left(
\frac{1}{\nu_t^M(R_t)}-\frac{1}{\nu_t(R_t)}
\right).
\]
Therefore, for all sufficiently large $M$,
\[
\sup_{h\in\mathcal H_t}
|\widehat P_t^M(h)-p_{\theta,t}(h)|
\le
2\sup_{h\in\mathcal H_t}
|\nu_t^M(R_t h)-\nu_t(R_t h)|
+
\sup_{h\in\mathcal H_t}|\nu_t(R_t h)|
\left|
\frac{1}{\nu_t^M(R_t)}-\frac{1}{\nu_t(R_t)}
\right|.
\]
Since
\[
\sup_{h\in\mathcal H_t}|\nu_t(R_t h)|
\le
\|R_t\|_\infty \sup_{h\in\mathcal H_t}\|h\|_\infty
<\infty,
\]
the first term converges to zero almost surely, and the second term also converges to zero almost surely because
$\nu_t^M(R_t)\to \nu_t(R_t)>0$.
This proves the result.
\end{proof}

Combining Theorems~\ref{thm:fixed_t_weighted_pu_consistency} and~\ref{thm:smc_a2}, we obtain the following.

\begin{corollary}[Fixed-time classifier consistency]
\label{cor:fixed-time-classifier-consistency}
Under Assumptions~\textnormal{(A1)} and \textnormal{(A3)--(A5)},
and assuming that a measurable empirical minimizer exists for each $(M,N)$,
for each fixed $t$ and along any sequence with $M,N\to\infty$,
\[
\bigl\|\widehat f_{t,M,N}-\eta_t\bigr\|_{L^1(P_t)}
\;\longrightarrow\;0
\qquad\text{a.s.}
\]
\end{corollary}

\begin{proof}
By \autoref{thm:smc_a2}, the weighted empirical measure constructed from
$\{(X_t^j,\widetilde w_{t,j}^M)\}_{j=1}^M$ satisfies Assumption~\textnormal{(A2)}.
The result then follows from \autoref{thm:fixed_t_weighted_pu_consistency} under Assumptions~\textnormal{(A1)} and \textnormal{(A3)}.
\end{proof}

\begin{lemma}[Asymptotic exactness under oracle classifiers]
\label{lem:oracle-exactness}
Assume \textnormal{(A4)}, and suppose the classifiers are oracle:
$f_t=\eta_t$ for all $t$.
Then the SMC algorithm with potentials $\{G_{t-1}^q\}$ and multinomial resampling satisfies, for every bounded measurable $\varphi$,
\[
\nu_0^M(\varphi)
\;\xrightarrow{\mathrm{a.s.}}\;
\nu_0(\varphi),
\qquad M\to\infty.
\]
\end{lemma}

\begin{proof}
Under oracle classifiers, the potentials $\{G_{t-1}^q\}$ are deterministic and,
by \autoref{lem:bounded-potentials}, uniformly bounded.
Therefore the standard strong law for SMC with multinomial resampling
\citep[Prop.~11.4]{chopin2020introduction} applies and yields the claim.
\end{proof}

\subsection{Asymptotic Exactness of the Reference-Based Instantiation with Learned Classifiers}
\label{app:plug-in-smc}

We now combine the consistency result of Appendix~\ref{app:pu-fixed-t} with the SMC approximation of Appendix~\ref{app:verify-a2} to establish asymptotic exactness of the reference-based instantiation of \autoref{alg:pu-steer-meta} when the classifiers are learned during sampling.
The proof-oriented variant introduced at the start of this appendix renders the learned classifiers exogenous, so the two results apply without modification.

For each $t=0,\dots,T$, let $\widehat f_t^{M',N}$ denote the classifier trained at time $t$ from the auxiliary system, using $M'$ unlabeled particles and $N$ samples from $\note{p_{\theta,t}(\cdot\mid c)}$.
Define its clipped version by
\[
\bar f_t^{M',N}(x)
:=
\min\{1-\delta,\max(\widehat f_t^{M',N}(x),\delta)\},
\]
where $\delta\in(0,1)$ is the constant in Assumption~\textnormal{(A4)}, and set
\[
\widehat\omega_t^{M',N}(x) := \bigl(1-\bar f_t^{M',N}(x)\bigr)^{\lambda_t}.
\]
The clipping yields $\delta\le\bar f_t^{M',N}(x)\le 1-\delta$ for all $x$, and hence $\delta^{\lambda_0}\le\widehat\omega_t^{M',N}(x)\le 1$.

The auxiliary system enters the analysis only through Assumption~\textnormal{(A2)}, namely as a weighted approximation of $p_{\theta,t}$.
Two constructions suffice.
The first is an auxiliary particle system run with the regularized Feynman--Kac construction of Appendix~\ref{app:verify-a2}, independent of the main one.
The second is direct simulation of the unconditional diffusion with unit weights.
Both yield a consistent approximation of $p_{\theta,t}$ and satisfy Assumption~\textnormal{(A2)}.

We emphasize that the specific form of the proposal kernel is not essential for the asymptotic argument.
What matters is that the associated importance weights define the correct Feynman--Kac change of measure and remain uniformly bounded.
Accordingly, we allow an arbitrary measurable proposal kernel
$\widehat M_{t-1}^{M',N}(x_t,\cdot)$ and regularize it by
\[
\widehat M_{t-1}^{M',N,\mathrm{reg}}(x_t,dx_{t-1})
:=
(1-\rho)\,\widehat M_{t-1}^{M',N}(x_t,dx_{t-1})
+
\rho\,K_{t-1}(x_t,dx_{t-1}),
\]
where $K_{t-1}$ is the reference reverse kernel from Appendix~\ref{app:verify-a2} and $\rho\in(0,1)$ is fixed.
As in that section, this mixture guarantees
\[
K_{t-1}(x_t,\cdot)\ll \widehat M_{t-1}^{M',N,\mathrm{reg}}(x_t,\cdot),
\qquad
\frac{dK_{t-1}(x_t,\cdot)}
     {d\widehat M_{t-1}^{M',N,\mathrm{reg}}(x_t,\cdot)}(x_{t-1})
\le \rho^{-1},
\]
so the Radon--Nikodym component of the incremental weight is uniformly bounded.
Thus, the role of the proposal in this subsection is purely instrumental:
any regularized proposal for which the corresponding potentials are bounded is admissible.

Given the learned classifiers, define the plug-in incremental weights by
\begin{equation}
\label{eq:plugin-potentials}
\begin{aligned}
\widehat G_{t-1}^{q,M',N}(x_t,x_{t-1})
&:=
\frac{dK_{t-1}(x_t,\cdot)}
     {d\widehat M_{t-1}^{M',N,\mathrm{reg}}(x_t,\cdot)}(x_{t-1})
\frac{\widehat \omega_{t-1}^{M',N}(x_{t-1})}
     {\widehat \omega_t^{M',N}(x_t)}, \\
\widehat G_{t-1}^{\emptyset,M',N}(x_t,x_{t-1})
&:=
\frac{dK_{t-1}(x_t,\cdot)}
     {d\widehat M_{t-1}^{M',N,\mathrm{reg}}(x_t,\cdot)}(x_{t-1}).
\end{aligned}
\end{equation}
The second family tracks the unconditional model, while the first tracks the tilted model obtained by replacing the oracle factor $(1-\eta_t)^{\lambda_t}$ with its learned plug-in counterpart.
The analysis therefore depends only on the behavior of the weights in Eq.~\eqref{eq:plugin-potentials}, not on any particular choice of proposal mechanism.

Finally, define the target density at terminal time by
\[
\nu_0(dx)
:=
\frac{(1-\eta_0(x))^{\lambda_0}\,p_{\theta,0}(dx)}
     {\int (1-\eta_0(x))^{\lambda_0}\,p_{\theta,0}(dx)}.
\]
Since the fixed-time consistency result applies at each $t$, in particular at $t=0$,
the asymptotic behavior of the algorithm is determined by whether the SMC procedure correctly tracks the corresponding Feynman--Kac model.
Under the above construction, this is ensured by the boundedness of the incremental weights together with the standard SMC convergence results.

\begin{theorem}[Asymptotic exactness of the reference-based instantiation in the analysis variant]
\label{thm:plugin_smc_exactness}
Assume \textnormal{(A1)} and \textnormal{(A3)--(A5)}.
Consider the proof-oriented variant described above, in which
\begin{itemize}
    \item for each $t=0,\dots,T$, the classifier $\widehat f_t^{M',N}$ is trained from an auxiliary sample independent of the particle system used in the final SMC approximation; and
    \item the final particle system uses the regularized proposal
    $\widehat M_{t-1}^{M',N,\mathrm{reg}}$ and plug-in potentials
    $\{\widehat G_{t-1}^{q,M',N}\}_{t=1}^T$ defined in Eq.~\eqref{eq:plugin-potentials}.
\end{itemize}
Then, for every bounded measurable $\varphi:\mathcal X_0\to\mathbb R$,
\[
\lim_{M',N\to\infty}\;
\lim_{M\to\infty}\;
\widehat \nu_0^{\,M\,|\,M',N}(\varphi)
=
\nu_0(\varphi)
\qquad\text{a.s.}
\]
This is the formal version of \autoref{thm:main-informal} for the reference-based instantiation of \autoref{sec:method:ref}.
\end{theorem}

\begin{proof}

We separate the total error into:
(I) the SMC approximation error for a fixed target, and
(II) the error induced by replacing the true target with a plug-in estimate.
For fixed $(M',N)$, let
\[
\widetilde\nu_0^{M',N}(dx)
:=
\frac{\widehat\omega_0^{M',N}(x)\,p_{\theta,0}(dx)}
{\int \widehat\omega_0^{M',N}(x)\,p_{\theta,0}(dx)}
\]
denote the terminal target induced by the learned classifiers,
and let $\widehat\nu_0^{M\,|\,M',N}$ denote the weighted empirical output of the main particle system with $M$ particles constructed using
$\{\widehat f_t^{M',N}\}_{t=0}^T$.
We decompose the error as
\begin{equation}
\label{eq:error-decomposition}
\widehat\nu_0^{M\,|\,M',N}(\varphi)-\nu_0(\varphi)
=
\underbrace{\bigl(\widehat\nu_0^{M\,|\,M',N}(\varphi)-\widetilde\nu_0^{M',N}(\varphi)\bigr)}_{\text{(I): SMC error}}
+
\underbrace{\bigl(\widetilde\nu_0^{M',N}(\varphi)-\nu_0(\varphi)\bigr)}_{\text{(II): plug-in target error}}.
\end{equation}

\paragraph{Step 1: SMC error for the main particle system.}
Fix $(M',N)$.
Since the classifiers are trained from an auxiliary sample independent of the main particle system, the learned quantities
\[
\Bigl\{\widehat\omega_t^{M',N},\;
\widehat M_{t-1}^{M',N,\mathrm{reg}},\;
\widehat G_{t-1}^{q,M',N},\;
\widehat G_{t-1}^{\emptyset,M',N}\Bigr\}
\]
are external to the main SMC approximation.
Therefore, conditional on the auxiliary sample, the main particle system is an ordinary Feynman--Kac model with deterministic proposal kernels and deterministic potentials.

By \autoref{prop:existence-common-proposal},
$0\le \widehat G_{t-1}^{\emptyset,M',N}\le \rho^{-1}.$
Moreover, by construction of the clipped classifier used to define
$\widehat\omega_t^{M',N}$, we have
$\delta^{\lambda_0}\le \widehat\omega_t^{M',N}(x)\le 1$ for all $t,x.$
Therefore
\[
0\le \widehat G_{t-1}^{q,M',N}(x_t,x_{t-1})
\le \rho^{-1}\delta^{-\lambda_0}.
\]
Thus the standard strong law for SMC with multinomial resampling \citep[Prop.~11.4]{chopin2020introduction} applies to the main particle system for this fixed environment, and yields
\[
\widehat\nu_0^{M\,|\,M',N}(\varphi)-\widetilde\nu_0^{M',N}(\varphi)
\to 0
\qquad \text{a.s. as } M\to\infty,
\]
conditionally on the auxiliary sample.
Since this holds for almost every auxiliary realization, it also holds jointly almost surely.

\paragraph{Step 2: Identification of the induced target.}
Applying the same telescoping argument as in \autoref{prop:path-identification}(ii) to the learned tempering factors,
\[
\prod_{s=1}^T
\frac{\widehat\omega_{s-1}^{M',N}(x_{s-1})}
{\widehat\omega_s^{M',N}(x_s)}
=
\frac{\widehat\omega_0^{M',N}(x_0)}{\widehat\omega_T^{M',N}(x_T)}
=
\widehat\omega_0^{M',N}(x_0),
\]
since $\lambda_T=0$ implies $\widehat\omega_T^{M',N}\equiv 1$.
Therefore the terminal marginal induced by the learned Feynman--Kac model is exactly $\widetilde\nu_0^{M',N}$.
In particular, the induced terminal target depends only on $\widehat\omega_0^{M',N}$,
although the intermediate classifiers may still affect the finite-particle behavior through the proposal and the incremental-weight decomposition.

\paragraph{Step 3: Plug-in convergence of the induced target.}
By construction of the auxiliary sample, Assumption~\textnormal{(A2)} holds at $t=0$.
Therefore \autoref{cor:fixed-time-classifier-consistency} applies at $t=0$, and yields
\[
\|\widehat f_0^{M',N}-\eta_0\|_{L^1(p_{\theta,0})}\to 0
\qquad \text{a.s. as } M',N\to\infty.
\]
Let
\[
\mathrm{clip}(u):=\min\{1-\delta,\max(u,\delta)\}, \qquad u\in\mathbb R.
\]
Then $\mathrm{clip}(\cdot)$ is $1$-Lipschitz, and by Assumption~\textnormal{(A4)} we have $\mathrm{clip}(\eta_0)=\eta_0$ a.s.

Since $\bar f_0^{M',N}=\mathrm{clip}(\widehat f_0^{M',N})$, it follows that
\[
\|\bar f_0^{M',N}-\eta_0\|_{L^1(p_{\theta,0})}
\le
\|\widehat f_0^{M',N}-\eta_0\|_{L^1(p_{\theta,0})}
\to 0
\qquad \text{a.s.}
\]
Now the map $u\mapsto (1-u)^{\lambda_0}$ is Lipschitz on the compact interval
$[\delta,1-\delta]$, so
\[
\|\widehat\omega_0^{M',N}-\omega_0\|_{L^1(p_{\theta,0})}\to 0
\qquad \text{a.s. as } M',N\to\infty.
\]
Since $\widehat\omega_0^{M',N}(x)\ge \delta^{\lambda_0}$ and
$\omega_0(x)\ge \delta^{\lambda_0}$ for all $x$, the corresponding normalizing constants are bounded away from zero.
Therefore, for every bounded measurable $\varphi$,
\[
\widetilde\nu_0^{M',N}(\varphi)\to \nu_0(\varphi)
\qquad \text{a.s. as } M',N\to\infty.
\]
This shows that term \textnormal{(II)} in Eq.~\eqref{eq:error-decomposition} vanishes.

\paragraph{Step 4: Conclusion.}
For each fixed $(M',N)$, Step~1 gives
$\widehat\nu_0^{M\,|\,M',N}(\varphi)-\widetilde\nu_0^{M',N}(\varphi) \to 0$ a.s. as $M\to\infty$.
Step~3 yields
$ \widetilde\nu_0^{M',N}(\varphi)\to \nu_0(\varphi)$ a.s. as $M',N\to\infty$.
Hence,
\[
\lim_{M',N\to\infty}\;
\lim_{M\to\infty}\;
\widehat\nu_0^{M\,|\,M',N}(\varphi)
=
\nu_0(\varphi)
\qquad \text{a.s.},
\]
which proves the claim.
\end{proof}

\paragraph{Relation to DAS and conceptual interpretation.}

\autoref{thm:plugin_smc_exactness} is closely related to the asymptotic exactness result for \note{DAS}~\citep{kim2025testtime}, but applies in a setting where the reward is replaced by a learned,
data-dependent discriminator.

\begin{enumerate}
\item[\textbf{(i)}] \textbf{Learned, time-dependent potentials.}
In DAS, the reward $R$ is fixed and known;
here, the potentials depend on learned classifiers $\widehat f_t^{M',N}$ that are both data-dependent and time-dependent.
Despite this additional randomness, asymptotic exactness is preserved.

\item[\textbf{(ii)}] \textbf{Dual-weight tracking and ratio recovery.}
The algorithm tracks two sets of weights ($w_q$ and $w_\emptyset$) on the same proposal.
The $w_\emptyset$ weights recover the unconditional marginal $p_{\theta,t}$ via a ratio transform that cancels the tilting factor algebraically,
regardless of classifier accuracy.

\item[\textbf{(iii)}] \textbf{Cancellation of intermediate errors.}
The telescoping structure ensures that the terminal marginal depends only on $\widehat\omega_0$, not on intermediate classifiers.
Combined with the ratio recovery in (ii), \note{in the particle limit,} classifier errors at intermediate times $t>0$ do not affect the unlabeled sample quality nor the terminal output density.
\end{enumerate}

\clearpage

\subsection{Extension to the Score-Based Instantiation: Dual Weighted Particle Systems}
\label{app:alg2-dual-weight-exactness}

We now treat the score-based instantiation, in which the positive density $P_t^+ = p_{\theta,t}(\cdot\mid c)$ is approximated by a second weighted particle system on the same proposal trajectories (cf.\ outline at the start of this appendix).
Throughout this subsection, we retain all notation from Appendix~\ref{app:verify-a2} and Appendix~\ref{app:plug-in-smc}.

\paragraph{Conditional reverse kernel.}
For $t=1,\dots,T$, define
$K_{t-1}^{c}(x_t,dx_{t-1}) := p_\theta(dx_{t-1}\mid x_t,c)$.
Since both conditional and unconditional diffusion marginals share the same Gaussian limit, we take $p_{\theta,T}(\cdot\mid c)=p_{\theta,T}$.

\paragraph{Three-way common proposal.}
Let $\widetilde M_{t-1}$ be any nominal guided proposal.
For $\rho_\emptyset,\rho_c\in(0,1)$ with $\rho_\emptyset+\rho_c<1$, define
\begin{equation}
\label{eq:alg2-common-proposal}
M_{t-1}^{\mathrm{alg2}}(x_t,dx_{t-1})
:=
(1-\rho_\emptyset-\rho_c)\,\widetilde M_{t-1}(x_t,dx_{t-1})
+
\rho_\emptyset\,K_{t-1}(x_t,dx_{t-1})
+
\rho_c\,K_{t-1}^{c}(x_t,dx_{t-1}).
\end{equation}

\begin{lemma}[Three-way common proposal]
\label{prop:alg2-common-proposal}
The kernels $\{M_{t-1}^{\mathrm{alg2}}\}$ satisfy
\[
\frac{dK_{t-1}(x_t,\cdot)}
     {dM_{t-1}^{\mathrm{alg2}}(x_t,\cdot)}
\le \rho_\emptyset^{-1},
\qquad
\frac{dK_{t-1}^{c}(x_t,\cdot)}
     {dM_{t-1}^{\mathrm{alg2}}(x_t,\cdot)}
\le \rho_c^{-1}.
\]
\end{lemma}

\begin{proof}
By construction, $M_{t-1}^{\mathrm{alg2}} \ge \rho_\emptyset K_{t-1}$ and $M_{t-1}^{\mathrm{alg2}} \ge \rho_c K_{t-1}^c$.
The Radon--Nikodym bounds follow.
\end{proof}

\paragraph{Three Feynman--Kac models.}
Define
\begin{align}
G_{t-1}^{\emptyset}(x_t,x_{t-1})
&:=
\frac{dK_{t-1}(x_t,\cdot)}
     {dM_{t-1}^{\mathrm{alg2}}(x_t,\cdot)}(x_{t-1}),
\label{eq:alg2-G-empty}
\\
G_{t-1}^{c}(x_t,x_{t-1})
&:=
\frac{dK_{t-1}^{c}(x_t,\cdot)}
     {dM_{t-1}^{\mathrm{alg2}}(x_t,\cdot)}(x_{t-1}),
\label{eq:alg2-G-c}
\\
G_{t-1}^{q}(x_t,x_{t-1})
&:=
G_{t-1}^{\emptyset}(x_t,x_{t-1})
\frac{\omega_{t-1}(x_{t-1})}{\omega_t(x_t)}.
\label{eq:alg2-G-q}
\end{align}

\begin{lemma}[Uniform boundedness of three potentials]
\label{lem:alg2-bounded-potentials}
Under \textnormal{(A4)} (with the full two-sided bound),
\[
0\le G_{t-1}^{\emptyset}\le \rho_\emptyset^{-1},
\qquad
0\le G_{t-1}^{c}\le \rho_c^{-1},
\qquad
0\le G_{t-1}^{q}\le \rho_\emptyset^{-1}\delta^{-\lambda_0}.
\]
\end{lemma}

\begin{proof}
The first two bounds follow from \autoref{prop:alg2-common-proposal};
the third follows as in \autoref{lem:bounded-potentials}.
\end{proof}

\paragraph{Positive-side particle approximation.}
All three Feynman--Kac models are built on the same proposal trajectories
$X_t^1,\dots,X_t^M$, but the corresponding target marginals differ.
To justify a uniform approximation of the positive conditional law
$P_t^+=p_{\theta,t}(\cdot\mid c)$, we proceed exactly as in \autoref{thm:smc_a2}:
rather than invoking a direct SMC law of large numbers for the $w_c$-weighted empirical measure, we recover $P_t^+$ from the tilted target $\nu_t$ by a ratio transform.

Recall that
\[
\nu_t(dx)
=
\frac{\omega_t(x)\,p_{\theta,t}(dx)}
{\int \omega_t(x)\,p_{\theta,t}(dx)},
\qquad
P_t^+(dx)
=
p_{\theta,t}(dx\mid c),
\]
and define
\[
S_t(x) := \frac{\eta_t(x)}{\omega_t(x)}.
\]
Since $\eta_t(x)\in[\delta,1-\delta]$ under Assumption~\textnormal{(A4)} and
$\omega_t(x)=(1-\eta_t(x))^{\lambda_t}$ $p_{\theta,t}$-a.e., we have
\[
0 \le S_t(x) \le \delta^{-\lambda_0}
\qquad
p_{\theta,t}\text{-a.e.}
\]
Hence $S_t$ is bounded, and therefore the class
$S_t\mathcal H_t:=\{S_t h:h\in\mathcal H_t\}$ is uniformly bounded and totally bounded whenever $\mathcal H_t$ is.
Define the particle-based estimator
\[
\widehat P_t^{+,M}(h)
:=
\frac{\nu_t^M(S_t h)}{\nu_t^M(S_t)},
\qquad h\in\mathcal H_t,
\qquad
\text{where}
\qquad
\nu_t^M
:=
\sum_{j=1}^M \bar w_{q,t}^j\,\delta_{X_t^j}
\]
is the tilted particle approximation from Appendix~\ref{app:verify-a2}.
Equivalently, if we define
\[
\bar w_{c,t}^j
:=
\frac{\bar w_{q,t}^j\,S_t(X_t^j)}
{\sum_{k=1}^M \bar w_{q,t}^k\,S_t(X_t^k)},
\qquad j=1,\dots,M,
\qquad
\text{then}
\qquad
\widehat P_t^{+,M}(h)
=
\sum_{j=1}^M \bar w_{c,t}^j\,h(X_t^j).
\]

\begin{theorem}[SMC verification of positive-side weighted uniform convergence]
\label{thm:alg2-positive-smc}
Assume \textnormal{(A4)} and \textnormal{(A5)}.
Then for each fixed $t$,
\[
\sup_{h\in\mathcal H_t}
\bigl|
\widehat P_t^{+,M}(h)-P_t^+(h)
\bigr|
\;\xrightarrow[M\to\infty]{\mathrm{a.s.}}\;
0.
\]
Equivalently,
\[
\sup_{h\in\mathcal H_t}
\left|
\sum_{j=1}^M \bar w_{c,t}^j\,h(X_t^j)-P_t^+(h)
\right|
\;\xrightarrow[M\to\infty]{\mathrm{a.s.}}\;
0.
\]
\end{theorem}

\begin{proof}
Fix $t$.
Since the tilted potentials $\{G_{t-1}^q\}$ are uniformly bounded by
\autoref{lem:alg2-bounded-potentials}, the same SMC law of large numbers used in Appendix~\ref{app:verify-a2} applies to $\nu_t^M$.
Therefore, for every bounded measurable $\varphi$,
\[
\nu_t^M(\varphi)
\;\xrightarrow{\mathrm{a.s.}}\;
\nu_t(\varphi).
\]

We first record the positive-side ratio identity.
For every bounded measurable $h$,
\[
\nu_t(S_t h)
=
\frac{\int S_t(x)h(x)\omega_t(x)\,p_{\theta,t}(dx)}
{\int \omega_t(x)\,p_{\theta,t}(dx)}
=
\frac{\int \eta_t(x)h(x)\,p_{\theta,t}(dx)}
{\int \omega_t(x)\,p_{\theta,t}(dx)}.
\]
Since $\eta_t(x)=\note{p_{\theta,t}(c\mid x)}$ and
$\int \eta_t(x)\,p_{\theta,t}(dx)=P_t(c)$, we obtain
\[
\nu_t(S_t h)
=
\frac{P_t(c)\,P_t^+(h)}
{\int \omega_t(x)\,p_{\theta,t}(dx)},
\qquad
\nu_t(S_t)
=
\frac{P_t(c)}
{\int \omega_t(x)\,p_{\theta,t}(dx)}.
\]
Hence
\[
P_t^+(h)
=
\frac{\nu_t(S_t h)}{\nu_t(S_t)}.
\]

Next, since $S_t$ is bounded and $\mathcal H_t$ is totally bounded,
the class
\[
S_t\mathcal H_t:=\{S_t h:h\in\mathcal H_t\}
\]
is also uniformly bounded and totally bounded in the supremum norm.
Therefore, exactly as in \autoref{lem:uniform_smc},
$
\sup_{h\in\mathcal H_t}
\bigl|
\nu_t^M(S_t h)-\nu_t(S_t h)
\bigr|
\;\xrightarrow{\mathrm{a.s.}}\;
0.
$
In particular,
$\nu_t^M(S_t)\xrightarrow{\mathrm{a.s.}}\nu_t(S_t).$

Moreover, $\nu_t(S_t)>0$ because
\[
\nu_t(S_t)
=
\frac{P_t(c)}{\int \omega_t(x)\,p_{\theta,t}(dx)},
\]
and $P_t(c)\in(0,1)$ by the standing assumptions of \note{Appendix~\ref{app:pu-fixed-t}}.
Hence, on an event of probability one, there exists $M_0$ such that
$\nu_t^M(S_t)\ge \nu_t(S_t)/2>0$ for all $M\ge M_0$.

Using the ratio identity above, for all sufficiently large $M\ge M_0$ we have
\[
\widehat P_t^{+,M}(h)-P_t^+(h)
=
\frac{\nu_t^M(S_t h)-\nu_t(S_t h)}{\nu_t^M(S_t)}
+
\nu_t(S_t h)
\left(
\frac{1}{\nu_t^M(S_t)}-\frac{1}{\nu_t(S_t)}
\right).
\]
Therefore,
\[
\sup_{h\in\mathcal H_t}\!
\bigl|
\widehat P_t^{+,M}(h)-P_t^+(h)
\bigr|
\mkern-2mu\le\mkern-2mu
\frac{2}{\nu_t(S_t)}\!
\sup_{h\in\mathcal H_t}\!
\left|
\nu_t^M(S_t h)-\nu_t(S_t h)
\right|
\mkern-2mu+\mkern-2mu
\sup_{h\in\mathcal H_t}\!|\nu_t(S_t h)|\!
\left|
\frac{1}{\nu_t^M(S_t)}-\frac{1}{\nu_t(S_t)}
\right|.
\]
Since
\[
\sup_{h\in\mathcal H_t}|\nu_t(S_t h)|
\le
\|S_t\|_\infty \sup_{h\in\mathcal H_t}\|h\|_\infty
<\infty,
\]
the first term converges to zero almost surely by the uniform convergence over
$S_t\mathcal H_t$, and the second term also converges to zero almost surely because $\nu_t^M(S_t)\to \nu_t(S_t)>0$.
This proves
\[
\sup_{h\in\mathcal H_t}
\bigl|
\widehat P_t^{+,M}(h)-P_t^+(h)
\bigr|
\to 0
\qquad\text{a.s.}
\]
Finally, by the definition of $\bar w_{c,t}^j$,
$\widehat P_t^{+,M}(h) = \sum_{j=1}^M \bar w_{c,t}^j\,h(X_t^j),$
so the equivalent weighted-empirical formulation follows.
\end{proof}

\paragraph{Dual weighted PU consistency.}
Using the positive-side estimator from \autoref{thm:alg2-positive-smc} and the unlabeled estimator from \autoref{thm:smc_a2},
define the dual weighted empirical PU risk by
\begin{equation}
\widehat L_{t,M}^{\mathrm{dual}}(g)
:=
-\pi_t \sum_{j=1}^M \bar w_{c,t}^j\, g(X_t^j)
+
\sum_{j=1}^M \widetilde w_t^j\, \ell_0\bigl(g(X_t^j)\bigr),
\qquad g\in\mathcal G_t,
\label{eq:alg2-dual-pu-risk}
\end{equation}
where $\{\bar w_{c,t}^j\}_{j=1}^M$ are the normalized weights defining the positive conditional estimator
$\widehat P_t^{+,M}(h) = \sum_{j=1}^M \bar w_{c,t}^j\,h(X_t^j),$
and $\{\widetilde w_t^j\}_{j=1}^M$ are the normalized weights defining the ratio-recovered unlabeled estimator
$\widehat P_t^{M}(h) = \sum_{j=1}^M \widetilde w_t^j\,h(X_t^j).$
Its population counterpart is
\[
L_t^{\mathrm{PU}}(g)
:=
-\pi_t\,P_t^+\!\bigl(g\bigr)+P_t\!\bigl(\ell_0\circ g\bigr),
\]
which is exactly the same population PU risk as in Appendix~\ref{app:pu-fixed-t}.

Assume, as in Appendix~\ref{app:pu-fixed-t},
that for every $M$ under consideration there exists a measurable empirical minimizer
$ \widehat g_{t,M}^{\mathrm{dual}} \in \arg\min_{g\in\mathcal G_t}\widehat L_{t,M}^{\mathrm{dual}}(g),$
and define
$\widehat f_{t,M}^{\mathrm{dual}}(x) := \sigma\!\bigl(\widehat g_{t,M}^{\mathrm{dual}}(x)\bigr).$

\begin{theorem}[Fixed-time consistency under dual weighted particles]
\label{thm:alg2-fixed-time-consistency}
Under \textnormal{(A3)--(A5)},
for each fixed $t$,
\[
\bigl\|
\widehat f_{t,M}^{\mathrm{dual}}-\eta_t
\bigr\|_{L^1(P_t)}
\;\xrightarrow[M\to\infty]{\mathrm{a.s.}}\;
0.
\]
\end{theorem}

\begin{proof}
Fix $t$.
By \autoref{thm:alg2-positive-smc}, for the class $\mathcal H_t$,
$
\sup_{h\in\mathcal H_t}
\bigl|
\widehat P_t^{+,M}(h)-P_t^+(h)
\bigr|
\;\xrightarrow{\mathrm{a.s.}}\;
0.
$
By \autoref{thm:smc_a2},
$
\sup_{h\in\mathcal H_t}
\bigl|
\widehat P_t^{M}(h)-P_t(h)
\bigr|
\;\xrightarrow{\mathrm{a.s.}}\;
0.
$
Since $\mathcal H_t$ contains $\mathcal G_t$ and $\{\ell_0\circ g:g\in\mathcal G_t\}$ by Assumption~\textnormal{(\note{A1})} (inherited from Appendix~\ref{app:pu-fixed-t}),
it follows that
\begin{align*}
\Delta_M^{\mathrm{dual}}
&:=
\sup_{g\in\mathcal G_t}
\bigl|
\widehat L_{t,M}^{\mathrm{dual}}(g)-L_t^{\mathrm{PU}}(g)
\bigr|
\\
&\le
\pi_t
\sup_{g\in\mathcal G_t}
\bigl|
\widehat P_t^{+,M}(g)-P_t^+(g)
\bigr|
+
\sup_{g\in\mathcal G_t}
\bigl|
\widehat P_t^{M}(\ell_0\circ g)-P_t(\ell_0\circ g)
\bigr|
\\
&\xrightarrow{\mathrm{a.s.}} 0.
\end{align*}

Thus the dual weighted empirical risk satisfies the same uniform convergence property as the empirical risk in \autoref{thm:fixed_t_weighted_pu_consistency}.
Moreover, its population target is the same risk $L_t^{\mathrm{PU}}$.
Therefore the proof of
\autoref{thm:fixed_t_weighted_pu_consistency}
applies verbatim with $\Delta_{M,N}$ replaced by
$\Delta_M^{\mathrm{dual}}$.
Using Assumption~\textnormal{(A3)} and the existence of the measurable empirical minimizer $\widehat g_{t,M}^{\mathrm{dual}}$, we obtain
\[
\bigl\|
\widehat f_{t,M}^{\mathrm{dual}}-\eta_t
\bigr\|_{L^1(P_t)}
\;\xrightarrow{\mathrm{a.s.}}\;
0.
\]
\end{proof}

\paragraph{Proof-oriented variant and its asymptotic exactness.}
The proof-oriented variant of Appendix~\ref{app:plug-in-smc} carries over: at each $t$, the classifier is trained from an auxiliary system, here a dual-weight particle system of size $M'$ independent of the main one.

For each $t$, let $\widehat f_t^{M'}$ denote the classifier trained from the auxiliary system, and define its clipped version by
\[
\bar f_t^{M'}(x)
:=
\min\{1-\delta,\max(\widehat f_t^{M'}(x),\delta)\},
\]
where $\delta\in(0,1)$ is the constant in Assumption~\textnormal{(A4)}, and set
\[
\widehat\omega_t^{M'}(x)
:=
\bigl(1-\bar f_t^{M'}(x)\bigr)^{\lambda_t}.
\]
This yields $\delta^{\lambda_0}\le\widehat\omega_t^{M'}(x)\le 1$ for all $t$ and $x$.

Let the main particle system use these clipped classifiers as exogenous objects,
together with any measurable proposal kernels regularized as in Appendix~\ref{app:alg2-dual-weight-exactness}, and denote by
$\widehat\nu_0^{M\,|\,M'}$
its weighted empirical output at time $0$.
Further define the plug-in terminal target
\[
\widetilde\nu_0^{M'}(dx)
:=
\frac{\widehat\omega_0^{M'}(x)\,p_{\theta,0}(dx)}
{\int \widehat\omega_0^{M'}(x)\,p_{\theta,0}(dx)}.
\]

\begin{theorem}[Asymptotic exactness of the score-based instantiation in the analysis variant]
\label{thm:alg2-plugin-exactness}
Assume \textnormal{(A3)--(A5)}.
Consider the proof-oriented variant described above,
in which the classifiers are trained from an auxiliary dual-weight particle system independent of the main particle system used in the final SMC approximation.
Assume moreover that, for each auxiliary size $M'$ and each $t$, a measurable empirical minimizer defining $\widehat f_t^{M'}$ exists.

Then, for every bounded measurable
$\varphi:\mathcal X_0\to\mathbb R$,
\[
\lim_{M'\to\infty}\;
\lim_{M\to\infty}\;
\widehat\nu_0^{M\,|\,M'}(\varphi)
=
\nu_0(\varphi)
\qquad\text{a.s.},
\]
where
\[
\nu_0(dx)
:=
\frac{(1-\eta_0(x))^{\lambda_0}\,p_{\theta,0}(dx)}
{\int (1-\eta_0(x))^{\lambda_0}\,p_{\theta,0}(dx)}.
\]
This establishes the extension of \autoref{thm:main-informal} to the score-based instantiation of \autoref{sec:method:score}, as stated informally at the end of \autoref{sec:theory}.
\end{theorem}

\begin{proof}
Fix a bounded measurable test function $\varphi$.
As in the proof of \autoref{thm:plugin_smc_exactness}, decompose
\begin{equation}
\label{eq:alg2-error-decomposition}
\widehat\nu_0^{M\,|\,M'}(\varphi)-\nu_0(\varphi)
=
\underbrace{\Bigl(
\widehat\nu_0^{M\,|\,M'}(\varphi)-\widetilde\nu_0^{M'}(\varphi)
\Bigr)}_{\textnormal{(I): SMC error}}
+
\underbrace{\Bigl(
\widetilde\nu_0^{M'}(\varphi)-\nu_0(\varphi)
\Bigr)}_{\textnormal{(II): plug-in target error}}.
\end{equation}

\paragraph{Step 1: SMC error for fixed learned classifiers.}
Fix $M'$.
Because the auxiliary particle system is independent of the main one, the collection
\[
\Bigl\{
\bar f_t^{M'},\widehat\omega_t^{M'}
\Bigr\}_{t=0}^T
\]
is deterministic from the viewpoint of the main particle system.
Hence, conditional on the auxiliary system, the main particle system is an ordinary Feynman--Kac model with deterministic proposal kernels and deterministic potentials.

By the three-way regularized common-proposal construction of Appendix~\ref{app:alg2-dual-weight-exactness}, the Radon--Nikodym part of the incremental weights is uniformly bounded.
Since
\[
\delta^{\lambda_0}
\le
\widehat\omega_t^{M'}(x)
\le 1
\qquad\text{for all }t,x,
\]
the corresponding plug-in tilted potentials are also uniformly bounded.
Therefore the standard strong law for SMC with multinomial resampling applies,
exactly as in Appendix~\ref{app:plug-in-smc}, and yields
\[
\widehat\nu_0^{M\,|\,M'}(\varphi)
-
\widetilde\nu_0^{M'}(\varphi)
\;\xrightarrow[M\to\infty]{\mathrm{a.s.}}\;
0
\]
conditionally on the auxiliary system, hence jointly almost surely.

\paragraph{Step 2: Identification of the induced terminal target.}
For fixed $M'$, the plug-in tilted potentials have the same telescoping structure as in the reference-based instantiation:
\[
\prod_{s=1}^T
\frac{\widehat\omega_{s-1}^{M'}(x_{s-1})}
     {\widehat\omega_s^{M'}(x_s)}
=
\frac{\widehat\omega_0^{M'}(x_0)}
     {\widehat\omega_T^{M'}(x_T)}.
\]
Since $\lambda_T=0$, we have $\widehat\omega_T^{M'}\equiv 1$.
Hence the terminal marginal induced by the learned Feynman--Kac model is exactly
\[
\widetilde\nu_0^{M'}(dx)
=
\frac{\widehat\omega_0^{M'}(x)\,p_{\theta,0}(dx)}
{\int \widehat\omega_0^{M'}(x)\,p_{\theta,0}(dx)}.
\]
As in the reference-based instantiation, intermediate classifier errors may affect the proposal mechanism and finite-particle performance, but not the form of the induced terminal target.

\paragraph{Step 3: Plug-in convergence of the terminal target.}
By \autoref{thm:alg2-fixed-time-consistency} applied at $t=0$,
\[
\|\widehat f_0^{M'}-\eta_0\|_{L^1(P_0)}
\;\xrightarrow[M'\to\infty]{\mathrm{a.s.}}\;
0.
\]
Since the clipping map
\[
\mathrm{clip}(u):=\min\{1-\delta,\max(u,\delta)\}
\]
is $1$-Lipschitz and, by Assumption~\textnormal{(A4)},
$\eta_0(x)\in[\delta,1-\delta]$ for $P_0$-a.e.\ $x$, we have
$\mathrm{clip}(\eta_0)=\eta_0$ $P_0$-a.e., and therefore
\[
\|\bar f_0^{M'}-\eta_0\|_{L^1(P_0)}
\le
\|\widehat f_0^{M'}-\eta_0\|_{L^1(P_0)}
\;\xrightarrow[M'\to\infty]{\mathrm{a.s.}}\;
0.
\]

Now the map
\[
u\mapsto (1-u)^{\lambda_0}
\]
is Lipschitz on the compact interval $[\delta,1-\delta]$.
Hence
\[
\|\widehat\omega_0^{M'}-\omega_0\|_{L^1(P_0)}
\;\xrightarrow[M'\to\infty]{\mathrm{a.s.}}\;
0,
\qquad
\omega_0(x):=(1-\eta_0(x))^{\lambda_0}.
\]
Moreover,
\[
\widehat\omega_0^{M'}(x)\ge \delta^{\lambda_0},
\qquad
\omega_0(x)\ge \delta^{\lambda_0},
\]
so the corresponding normalizing constants are uniformly bounded away from zero.
Therefore, for every bounded measurable $\varphi$,
\[
\widetilde\nu_0^{M'}(\varphi)
\;\xrightarrow[M'\to\infty]{\mathrm{a.s.}}\;
\nu_0(\varphi).
\]
Thus term \textnormal{(II)} in Eq.~\eqref{eq:alg2-error-decomposition} vanishes.

\paragraph{Step 4: Conclusion.}
For each fixed $M'$, Step~1 gives
\[
\lim_{M\to\infty}\widehat\nu_0^{M\,|\,M'}(\varphi)
=
\widetilde\nu_0^{M'}(\varphi)
\qquad\text{a.s.}
\]
Step~3 gives
\[
\widetilde\nu_0^{M'}(\varphi)
\;\xrightarrow[M'\to\infty]{\mathrm{a.s.}}\;
\nu_0(\varphi).
\]
Combining these two limits proves
\[
\lim_{M'\to\infty}\;
\lim_{M\to\infty}\;
\widehat\nu_0^{M\,|\,M'}(\varphi)
=
\nu_0(\varphi)
\qquad\text{a.s.},
\]
as claimed.
\end{proof}

\clearpage

%% file: appendix_exp.tex
\section{Experimental Details and Additional Results}
\label{app:exp}

This appendix expands \autoref{sec:experiments} with full setups and additional results.
\autoref{app:exp:gmm} covers the score-based instantiation (\autoref{sec:method:score}, \autoref{alg:neg_das_pu_singlechain}),
and \autoref{app:exp:cifar10} the reference-based instantiation (\autoref{sec:method:ref}, \autoref{alg:neg_das_pu_fixed}).
\note{\autoref{app:exp:sd} scales the score-based instantiation to Stable Diffusion.}
Code and per-experiment documentation are provided in the supplementary material.

\paragraph{Compute.}
GMM experiments run on CPU (Intel Xeon Max 9480).
CIFAR-10 \note{and Stable Diffusion} experiments run on a single NVIDIA GH200 GPU.
Python and library versions are pinned in the supplementary material.

\subsection{GMM with Score-Based Setting}
\label{app:exp:gmm}

This appendix gives the full specification of the GMM benchmark in \autoref{sec:exp:gmm}.
The closed-form structure of the GMM yields exact expressions for conditional and unconditional scores $\nabla_{x_t}\log p_t(x_t\mid\cdot)$, which lets us evaluate samplers against ground-truth quantities without confounding from a learned score model.

\subsubsection{Setup}
\label{app:exp:gmm:setup}

\paragraph{Data density and class partition.}
The unconditional density $p(x)$ is a $12$-component GMM in $\mathbb{R}^{3}$, generated as follows.
\textbf{Mode means.} For $i=0,\ldots,11$, let $\tilde\theta_i = 2\pi i/12 + \delta_i$ with $\delta_i \sim U(-0.2, 0.2)$, $r_i = 10(1+\varepsilon_i)$ with $\varepsilon_i \sim U(-0.3, 0.3)$, and $z_i \sim U(-1, 1)$, all independent. Set $\tilde\mu_i = (r_i\cos\tilde\theta_i,\; r_i\sin\tilde\theta_i,\; z_i)^\top$ and center the means by $\mu_i = \tilde\mu_i - \tfrac{1}{12}\sum_{j} \tilde\mu_j$.
\textbf{Mode covariances.} For each mode $i$, define base unit vectors $\mathbf{e}_r = (\cos\tilde\theta_i, \sin\tilde\theta_i, 0)^\top$ and $\mathbf{e}_t = (-\sin\tilde\theta_i, \cos\tilde\theta_i, 0)^\top$. Draw $\rho \sim U(0, 0.3)$, $\alpha \sim U(0, 2\pi)$, and set $\phi = \alpha\rho$. Rotate the base frame in the $(\mathbf{e}_r, \mathbf{e}_t)$ plane: $\mathbf{u}_0 = \cos\phi\, \mathbf{e}_r + \sin\phi\, \mathbf{e}_t$ and $\mathbf{v}_0 = -\sin\phi\, \mathbf{e}_r + \cos\phi\, \mathbf{e}_t$. Perturb the first axis with $\boldsymbol{\xi} \sim \mathcal{N}(\mathbf{0}, 0.01\, I_3)$ and normalize: $\mathbf{u} = (\mathbf{u}_0 + \boldsymbol{\xi})/\|\mathbf{u}_0 + \boldsymbol{\xi}\|$. Re-orthogonalize via Gram--Schmidt: $\mathbf{v} = (\mathbf{v}_0 - (\mathbf{v}_0^\top \mathbf{u})\mathbf{u})/\|\mathbf{v}_0 - (\mathbf{v}_0^\top \mathbf{u})\mathbf{u}\|$ and $\mathbf{w} = \mathbf{u} \times \mathbf{v}$. Draw eigenvalues $\lambda_r \sim U(1.0, 2.5)$, $\lambda_t \sim U(0.3, 0.8)$, $\lambda_z \sim U(0.5, 2.0)$ independently, and set $\Sigma_i = R_i\, \mathrm{diag}(\lambda_r, \lambda_t, \lambda_z)\, R_i^\top$ with $R_i = [\mathbf{u}\;\, \mathbf{v}\;\, \mathbf{w}]$.
\textbf{Weights and class partition.} Draw $(w_0, \ldots, w_{11}) \sim \mathrm{Dir}(5 \cdot \mathbf{1}_{12})$. Select an index set $\mathcal{P} \subset \{0, \ldots, 11\}$ of size $3$ uniformly without replacement, and let $\mathcal{N}$ be its complement. Renormalize within each class: $w_i^{+} = w_i / \sum_{j \in \mathcal{P}} w_j$ for $i \in \mathcal{P}$ and $w_i^{-} = w_i / \sum_{j \in \mathcal{N}} w_j$ for $i \in \mathcal{N}$. The joint density is $p(x) = \pi \sum_{i \in \mathcal{P}} w_i^{+}\, \mathcal{N}(x; \mu_i, \Sigma_i) + (1 - \pi) \sum_{i \in \mathcal{N}} w_i^{-}\, \mathcal{N}(x; \mu_i, \Sigma_i)$ with class prior $\pi = 3/12 = 0.25$. All randomness above is controlled by a single integer layout seed.

\begin{figure}[H]
\centering
\includegraphics[width=0.95\linewidth]{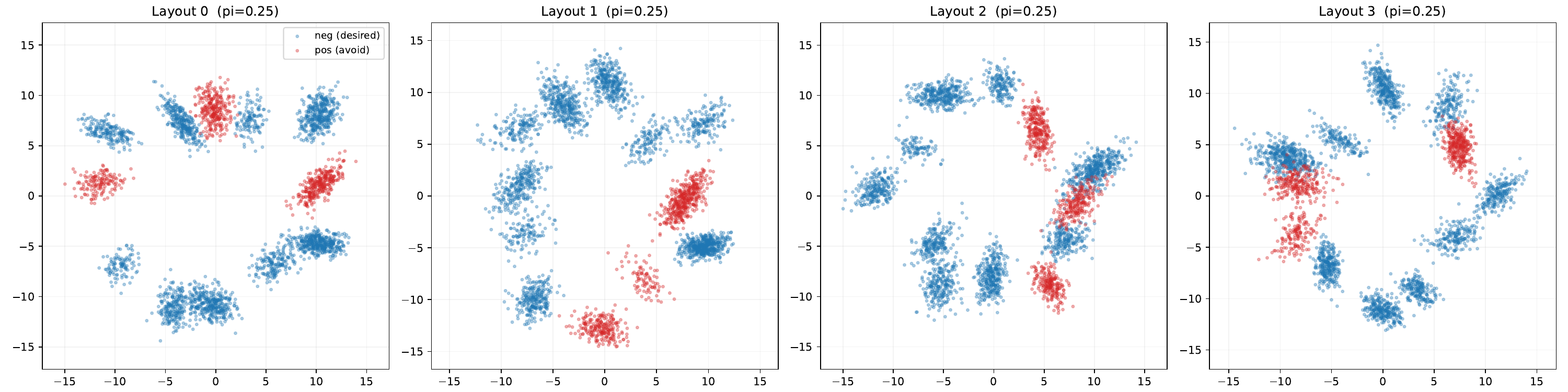}
\caption{Four GMM layouts drawn at random.
Each panel projects $3000$ unconditional samples onto the $xy$-plane.
Red points are samples from the avoided component density $p(\cdot\mid c)$ and blue points are samples from $p(\cdot\mid\neg c)$.
}
\label{fig:gmm-layouts}
\end{figure}

\paragraph{Diffusion process and ground-truth quantities.}
We use the cosine schedule of variance-preserving DDPM~\cite{pmlr-v139-nichol21a}, 
offset $s=0.008$, $\beta_{\max}=0.999$, and $T=100$ timesteps.
Each marginal $p_t(\cdot\mid c)$ and $p_t(\cdot\mid\neg c)$ is a closed-form GMM with means $\sqrt{\bar\alpha_t}\mu_k$ and covariances $\bar\alpha_t\Sigma_k+(1-\bar\alpha_t)I$, so the conditional scores and the posterior $p_t(c\mid x_t)$ are analytic.
DNG and ours in \autoref{alg:neg_das_pu_singlechain} receive the unconditional score $\nabla_{x_t}\log p_t(x_t)$ and the conditional score $\nabla_{x_t}\log p_t(x_t\mid c)$.

\clearpage

\paragraph{Methods and sweep grids.}
We compare four samplers below.
\begin{itemize}
\item \textbf{Oracle.} DDPM sampling with $\nabla_{x_t}\log p_t(x_t\mid\neg c)$.
\item \textbf{NegCFG.} Repulsion strength $\eta\in\{0,0.1,0.2,0.5,0.7,1.0,1.5,2.0,3.0\}$.
\item \textbf{DNG} \citep{koulischer2025dynamic}.
The CFG strength is $-\eta \  \hat p_t(c|x_t) / (1-\hat p_t(c|x_t))$.
We follow their implementation in \citep{koulischer2025dynamic}, which estimates $\hat p_t(c|x_t)$ by
$
\log \hat p_{t-1}(c| x_{t-1})
\leftarrow
\log \hat p_t(c| x_t)
-
\tfrac{\tau}{2\sigma_t^2}
\left(
\|x_{t-1}-\mu_\theta(x_t,t| c)\|^2
-
\|x_{t-1}-\mu_\theta(x_t,t)\|^2
\right) + \tfrac{\delta}{2\sigma_t^2},
$
introducing the hyperparameters $\tau$ and $\delta$ in place of the original update~\eqref{eq:dngposterior}. We set $\hat p_T(c\mid x_T)=\pi$ and $\delta=0$,
and sweep $\lambda_0\in\{0.5,1,2,5,10,20\}$ at $\tau\in\{0.2,1.0\}$;
the main-text comparison uses $(\lambda_0,\tau)=(1,1)$, written as ``DNG ($\eta{=}1$)'' there.
\item \textbf{Ours.} \autoref{alg:neg_das_pu_singlechain} with
$\mu_t^{\mathrm{prop}} = \mu_\theta(x_t,t) + \sigma_t^2\lambda_{t-1}\nabla_{x_t}\log(1-f_t(\hat x_0(x_t)))$
of \eqref{eq:prop-fixed}.
\end{itemize}

\paragraph{Online PU classifier.}
$f_t = g_0 \circ \hat x_0$ is a kernel logistic regression in $\hat x_0$ space, refit at every step. The dictionary at step $t$ comprises $n_p{=}128$ centers subsampled from the conditional Tweedie estimates and $n_u{=}64$ from the unconditional ones, taken over the current proposal particles. The bandwidth uses the median heuristic scaled by $0.3$. Weights are obtained by L-BFGS minimization of the uPU logistic risk \eqref{eq:upu} with an RKHS-norm penalty $\tfrac{\note{\lambda_{\mathrm{reg}}}}{2}\|g\|_{\mathcal H_k}^2$, $\note{\lambda_{\mathrm{reg}}} = 10^{-4}$; the uPU class prior is annealed as $\pi_t = 0.02 + 0.13\exp(-5(1 - t/T))$.

\paragraph{SMC parameters.}
Tempering of $\lambda_t$ is adaptive: $\delta_t \in [0, 1.5/T]$ is found by binary search so that the ESS of $\{w_{q,t}^j(1 - f_t(\hat x_0(x_t^j)))^{\delta_t}\}$ matches the target $0.9M$, and $\lambda_{t-1} \leftarrow \min(\lambda_t + \delta_t, 1)$. ESS-based resampling \citep{chopin2020introduction} fires below $0.9M$ and is disabled for the final $5$ steps.

\subsubsection{Results}
\label{app:exp:gmm:results}

\begin{table}[t]
\centering
\small
\caption{NegCFG $\eta$ sweep, $50$ trials, mean$\pm$std.
Mode coverage\ $\!=\mathbb P[\min_{k\in\neg c}\|x-\mu_k\|<5]$.}
\label{tab:gmm-negcfg}
\begin{tabular}{rccc}
\toprule
$\eta$ & Leak (mean) $\downarrow$ & $W_2$ $\downarrow$ & Mode coverage\ $\uparrow$\\
\midrule
$0.0$ (uncond) & $0.253\pm0.016$ & $3.6\pm0.6$ & $1.00$ \\
$0.1$          & $0.193\pm0.014$ & $2.7\pm0.4$ & $1.00$ \\
$0.2$          & $0.150\pm0.013$ & $3.5\pm0.7$ & $1.00$ \\
$0.5$          & $0.068\pm0.008$ & $9.7\pm2.4$ & $0.99$ \\
$0.7$          & $0.041\pm0.007$ & $14.1\pm3.8$ & $0.97$ \\
$1.0$          & $0.018\pm0.005$ & $20.5\pm6.3$ & $0.94$ \\
$1.5$          & $0.005\pm0.002$ & $34.6\pm17.3$ & $0.85$ \\
$2.0$          & $0.001\pm0.001$ & $50.9\pm34.5$ & $0.74$ \\
$3.0$          & $0.000\pm0.000$ & $102.0\pm107.2$ & $0.50$ \\
\bottomrule
\end{tabular}
\end{table}

\begin{table}[t]
\centering
\small
\caption{DNG sweep, $50$ trials. The main-text comparison uses $(\lambda_0,\tau)=(1,1)$.}
\label{tab:gmm-dng}
\begin{tabular}{rcccc}
\toprule
       & \multicolumn{2}{c}{$\tau=1.0$} & \multicolumn{2}{c}{$\tau=0.2$}\\
\cmidrule(lr){2-3}\cmidrule(lr){4-5}
$\lambda_0$ & Leak $\downarrow$ & $W_2$ $\downarrow$ & Leak $\downarrow$ & $W_2$ $\downarrow$ \\
\midrule
$0.5$  & $0.101\pm0.013$ & $2.2\pm0.4$ & $0.164\pm0.013$ & $2.3\pm0.4$ \\
$1.0$  & $0.052\pm0.010$ & $1.6\pm0.3$ & $0.105\pm0.011$ & $2.1\pm0.3$ \\
$2.0$  & $0.015\pm0.005$ & $3.5\pm0.7$ & $0.040\pm0.006$ & $4.6\pm0.7$ \\
$5.0$  & $0.001\pm0.001$ & $7.8\pm1.1$ & $0.002\pm0.001$ & $9.5\pm1.1$ \\
$10.0$ & $0.000\pm0.000$ & $11.1\pm1.2$ & $0.000\pm0.000$ & $12.9\pm1.6$ \\
$20.0$ & $0.000\pm0.000$ & $14.6\pm1.6$ & $0.000\pm0.000$ & $15.7\pm1.9$ \\
\bottomrule
\end{tabular}
\end{table}

\paragraph{Baseline sweeps.}
Tables~\ref{tab:gmm-negcfg} and~\ref{tab:gmm-dng} report the full sweeps for NegCFG and DNG.
Both methods exhibit a clear trade-off between leak and $W_2$.
NegCFG drives leak monotonically toward zero as $\eta$ grows, but $W_2$ inflates and mode coverage drops below $1$ for $\eta\ge 0.7$, indicating mode collapse onto a few $\neg c$-modes.
DNG shows the same trade-off at both $\tau\in\{0.2,1.0\}$, and the configuration closest to the Oracle leak ($\lambda_0=2,\tau=1$) still attains $W_2=3.5$, more than twice the Oracle.
These patterns are consistent with the local-repulsion limitations discussed in \autoref{sec:existing}.

\note{
\paragraph{Timestep count for DNG.}
\label{app:exp:gmm-dng-T}
With exact scores and at $\eta{=}1$, DNG is exact only as $T\to\infty$ (Appendix~\ref{app:prelim:dng}).
Table~\ref{tab:gmm-dng-T} measures how the sampler approaches that limit.
Only the timestep count $T$ varies; the method is held at the main-text configuration $(\lambda_0,\tau)=(1,1)$ at $M{=}1024$ over $15$ trials,
and the Oracle is recomputed at each $T$ as a matched floor.
The leak falls monotonically from $0.053$ to $0.020$, and most of the gain is already reached by $T{=}500$.
The gap to the matched Oracle narrows from $0.041$ to $0.009$ without closing,
so a twentyfold serial budget still leaves the leak at about $1.8$ times the Oracle value.

\begin{table}[t]
\centering
\small
\caption{Timestep-count ablation for DNG on the Gaussian mixture,
mean$\pm$std over $15$ trials, at $(\lambda_0,\tau)=(1,1)$ and $M{=}1024$.
The Oracle is recomputed at each $T$, and Gap is the difference between the two leak columns.}
\label{tab:gmm-dng-T}
\begin{tabular}{rccc}
\toprule
$T$ & Leak $\downarrow$ & Oracle leak & Gap \\
\midrule
$100$  & $0.053\pm0.010$ & $0.012$ & $0.041$ \\
$200$  & $0.040\pm0.009$ & $0.012$ & $0.028$ \\
$500$  & $0.028\pm0.006$ & $0.010$ & $0.018$ \\
$1000$ & $0.023\pm0.008$ & $0.011$ & $0.012$ \\
$2000$ & $0.020\pm0.009$ & $0.011$ & $0.009$ \\
\bottomrule
\end{tabular}
\end{table}
}

\note{
\paragraph{Particle count.}
Table~\ref{tab:gmm-M} varies the particle count $M$ with true scores over $15$ trials.
The Oracle is recomputed at each $M$, since the empirical $W_2$ of $M$ particles carries a finite-sample floor that grows as $M$ shrinks.
The leak is nearly $M$-independent at the matched Oracle floor (the mean at $M{=}256$ is inflated by a few trials), and $W_2$ decreases with $M$ as the floor does.
The ESS$/M$ stays at $0.96$--$0.97$ with two to three resampling events at every $M$,
so the degeneracy signature of a poor SMC path, an ESS collapse or constant resampling, is not observed.
Against the closed-form posterior, the AUC of $f_t$ near the terminal steps is $0.997$, $0.989$, and $0.980$ for $M{=}64$, $256$, and $1024$,
while the cosine similarity between $\nabla\log(1-f_t)$ and $\nabla\log(1-p(c\mid x))$ stays around $0.17$--$0.18$ at every $M$:
the estimate identifies the region to avoid accurately at every $M$ without being pointwise accurate.

\begin{table}[t]
\centering
\small
\caption{Particle-count ablation on the Gaussian mixture with true scores, mean$\pm$std over $15$ trials.
The Oracle is recomputed at each $M$. The row at $M{=}1024$ differs from Table~\ref{tab:mog}, which averages over $50$ trials.}
\label{tab:gmm-M}
\begin{tabular}{rcccccc}
\toprule
$M$ & Oracle leak & Ours leak $\downarrow$ & Oracle $W_2$ & Ours $W_2$ $\downarrow$ & ESS$/M$ & Resamples \\
\midrule
$64$   & $0.011$ & $0.009\pm0.014$ & $3.6$ & $3.5\pm0.7$ & $0.96$ & $3.1$ \\
$256$  & $0.014$ & $0.021\pm0.028$ & $2.5$ & $2.1\pm0.4$ & $0.96$ & $2.7$ \\
$1024$ & $0.012$ & $0.008\pm0.006$ & $1.5$ & $1.9\pm0.6$ & $0.97$ & $2.2$ \\
\bottomrule
\end{tabular}
\end{table}
}

\paragraph{Comparison with Ours.}
\note{Among the guidance methods in Table~\ref{tab:mog}, only ours} matches the Oracle leak ($0.012$ vs.\ $0.012$) without inflating $W_2$ ($1.8\pm0.6$ vs.\ $1.5\pm0.3$).
This is in line with the asymptotic-exactness result in \autoref{sec:theory}.
The residual gap in $W_2$ relative to the Oracle is attributable to a slight inter-mode imbalance induced by resampling, visible in \autoref{fig:mog} as a non-uniform allocation of particles across the nine $\neg c$-modes.

\subsection{CIFAR-10 with Reference-Based Setting}
\label{app:exp:cifar10}

This appendix gives the full specification of the CIFAR-10 experiment in \autoref{sec:exp:cifar10}.
The complete code is available in our supplementary material.

\subsubsection{Setup}
\label{app:exp:cifar10:setup}

\paragraph{Pretrained score model.}
We use the unconditional CIFAR-10 checkpoint from \cite{pmlr-v139-nichol21a}, distributed under the MIT license at \url{https://github.com/openai/improved-diffusion}.
Following the instructions in that repository, the weights were obtained from \url{https://openaipublic.blob.core.windows.net/diffusion/march-2021/cifar10_uncond_50M_500K.pt}.
This URL was confirmed accessible on 2026-05-05.
Sampling uses DDPM with the respaced schedule of $T=250$ steps, and images are normalized to $[-1,1]$.

\paragraph{Avoidance conditions and references.}
The single-class conditions for the avoidance experiment are $c\in\{\{\text{plane}\},\{\text{cat}\},\{\text{frog}\},\{\text{ship}\}\}$, and the multi-class condition \emph{Vehicles} is $c=\{\text{plane},\text{car},\text{ship},\text{truck}\}$.
This gives $N=200$ references for single-class conditions and $N=800$ for Vehicles.
The reference-selection seed is offset from the diffusion-noise seed so that the two are statistically independent.

\paragraph{Reference-based guidance: Safe Denoiser.}
We use the Safe Denoiser of \citet{kim2025trainingfree} (\autoref{sec:existing}) both as a baseline and as the proposal kernel in our SMC sampler.
Its bandwidth is set once per run by the median pairwise pixel distance among the references.
Following the original implementation, guidance is applied for $t > 0.08\,T = 20$ and disabled for the final steps.
We sweep the repulsion strength $\eta\in\{0.0, 0.2, 0.3, 0.5, 0.7, 1.0\}$.

\paragraph{Online PU classifier on UNet bottleneck features.}
The classifier $f_t$ is a linear logistic head on the $256$-dim mean-pooled output of the UNet bottleneck block (shape $4{\times}4{\times}256$).
Reference features are encoded once at $t=0$ on the clean references.
For the single-class conditions, particle features are read from the same forward pass that produces $\mu_\theta(x_t,t)$, giving $f_t(x_t) = h(\Phi_t(x_t))$ with a time-conditioned encoder $\Phi_t$.
For Vehicles, an additional UNet pass at $t=0$ on the clipped Tweedie estimate is applied, giving $f_t(x_t) = h(\Phi_0(\hat x_0(x_t)))$.
The head is fit by L-BFGS with warm-starting from the previous step, using the BCE surrogate of \autoref{sec:exp:cifar10}.

\paragraph{SMC parameters.}
We use a precomputed tempering schedule: $\lambda_t$ ramps linearly from $0$ to a terminal $\lambda$ over the last $0.6\,T=150$ steps.
ESS-based sorted systematic resampling fires below $0.9M$, and is enabled only for $t>15$ in single-class runs and $t>50$ in Vehicles.

\paragraph{Implementation choices that depart from the abstract algorithm.}
We set the proposal-ratio factor $r_t$ in \eqref{eq:wq-recursion} to $1$ \note{in the recursion of $w_{q,t}$}.
At $d=3072$, this ratio has high variance, and a $128$-particle filter cannot maintain a healthy ESS when it is included.
\note{The unconditional weight $w_{\emptyset,t}$ keeps the full ratio $r_t$ at every step as in~\eqref{eq:wempty-recursion}, and its update at resampling is the standard one of \autoref{sec:method:framework}.}
This treatment, together with the BCE surrogate, deviates from the analysis in \autoref{sec:theory}\note{, so the procedure is a practical approximation not covered by \autoref{thm:main-informal}}; the experiments then test whether the existing negative-guidance method can be corrected by a lightweight auxiliary classifier and resampling.
The classifier is trained and used only on the steps where $\lambda_t > 0$; earlier steps run the Safe Denoiser alone.
The L-BFGS head uses up to $100$ iterations on the first step and up to $20$ iterations on later steps, so the online training cost is small.

\paragraph{Metrics.}

The generated samples carry no class labels by construction, so we evaluate them with an independently trained ResNet-$18$ classifier~\citep{He_2016_CVPR}.
This classifier is not used during sampling and is fully decoupled from the diffusion pipeline.
We start from the torchvision ResNet-$18$ with the \texttt{IMAGENET1K\_V1} pretrained weights\footnote{\url{https://docs.pytorch.org/vision/stable/models/generated/torchvision.models.resnet18.html}} and apply two  modifications for low-resolution inputs: the first convolutional layer is replaced by a $3{\times}3$, stride-$1$ convolution, and the max-pooling layer is replaced by an identity mapping.
The final fully connected layer is replaced by a 10-class linear head.
We fine-tune all parameters for $50$ epochs on the CIFAR-10 training split using SGD (learning rate $0.01$, momentum $0.9$, weight decay $5{\times}10^{-4}$) with cosine annealing, a batch size of $256$, and data augmentation (random horizontal flipping and random $32{\times}32$ crops with four pixels of padding).
Train accuracy exceeds $99\%$, and the checkpoint with the highest test accuracy reaches ${\sim}95\%$ Top-1 on the CIFAR-10 test split, which we consider sufficient for classifying our generated samples.
The training code is included in the supplementary material.

\begin{itemize}
\item \emph{Target-class rate.}
The fraction of generated samples classified into $c$ by the ResNet-$18$.
For multi-class $c$, the rate counts any class in $c$.
\item \emph{FID.}
Computed via \texttt{pytorch-fid}~$0.3.0$ \citep{Seitzer2020FID} between the generated samples and the CIFAR-10 \emph{training} images that lie in $\neg c$, so the floor is the Fr\'echet distance between samples from $p(\cdot\mid\neg c)$ and the $\neg c$-only training set.
\item \emph{KL.}
$\mathrm{KL}(u_{\neg c}\,\|\,\hat p_{10})$, where $u_{\neg c}$ is the uniform distribution over non-avoided classes (mass $1/(10-|c|)$ on each $i \notin c$, $0$ on $c$), and $\hat p_{10}$ is the empirical $10$-class histogram of the ResNet-$18$ predictions clipped at $10^{-10}$.
\item \emph{Pairwise LPIPS.}
Average pairwise LPIPS \citep{Zhang_2018_CVPR} with the AlexNet backbone, computed on $256$ random samples per configuration.
\end{itemize}
The seed pool is shared across configurations so that paired comparisons (Ours vs.\ the Safe Denoiser at the same $\eta$) are well-defined.

\subsubsection{Results}
\label{app:exp:cifar10:results}

\begin{table}[!b]
\centering

\caption{PU regularization ablation at $\lambda=1.5$, $80$ seeds per cell.
Our method is run at $\eta=0.3$, while the Safe Denoiser baseline is shown at $\eta\in\{0.3, 0.5\}$, the two values that bracket the knee of its rate--FID curve in \autoref{fig:cifar10} and \autoref{fig:cifar-perclass-sweeps}.}
\label{tab:cifar-reg-ablation}

\begin{tabular}{llccccc}
\toprule
Class & Method & $\eta$ & Rate (\%) $\downarrow$ & FID $\downarrow$ & KL $\downarrow$ & LPIPS $\uparrow$ \\
\midrule
Plane & Safe Denoiser                          & $0.3$ & $4.78$ & $20.75$ & $0.127$ & $0.200$ \\
      & Safe Denoiser                          & $0.5$ & $2.95$ & $45.96$ & $0.279$ & $0.175$ \\
      & Ours, $\lambda_{\mathrm{reg}}=10^{-3}$ & $0.3$ & $3.50$ & $24.65$ & $0.148$ & $0.197$ \\
      & Ours, $\lambda_{\mathrm{reg}}=1.0$     & $0.3$ & $4.39$ & $23.66$ & $0.136$ & $0.198$ \\
\midrule
Cat   & Safe Denoiser                          & $0.3$ & $7.16$ & $20.63$ & $0.150$ & $0.207$ \\
      & Safe Denoiser                          & $0.5$ & $6.10$ & $44.53$ & $0.282$ & $0.183$ \\
      & Ours, $\lambda_{\mathrm{reg}}=10^{-3}$ & $0.3$ & $5.66$ & $24.95$ & $0.163$ & $0.203$ \\
      & Ours, $\lambda_{\mathrm{reg}}=1.0$     & $0.3$ & $6.17$ & $25.38$ & $0.200$ & $0.205$ \\
\midrule
Frog  & Safe Denoiser                          & $0.3$ & $2.55$ & $41.53$ & $0.181$ & $0.183$ \\
      & Safe Denoiser                          & $0.5$ & $0.87$ & $93.59$ & $0.503$ & $0.162$ \\
      & Ours, $\lambda_{\mathrm{reg}}=1.0$     & $0.3$ & $1.96$ & $49.25$ & $0.249$ & $0.183$ \\
      & Ours, $\lambda_{\mathrm{reg}}=10^{-3}$ & $0.3$ & $1.90$ & $50.31$ & $0.245$ & $0.185$ \\
\midrule
Ship  & Safe Denoiser                          & $0.3$ & $3.63$ & $23.68$ & $0.130$ & $0.194$ \\
      & Safe Denoiser                          & $0.5$ & $1.76$ & $55.77$ & $0.296$ & $0.169$ \\
      & Ours, $\lambda_{\mathrm{reg}}=1.0$     & $0.3$ & $2.52$ & $27.44$ & $0.133$ & $0.196$ \\
      & Ours, $\lambda_{\mathrm{reg}}=10^{-3}$ & $0.3$ & $3.22$ & $26.08$ & $0.138$ & $0.203$ \\
\bottomrule
\end{tabular}
\end{table}

\paragraph{Choice of the PU L2 regularizer.}
We study the sensitivity of our method to the L2 regularization strength $\lambda_{\mathrm{reg}}$ by comparing two values spanning three orders of magnitude ($10^{-3}$ and $1.0$).
Table~\ref{tab:cifar-reg-ablation} reports a per-class ablation at $\lambda=1.5$, with our method run at $\eta=0.3$.
For reference, we also list the Safe Denoiser at $\eta=0.3$ and $\eta=0.5$, the two operating points that bracket the knee of its rate--FID curve.
The choice of $\lambda_{\mathrm{reg}}$ shifts the metrics by a small amount, with within-class variation of at most $0.9$ percentage points on rate and $1.4$ FID points.
At either value, our method already reaches a lower target-class rate than the Safe Denoiser at $\eta=0.3$, while paying only \note{$2$--$9$} FID points above it.
The Safe Denoiser only matches our rate by raising $\eta$ to $0.5$, at which point its FID degrades by \note{$24$--$52$} points beyond $\eta=0.3$.
In other words, both choices of $\lambda_{\mathrm{reg}}$ yield Pareto improvements over the guidance-only baseline.
Based on this, we fix $\lambda_{\mathrm{reg}}=10^{-3}$ for plane and cat, and $\lambda_{\mathrm{reg}}=1.0$ for frog, ship, and Vehicles.

\paragraph{Per-class sweeps.}
\autoref{fig:cifar-perclass-sweeps} reports the rate-vs-quality trade-off for the cat, frog, and ship single-class conditions, using the defaults above.
The qualitative picture is consistent with the plane and Vehicles results in \autoref{fig:cifar10} in \autoref{sec:exp:cifar10}: our $\lambda$-sweep Pareto-dominates the Safe Denoiser at the same $\eta$ on FID and KL \note{in most configurations}, and remains comparable on pairwise LPIPS.
The shape of the Safe Denoiser curve clarifies why our method helps: in the small-$\eta$ regime, increasing $\eta$ improves the avoidance rate at modest cost to FID and KL, but beyond a knee around $\eta=0.5$ the quality metrics degrade sharply for any further rate reduction.
Our method takes a Safe Denoiser proposal in this small-$\eta$ regime and adds resampling driven by the PU classifier, which extracts further rate reduction along the part of the curve where the quality cost is still mild.

\begin{figure}[tp]
\centering
\includegraphics[width=0.891\linewidth]{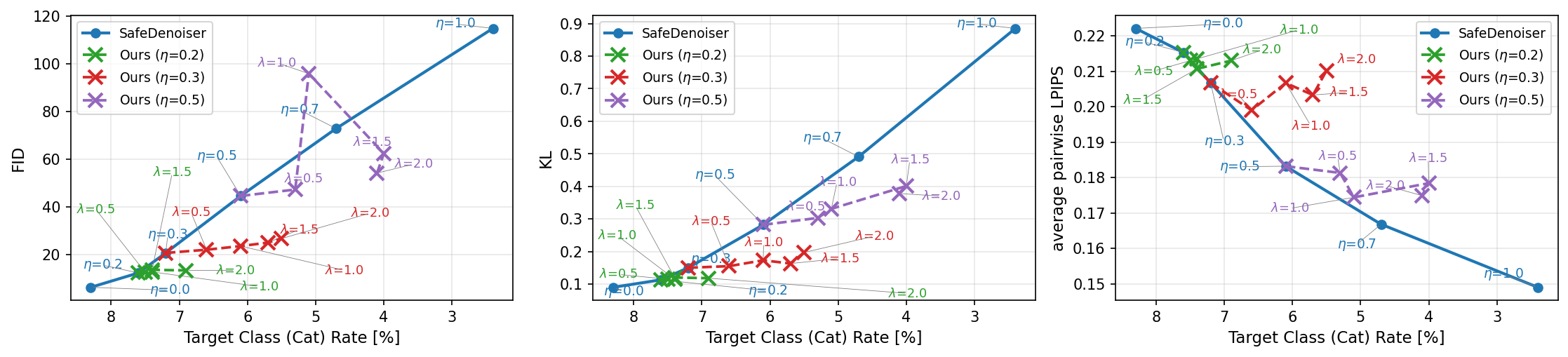}
\includegraphics[width=0.891\linewidth]{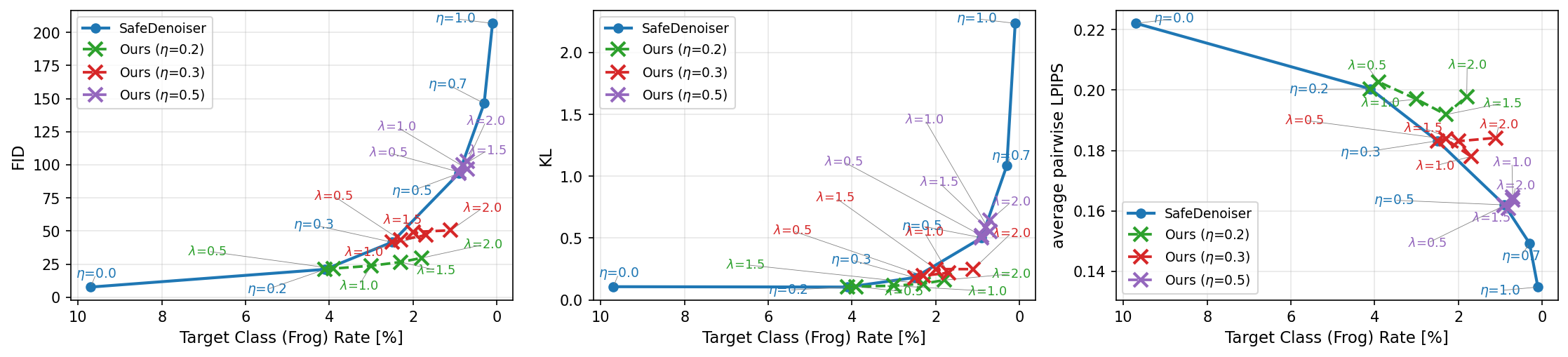}
\includegraphics[width=0.891\linewidth]{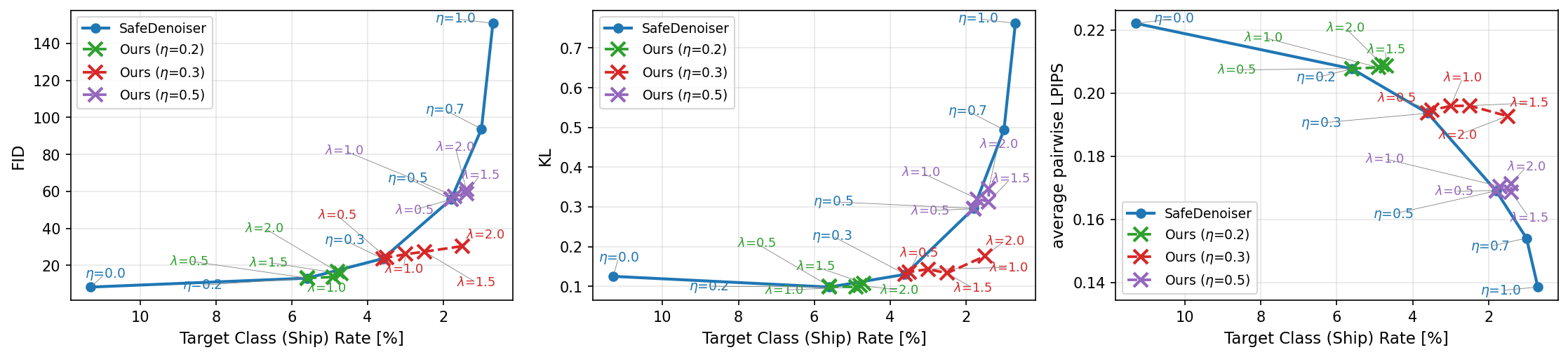}
\caption{Avoidance rate vs.\ generation quality on CIFAR-10 for the cat (top), frog (middle), and ship (bottom) single-class conditions. Columns plot target-class rate against FID, KL, and average pairwise LPIPS (left to right); lower is better except for LPIPS.}
\label{fig:cifar-perclass-sweeps}
\end{figure}

\note{
\paragraph{ESS diagnostics.}
Table~\ref{tab:cifar-ess} reports the ESS of $w_{q,t}$, the number of distinct surviving ancestors,
and the number of resampling events on the main configurations ($\eta{=}0.3$, $M{=}128$, $T{=}250$;
mean$\pm$std over $80$ runs each).
The time-averaged ESS stays at $110$--$127$ of $128$, and resampling stays sparse ($0$--$7$ events over the $250$ steps) under the conservative threshold $0.9M$,
a stability that the $r_t{=}1$ substitution provides.
Stronger avoidance $\lambda$ or the harder Vehicles condition deepens the ESS minimum, yet even the hardest setting keeps $56$ of $128$ distinct ancestors,
so the particles do not collapse onto a few ancestors.
The ESS minima occur in the final steps, where the tilt is largest and resampling is intentionally disabled to avoid duplicating nearly final particles.

\begin{table}[tp]
\centering
\small
\caption{ESS diagnostics on CIFAR-10 at $\eta{=}0.3$, $M{=}128$, $T{=}250$, mean$\pm$std over $80$ runs.
Ancestors counts the distinct ancestors among the final $128$ particles.}
\label{tab:cifar-ess}
\begin{tabular}{lccccc}
\toprule
Class & $\lambda$ & ESS (time avg.) & ESS (min) & Ancestors$/128$ & Resamples \\
\midrule
Plane    & $0.5$ & $126.5\pm0.6$ & $59.6\pm22.7$ & $105.4\pm19.0$ & $1.3\pm1.4$ \\
Plane    & $1.0$ & $125.9\pm0.8$ & $43.9\pm25.5$ & $85.2\pm27.0$  & $2.9\pm2.2$ \\
Plane    & $2.0$ & $125.1\pm0.9$ & $37.5\pm24.1$ & $66.6\pm28.2$  & $4.7\pm3.0$ \\
Vehicles & $0.5$ & $121.9\pm1.8$ & $72.5\pm18.0$ & $121.1\pm9.2$  & $0.4\pm0.5$ \\
Vehicles & $1.0$ & $116.4\pm2.8$ & $33.2\pm18.6$ & $94.0\pm14.1$  & $2.3\pm1.3$ \\
Vehicles & $2.0$ & $110.4\pm3.1$ & $10.1\pm7.1$  & $56.2\pm18.4$  & $6.4\pm2.3$ \\
\bottomrule
\end{tabular}
\end{table}

\paragraph{Batch-level spread.}
Table~\ref{tab:cifar-spread} reports the target-class rate per $128$-sample batch as mean$\pm$std over the $80$ seeds at $\eta{=}0.3$.
The means equal the rates plotted in \autoref{fig:cifar10}, and the standard error of each mean is std$/\sqrt{80}$.
The larger std under strong avoidance (Vehicles, $\lambda{=}2.0$) is consistent with the low ESS minimum in Table~\ref{tab:cifar-ess}.
The evaluator (test accuracy about $95\%$, training accuracy above $99\%$) follows the protocol of prior class-avoidance work~\citep{takahashi2025positiveunlabeled};
all methods are scored by the same evaluator, which is independent of every generation pipeline.

\begin{table}[!b]
\centering
\small
\caption{Target-class rate [\%] per $128$-sample batch on CIFAR-10 at $\eta{=}0.3$, mean$\pm$std over $80$ seeds.}
\label{tab:cifar-spread}
\begin{tabular}{lcc}
\toprule
Method & Plane & Vehicles \\
\midrule
Safe Denoiser            & $4.78\pm1.99$ & $27.60\pm3.98$ \\
Ours ($\lambda{=}0.5$)   & $3.88\pm2.31$ & $25.86\pm5.27$ \\
Ours ($\lambda{=}1.0$)   & $3.28\pm2.66$ & $19.87\pm6.09$ \\
Ours ($\lambda{=}2.0$)   & $2.54\pm3.01$ & $20.93\pm18.06$ \\
\bottomrule
\end{tabular}
\end{table}
}

\clearpage
\note{
\subsection{Stable Diffusion with Score-Based Setting}
\label{app:exp:sd}

This subsection instantiates the score-based setting of
\autoref{sec:method:score} on Stable Diffusion v1.4 \citep{Rombach_2022_CVPR}\footnote{\url{https://huggingface.co/CompVis/stable-diffusion-v1-4} \note{(CreativeML OpenRAIL-M license).}},
a latent diffusion model that generates $512\times512$ images from text prompts and whose state is a $64\times64$ latent grid with four channels, so $d = 16{,}384$.

\paragraph{Task and prompts.}
The content prompt is \texttt{"a painting of a cafe terrace at night under a starry sky"}, which never names the avoided style.
Following the negative-prompting format \citep{ho2022classifierfreediffusionguidance,10.1007/978-3-031-73024-5_12},
the style prompt appends \texttt{"in the style of Van Gogh"} to the content prompt,
and supplies the conditional predictor $\hat\epsilon_\theta(\cdot,\cdot\mid c)$ required by \autoref{sec:method:score}.
The evaluation pools and the detector calibration set use the longer suffix \texttt{"in the style of Vincent van Gogh"}.

\paragraph{Scores and sampling.}
The base model $p_{\theta,t}$ is the classifier-free-guidance combination of the content-prompt and unconditional predictions at scale $7.5$,
which is also the sampler that generates the content anchor.
Sampling uses $T{=}50$ steps in half precision.
SLD and DNG modify this same prediction, so their density without avoidance is also the content anchor; NegCFG is instead anchored on the style prediction (below).
SLD \citep{Schramowski_2023_CVPR} at its \textsc{xxmax} setting below serves as the SMC proposal $\mu_t^{\mathrm{prop}}$ for PU Steering throughout,
as the Safe Denoiser does in \autoref{app:exp:cifar10}.

\paragraph{Online PU classifier.}
Each particle is represented by average-pooling the four latent channels of its Tweedie estimate to a $16\times16$ grid, giving a $1024$-dimensional feature.
Positives are the style-conditional Tweedie estimates and unlabeled samples are the content Tweedie estimates of the same particles,
\note{a practical variant of the construction of \autoref{sec:method:score} with uniformly weighted positives and $w_{q,t}$-weighted unlabeled samples}.
Both reuse noise predictions that the SLD proposal already computes at every timestep, so the features cost no additional UNet evaluations.
The classifier is a kernel logistic model with an RBF bandwidth set by the median heuristic scaled by $0.4$,
fit to the uPU risk~\eqref{eq:upu} with an RKHS-norm penalty of $10^{-2}$ by L-BFGS at class prior $\pi{=}0.15$, refit at every timestep over the last $70\%$ of the trajectory.
Its logit is standardized across particles before the sigmoid and the resulting probability is clamped to $[10^{-3},1-10^{-3}]$.
The two estimates start from the same particle, so they share the prompted content and differ mainly in the avoided style, and this contrast keeps the uPU fit stable:
unlike the CIFAR-10 experiment of \autoref{sec:exp:cifar10}, where the small $N$ and $M$ forced the BCE substitution, this instantiation uses the uPU objective as written.

\paragraph{SMC control.}
As in \autoref{app:exp:cifar10}, we set the proposal-ratio factor $r_t$ in \eqref{eq:wq-recursion} to $1$.
The tempering increment is chosen by bisection to hold the effective sample size at $0.5M$, capped at $1.5\lambda/T$ per timestep and at the terminal $\lambda$.
Systematic resampling fires when the effective sample size falls below $0.7M$,
only within the middle window between $30\%$ and $70\%$ of the reverse steps, at a cadence of five steps (four for the $M{=}128$, $\lambda{=}20$ configuration),
and never at the final timestep, so each run returns weighted particles that can be read either unweighted as plain images or weighted as the SMC estimate.
Log-weights are clipped to $[-2,2]$ around their mean.

\paragraph{Baselines.}
SLD \citep{Schramowski_2023_CVPR} runs at three settings on the same guidance scale.
\textsc{max} is its strongest published preset \textsc{Hyp-Max} (no warm-up, safety scale $5000$, threshold $1$, momentum scale $0.5$, momentum decay $0.7$),
which leaves the style rate at $0.334$ against the floor value $0.238$.
We therefore add two settings beyond the published presets:
\textsc{xxmax} raises the safety scale to $20000$ and the momentum scale to $0.7$, and \textsc{x3max} further raises the momentum scale to $1$ and the momentum decay to $0.8$.
For NegCFG, preliminary sweeps on the guidance combination did not lower the style rate, the content guidance at scale $7.5$ dominating the style term,
so it is run in the negative-prompting form
$\hat\epsilon_{\mathrm{style}}+s\,(\hat\epsilon_{\mathrm{content}}-\hat\epsilon_{\mathrm{style}})$ with $s = \eta+1$ and $\eta\in\{0.5,1,1.5,2,3,4,5,6.5\}$.
DNG \citep{koulischer2025dynamic} runs the guided prediction with the recursion~\eqref{eq:dngposterior} in the form of its released implementation,
at temperature $\tau{=}0.1$, offset $\delta{=}0.02$,
squared distances summed over the $16{,}384$ latent dimensions, an initial posterior estimate of $0.05$, and the estimate clamped to $[10^{-6},0.99]$,
swept over $\eta\in\{100,150,180,200,220,240,260,280,300\}$.

\begin{figure}[t]
\centering
\begin{minipage}[t]{0.48\linewidth}
  \centering
  \includegraphics[width=\linewidth]{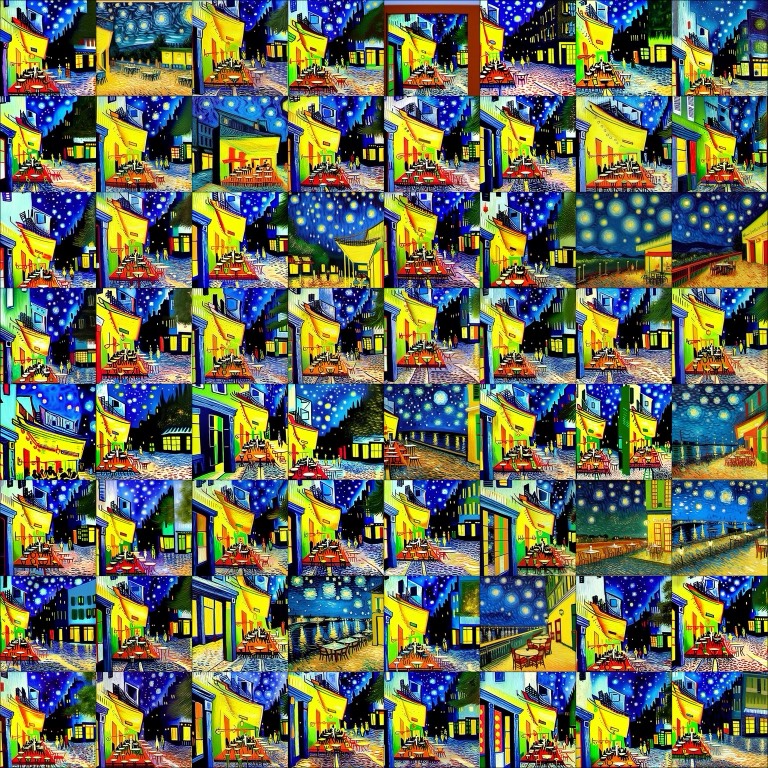}\\
  {\small (a) Avoided style}
\end{minipage}\hfill
\begin{minipage}[t]{0.48\linewidth}
  \centering
  \includegraphics[width=\linewidth]{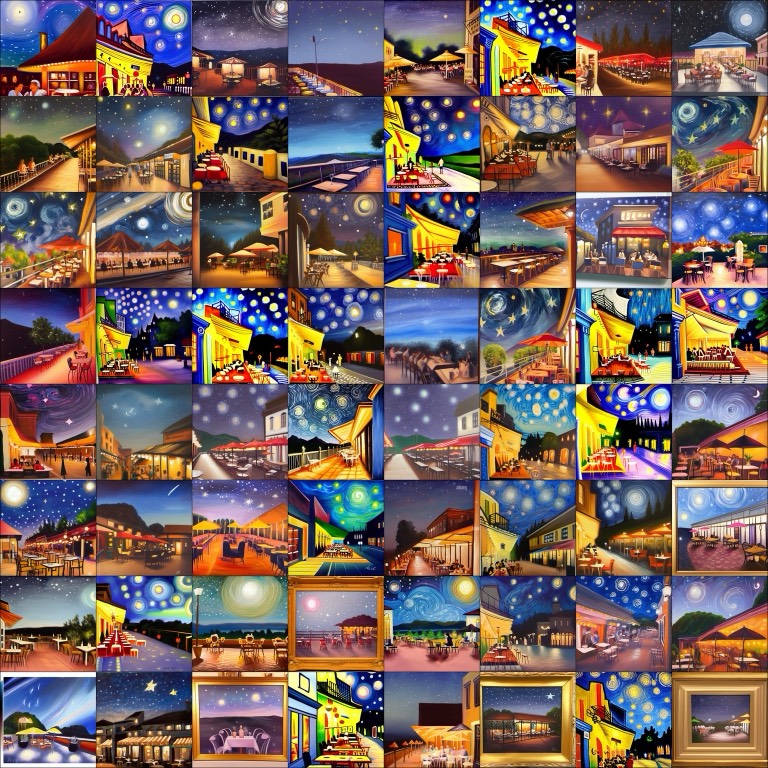}\\
  {\small (b) PU Steering}
\end{minipage}\\[4pt]
\begin{minipage}[t]{0.48\linewidth}
  \centering
  \includegraphics[width=\linewidth]{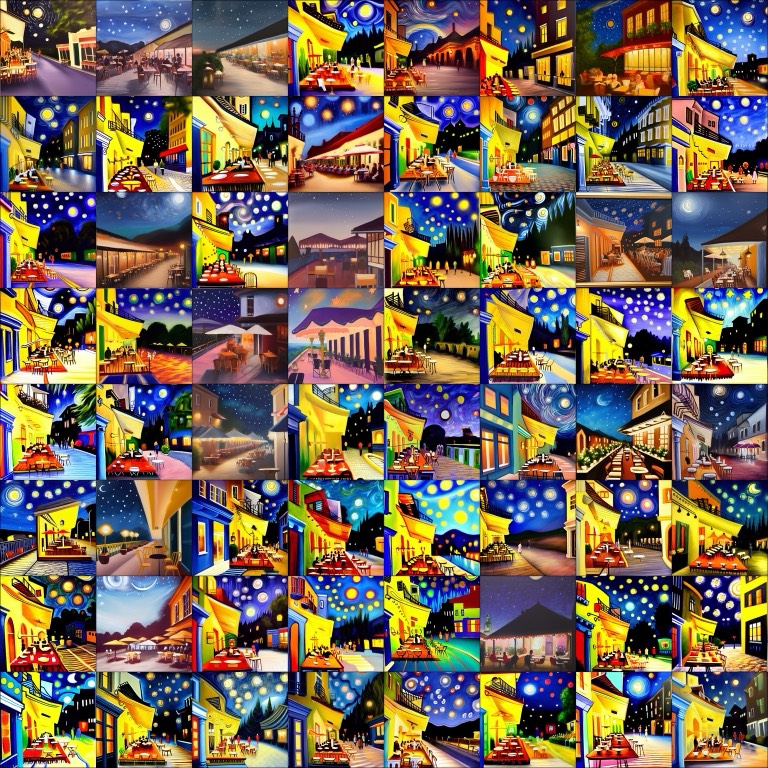}\\
  {\small (c) SLD}
\end{minipage}\hfill
\begin{minipage}[t]{0.48\linewidth}
  \centering
  \includegraphics[width=\linewidth]{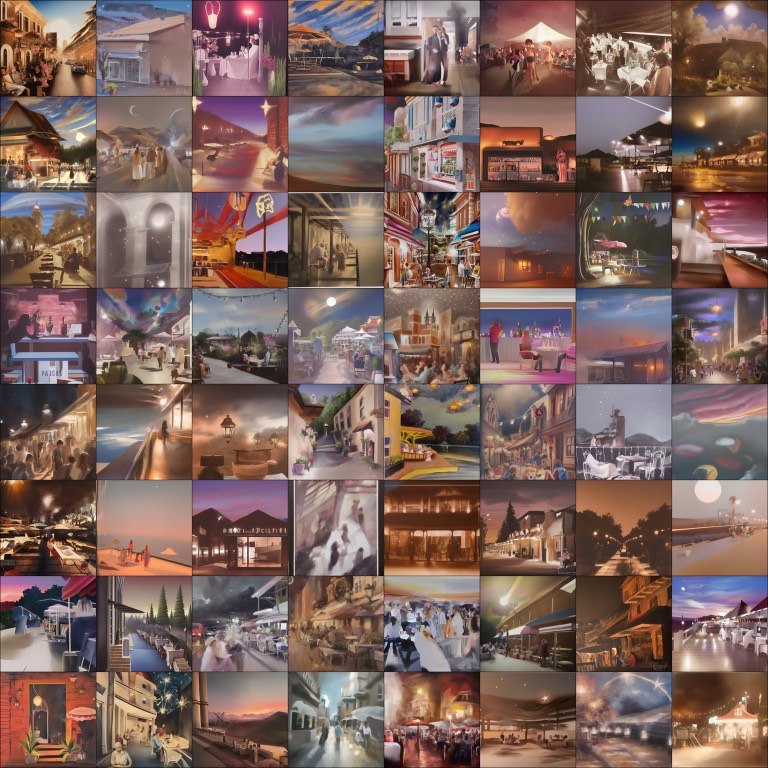}\\
  {\small (d) NegCFG}
\end{minipage}
\caption{Randomly selected $8\times8$ image grids,
from Table~\ref{tab:sd-full}: PU Steering at $M{=}128$, $\lambda{=}16$ read as the equal-weight output after resampling, SLD at \textsc{xxmax},
and NegCFG at $\eta{=}3$, its lowest-FID strength.}
\label{fig:sd-grids}
\end{figure}

\paragraph{Detector, reference, and floor.}
The style detector is a logistic head on the CSD style embedding
\citep{somepalli2024measuring}\footnote{\url{https://huggingface.co/tomg-group-umd/CSD-ViT-L} \note{(CC BY 4.0 license).}},
a $768$-dimensional descriptor of artistic style,
fit by L-BFGS with an $\ell_2$ penalty of $10^{-2}$ on a calibration set whose positives are Van Gogh renderings of the content and whose negatives are Rembrandt renderings and plain content,
all from seed blocks disjoint from every evaluation pool.
The reference set keeps the $2048$ images least like the avoided style, the bottom quarter by detector score of an $8192$-image content pool.
The oracle floor regenerates an independent $8192$-image pool with disjoint seeds and applies the same filter.
FID uses the Inception-$2048$ features of \texttt{pytorch-fid} \citep{Seitzer2020FID},
KID \citep{binkowski2018demystifying} is the unbiased cubic-polynomial-kernel estimate averaged over $100$ random subsets of $1024$,
and Style-FD is the Fr\'echet distance in the CSD embedding space.
The reported rate is the mean of the detector probability over the evaluated images, not a thresholded fraction.
For the weighted rows, all three distances instead use closed-form importance-weighted moments over the full set, without subsampling.
Because the oracle floor is itself a filtered content pool, its rate $0.238$ is the residual style rate of that distribution; it is a reference point rather than a lower bound.

\paragraph{Qualitative comparison.}
\autoref{fig:sd-grids} shows randomly selected images from each sampler.
Grid~(a) is the avoided style, generated with the style prompt: most images follow the composition of \emph{Caf\'e Terrace at Night}, with the yellow awning and the swirling starry sky.
PU Steering in grid~(b) keeps the prompted scene while spreading over a variety of painting styles, a minority of its images still resembling the avoided style.
SLD in grid~(c) stays close to the avoided style in a visibly larger fraction of images, and NegCFG in grid~(d) loses the saturated palette and shows washed-out colors and distortions,
removing the style at the cost of the scene.

\paragraph{Results.}
Table~\ref{tab:sd-full} reports every configuration and \autoref{fig:sd-frontier} plots the resulting frontier.
PU Steering at $M{=}128$, $\lambda{=}16$ reduces the rate of the equal-weight output to $0.235$, essentially the floor value,
while attaining the lowest FID ($28.10$ against the floor's $18.52$) and KID ($0.0054$ against $0.0002$) among the equal-weight outputs.
NegCFG reaches rates down to $0.222$, but its FID never falls below $82.75$,
worse than the content anchor at $60.98$, and its Style-FD stays at or above $0.897$:
it removes the style by destroying the content.
DNG traces a different curve, its FID falling to a minimum of $44.88$ at $\eta{=}200$ and rising again to $93.17$ at $\eta{=}300$ while the rate never falls below $0.277$.
SLD leaves the rate at $0.334$ at its strongest published preset, and the amplified settings reach $0.296$ with FID between $41.82$ and $43.70$.
Raising the particle count from $M{=}64$ to $M{=}128$ at $\lambda{=}16$ lowers the rate from $0.276$ to $0.235$, the FID from $33.46$ to $28.10$ and the KID from $0.0105$ to $0.0054$,
while Style-FD rises slightly from $0.086$ to $0.095$.
The mean terminal ESS stays between $0.47M$ and $0.51M$ across \note{configurations}.
The last block disables resampling and switches the tempering to a linear ramp to keep the run stable;
the weighted estimate then stays accurate while the plain images degrade toward the proposal,
which is why the main protocol reads the resampled particles unweighted.

\begin{table}[p]
\centering
\footnotesize
\setlength{\tabcolsep}{3pt}
\caption{Stable Diffusion v1.4 avoiding Van Gogh style, $2048$ images per configuration, lower is better throughout.
Gray rows are anchors rather than competing methods.
For PU Steering, ``u.'' evaluates the plain images and ``w.'' the SMC-weighted estimate, ESS$/M$ is the mean per-run terminal effective sample size, and $\lambda$ is the terminal tempering strength.
Bold marks the best FID, KID, and Style-FD among the methods outside the no-resampling ablation. Representative baseline strengths are listed;
\autoref{fig:sd-frontier} plots every configuration of the sweeps.}
\label{tab:sd-full}
\begin{tabular}{llccccc}
\toprule
Method & Config & Rate $\downarrow$ & FID $\downarrow$ & KID $\downarrow$ & Style-FD $\downarrow$ & ESS$/M$ \\
\midrule
\textcolor{gray}{Oracle floor} & \textcolor{gray}{$2048$ vs.\ $2048$} & \textcolor{gray}{$0.238$} & \textcolor{gray}{$18.52$} & \textcolor{gray}{$0.0002$} & \textcolor{gray}{$0.009$} & \textcolor{gray}{---} \\
\textcolor{gray}{Content (no avoidance)} & & \textcolor{gray}{$0.409$} & \textcolor{gray}{$60.98$} & \textcolor{gray}{$0.0552$} & \textcolor{gray}{$0.213$} & \textcolor{gray}{---} \\
\textcolor{gray}{Prompted Van Gogh} & & \textcolor{gray}{$0.645$} & \textcolor{gray}{$164.32$} & \textcolor{gray}{$0.1982$} & \textcolor{gray}{$0.661$} & \textcolor{gray}{---} \\
\midrule
SLD & \textsc{max}   & $0.334$ & $43.08$ & $0.0294$ & $0.114$ & --- \\
SLD & \textsc{xxmax} & $0.317$ & $41.82$ & $0.0263$ & $0.118$ & --- \\
SLD & \textsc{x3max} & $0.296$ & $43.70$ & $0.0284$ & $0.180$ & --- \\
NegCFG & $\eta{=}0.5$ & $0.279$ & $112.58$ & $0.1074$ & $0.939$ & --- \\
NegCFG & $\eta{=}2$   & $0.222$ & $85.56$ & $0.0678$ & $0.900$ & --- \\
NegCFG & $\eta{=}3$   & $0.224$ & $82.75$ & $0.0608$ & $0.897$ & --- \\
NegCFG & $\eta{=}6.5$ & $0.280$ & $97.35$ & $0.0775$ & $0.992$ & --- \\
DNG & $\eta{=}100$ & $0.382$ & $56.87$ & $0.0500$ & $0.192$ & --- \\
DNG & $\eta{=}200$ & $0.303$ & $44.88$ & $0.0273$ & $0.145$ & --- \\
DNG & $\eta{=}260$ & $0.277$ & $63.84$ & $0.0411$ & $0.243$ & --- \\
DNG & $\eta{=}300$ & $0.289$ & $93.17$ & $0.0737$ & $0.366$ & --- \\
\midrule
PU Steering (u.) & $M{=}128$, $\lambda{=}16$ & $0.235$ & $\mathbf{28.10}$ & $0.0054$ & $0.095$ & --- \\
\quad (w.) & & $0.228$ & $38.68$ & $\mathbf{0.0047}$ & $0.110$ & $0.47$ \\
PU Steering (u.) & $M{=}128$, $\lambda{=}20$ & $0.243$ & $31.07$ & $0.0056$ & $0.085$ & --- \\
\quad (w.) & & $0.245$ & $43.25$ & $0.0072$ & $0.092$ & $0.51$ \\
PU Steering (u.) & $M{=}128$, $\lambda{=}24$ & $0.253$ & $29.54$ & $0.0058$ & $\mathbf{0.082}$ & --- \\
\quad (w.) & & $0.239$ & $39.95$ & $0.0060$ & $0.111$ & $0.50$ \\
PU Steering (u.) & $M{=}64$, $\lambda{=}16$ & $0.276$ & $33.46$ & $0.0105$ & $0.086$ & --- \\
\quad (w.) & & $0.283$ & $43.27$ & $0.0115$ & $0.096$ & $0.48$ \\
PU Steering (u.) & $M{=}64$, $\lambda{=}16$ (repeat) & $0.280$ & $33.41$ & $0.0106$ & $0.090$ & --- \\
\quad (w.) & & $0.281$ & $43.42$ & $0.0116$ & $0.094$ & $0.49$ \\
\midrule
PU Steering, no rs.\ (u.) & $M{=}64$, $\lambda{=}8$ & $0.314$ & $40.41$ & $0.0243$ & $0.114$ & --- \\
\quad (w.) & & $0.272$ & $39.23$ & $0.0104$ & $0.103$ & $0.50$ \\
PU Steering, no rs.\ (u.) & $M{=}64$, $\lambda{=}12$ & $0.319$ & $41.76$ & $0.0256$ & $0.122$ & --- \\
\quad (w.) & & $0.268$ & $38.08$ & $0.0090$ & $0.106$ & $0.51$ \\
\bottomrule
\end{tabular}
\end{table}

\begin{figure}[p]
\centering
\includegraphics[width=0.99\linewidth]{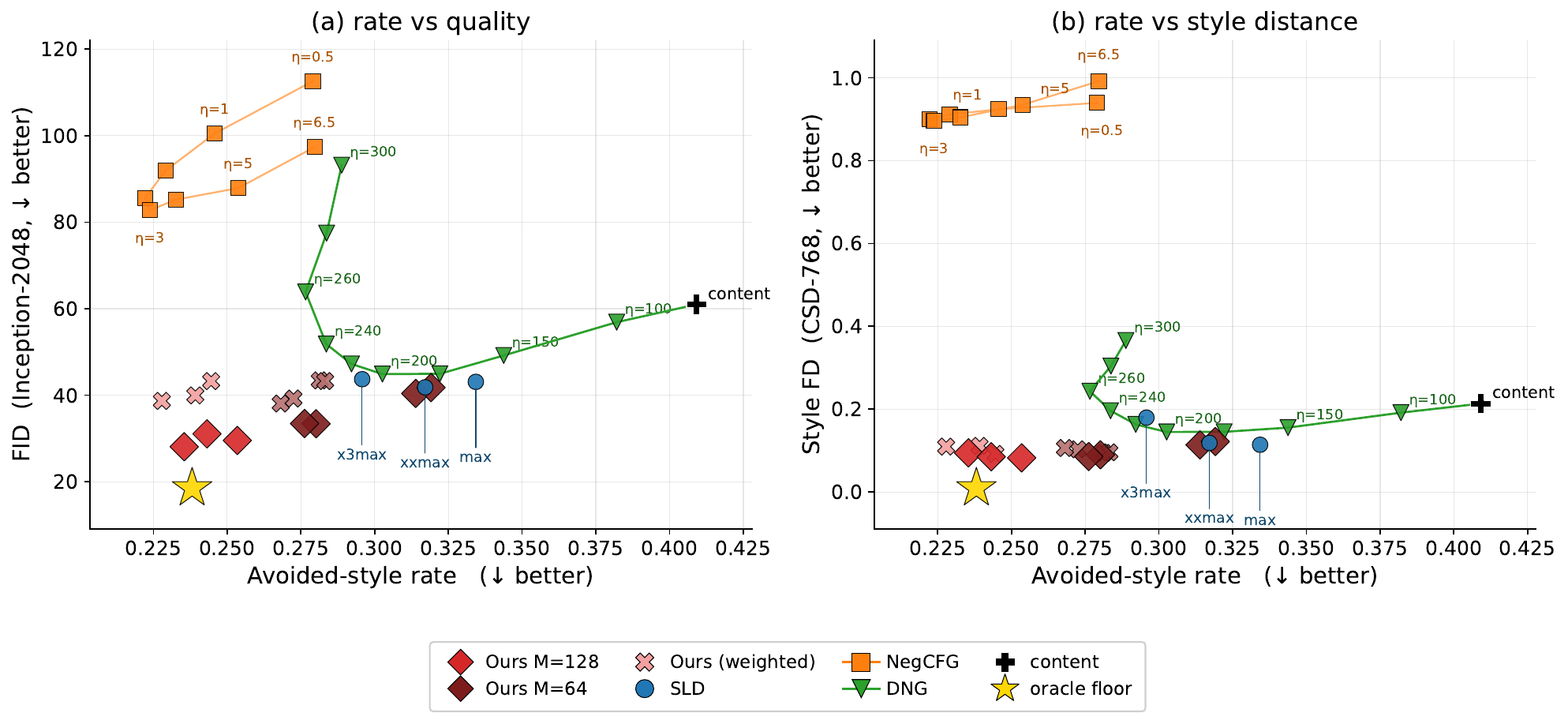}
\caption{Stable Diffusion, avoided-style rate against FID (left) and Style-FD (right); lower left is better in both panels.
Every configuration of Table~\ref{tab:sd-full} is plotted, including the two no-resampling points,
which are the rightmost $M{=}64$ markers.}
\label{fig:sd-frontier}
\end{figure}

\paragraph{Sample-size stability.}
The $2048$ images per configuration are far fewer than the $30{,}000$ on which FID is usually evaluated for text-to-image models \citep{Rombach_2022_CVPR}.
The independent repetition of the $M{=}64$, $\lambda{=}16$ configuration in Table~\ref{tab:sd-full} gives a rate of $0.280$ against $0.276$ and an FID of $33.41$ against $33.46$,
with the remaining metrics similarly close, so the reported values are stable across seeds at this sample size.
}